\documentclass{article}

\usepackage[margin=1in]{geometry}

\RequirePackage[
  style=numeric,
  backend=biber,
  giveninits=true
  ]{biblatex}

\usepackage{amssymb,amsmath,amsfonts,amsthm}
\newtheorem{theorem}{Theorem}
\newtheorem{corollary}[theorem]{Corollary}
\newtheorem{lemma}[theorem]{Lemma}
\newtheorem{observation}[theorem]{Observation}
\theoremstyle{definition}
\newtheorem{definition}{Definition}[section]

\usepackage{mathtools}

\usepackage[noend]{algpseudocode}
\usepackage{algorithm}
\usepackage{array}
\usepackage{tikz}

\newcommand{\open}{\mathit{open}}
\newcommand{\closed}{\mathit{generated}}

\newcommand{\htop}{\mathit{top}}
\newcommand{\hpop}{\mathit{pop}}
\newcommand{\hempty}{\mathit{empty}}

\newcommand{\safe}{\mathit{safe}}
\newcommand{\isGoal}{\mathit{isGoal}}
\newcommand{\true}{\mathit{true}}
\newcommand{\false}{\mathit{false}}

\newcommand{\tref}{$t_\mathit{ref}$}
\newcommand{\trefm}{t_\mathit{ref}}

\newcommand{\aSIPPab}{aSIPP$_{\alpha\beta}$}
\usetikzlibrary{arrows.meta, positioning}

\usepackage{color}

\usepackage{float}

\usepackage{subcaption}
\usetikzlibrary{positioning, shapes, calc}
\usepackage{MnSymbol}

\usepackage{authblk}   
\usepackage{orcidlink} 

\begin{document}

\title{Optimal Planning in a Dynamic World}

\author[1]{Devin Wild Thomas}
\author[2]{Solomon Eyal Shimony}
\author[1]{Wheeler~Ruml}
\author[3]{Erez~Karpas}
\author[2]{Shahaf~S.~Shperberg}
\author[4]{Andrew~Coles}
 
\affil[1]{University of New Hampshire, USA}
\affil[2]{Ben-Gurion University of the Negev, Israel}
\affil[3]{Technion -- Israel Institute of Technology, Israel}
\affil[4]{King's College London, UK}
\affil[ ]{\small\texttt{dwt@cs.unh.edu}, \texttt{shimony@cs.bgu.ac.il},
  \texttt{ruml@cs.unh.edu}, \texttt{karpase@technion.ac.il},
  \texttt{shperbsh@post.bgu.ac.il}, \texttt{andrew.coles@kcl.ac.uk}}



\maketitle
\begin{abstract}

{\bf Background:} 
    We address the problem of planning when the set of feasible states or actions changes over time.  For example, in the problem of path planning among moving obstacles (sometimes known as SIPP), the feasibility of being at a particular location can change as the obstacles move.  Or, the action of boarding a particular train is feasible only while it is stopped at the station.  This dynamism means that the optimal plan and its duration can change depending on when execution begins.  
In practice, execution start time is often unknown until planning has completed or another agent gives the go-ahead. However, most prior planning work either ignores dynamism or assumes a known start time.  This makes it straightforward to assess state and action feasibility but is impractical for some applications.
    
    {\bf Objectives:}
 In this paper, we relax the assumption of a known start time.  We define the setting of {\em any-start-time planning} and provide algorithms for it.

    {\bf Methods:}
   We present a data structure called a compound arrival time function (cATF) that compactly encodes the optimal plan as a function of start time.  We provide general-purpose planning algorithms, based on heuristic graph search, that assemble cATFs by propagating functions along edges instead of scalar costs.

    {\bf Results:}
   We prove that the size of a cATF is at most linear in the problem size.  An experimental evaluation of an implementation for the specific problem of SIPP shows that, on difficult problems, agents that rely on replanning often fail, while any-start-time algorithms using cATFs can quickly look up the optimal plan once the execution start time is known.
    
    {\bf Conclusions:}
    By enabling efficient representations and reasoning for time-dependent plans, this work provides a foundation for planning in dynamic worlds.
\end{abstract}

\section{Introduction}

Many applications of planning involve a dynamic world.  
For example, buses and trains move from stop to stop and shopkeepers open and close their stores without any action on our part.
Buses and trains can only be boarded after they arrive and before they depart and groceries can only be purchased when the store is open, so in general, the optimal plan can depend on when execution will begin.
Traditionally, planning algorithms have either assumed a known execution start time or assumed that the world is static except for the agent, so its plans will forever remain valid.
However, this assumption of a static world or a known start time is frequently violated in real-world settings.
Examples include when the start of execution must wait for planning to complete,
for some external event to occur, or
for the agent to receive permission to act.
In such situations, it can be preferable to have a valid plan at the ready, rather than needing to wait until the desired start time is known to even begin planning.

In this paper we describe how to plan for all start-times. 
The resulting any-start-time plan has many uses. It enables reasoning about the effects of delays, both during planning and online during execution.
In order to apply this approach, we simply need to be willing to represent the problem using a state space where states and actions are either safe or unsafe in continuous time.  
In return for using this state-space representation, our planner efficiently finds an any-start-time plan --- a set of optimal plans for every execution start-time.  


The dynamic world can be viewed as a source of uncertainty for the planner, one approach to handling uncertainty in planning is to plan for all eventualities. 
This approach, called universal planning, has been considered ``an (almost) universally bad idea'' \cite{ginsberg1989universal} because universal plans generally grow exponentially with problem size.
The time-dependent shortest path problem (TDSP) \cite{dreyfus1969} is one example of universal planning; it handles uncertainty in the form of an unknown start time.
In the TDSP, the planner's objective is to find shortest duration paths for all possible departure times of the agent in a graph with time-dependent edge costs.
When the TDSP edges are first-in first-out (i.e., an agent departing later can not arrive earlier), the number of optimal plans in the any-start-time policy solving a TDSP is super-polynomial ($n^{\Theta(\log(n))}$) in the size of the graph \cite{foschini2014}.  But hang with us for a minute.

One popular planning setting with time-dependent edge costs is so-called safe interval path planning (SIPP) \cite{phillips2011sipp}, where being in a state or taking an action is either safe or unsafe depending on the time. 
(Despite the name, nothing in SIPP is specific to path planning or to safety and, as we explain further below, it applies very generally to planning with dynamic action applicability.)
While previous SIPP algorithms plan for only a single start time, in this paper we introduce \textit{any-start-time planning}, which is a generalization of the SIPP problem to plan for all start times, rather than assuming a single known start time.
Any-start-time planning is a subclass of the TDSP that has better solution complexity.
While the general TDSP can have a solution that is super-polynomial, we will see below that any-start-time planning is a subset of universal planning that is actually a good idea, because the resulting universal plan grows only linearly with problem size. 
As time also progresses linearly, this means it can be applied to defer the commitment to an execution start time to be online, rather than an offline prerequisite to planning.
Figure~\ref{fig:sipptdsptopo} illustrates the relationship between SIPP and the TDSP, with our any-start-time planning being the linear complexity overlap.
In this paper, we present general representations and algorithms for handling planning domains in which the applicability of actions can change over time and we show via a specific example implementation for grid path planning that they can be more effective in dynamic worlds than replanning.

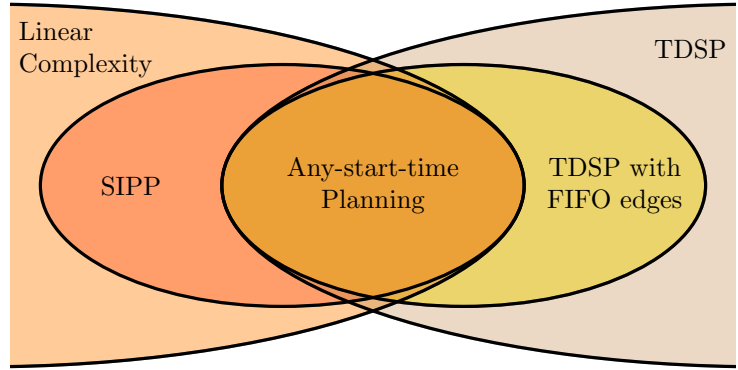
\begin{figure}[t]
    \centering
    \begin{tikzpicture}[scale = 0.8]
      \clip (-4.5,-3.1) rectangle (7.5,3.1);

      \draw[fill=red, opacity=0.3, very thick] (0,0) ellipse [x radius=4, y radius=2];
      \draw[fill=yellow, opacity=0.5, very thick] (3,0) ellipse [x radius=4, y radius=2];

      \draw[fill=orange, opacity=0.4, very thick] (-5,0) ellipse [x radius=9, y radius=3];
      \draw[fill=brown, opacity=0.3, very thick] (8,0) ellipse [x radius=9, y radius=3];

      \draw[very thick] (0,0) ellipse [x radius=4, y radius=2];
      \draw[very thick] (3,0) ellipse [x radius=4, y radius=2];

      \draw[very thick] (-5,0) ellipse [x radius=9, y radius=3];
      \draw[very thick] (8,0) ellipse [x radius=9, y radius=3];

      \node at (-2.5,0) {SIPP};
      \node[align = center] at (5.5,0) {TDSP with\\FIFO edges};
        \node[align = left] at (-3.25,2.25) {Linear\\Complexity};
        \node[align = right] at (6.75,2.25) {TDSP};
      \node[align=center] at (1.5,0) {Any-start-time\\Planning};

    \end{tikzpicture}
    \caption{Typology of SIPP and TDSP Planning Problems.}
    \label{fig:sipptdsptopo}
\end{figure}

\subsection{Running Example: Food Delivery Robot}

Consider a food delivery robot planning its delivery route (Figure~\ref{fig:delivery}, left panel). 
The robot has two roads that it can choose between, taking 10 and 20 minutes respectively, to reach the delivery target.
However, the faster of the two roads is blocked from 3:20 to 3:40 by a freight train.
The robot must wait for a human chef to finish preparing the meal before it can be delivered. 
If we know ahead of time when the food will be ready, i.e., the starting time, we can find the optimal plan assuming departure at that time.
This can be done by searching our road network for the earliest possible arrival time at each location, until we reach the goal. 
But what if we lack clairvoyance and do not know the departure time in advance?  
If we want to take advantage of the idle time before the cooking is complete, we can consider a range of possible departure times.
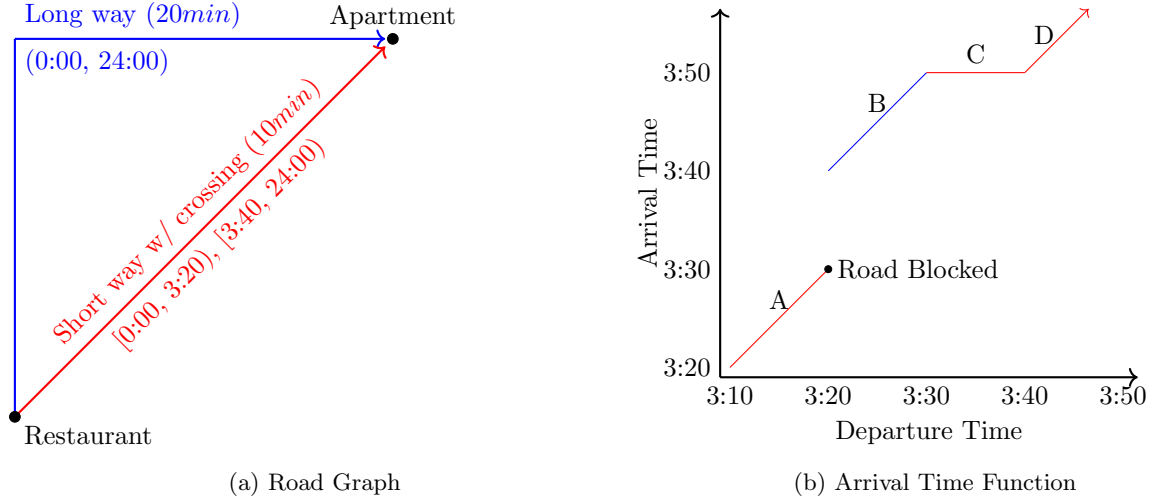
\begin{figure}[t]
  \centering
\begin{subfigure}{0.49\textwidth}
\begin{tikzpicture}[scale = 1]
    \draw[thick,blue] (0,0) -- (0,5) node[anchor=south west] {Long way ($20 min$)};
    \draw[thick,blue,->] (0,5) node[anchor = north west] {(0:00, 24:00)} -- (4.9,5);
    \draw[thick,red,->] (0,0) -- (4.9,4.9);
    \draw[red] (2.5,2.5) -- (2.5, 2.5) node[anchor=south, rotate=45] {Short way w/ crossing ($10 min$)};
    \draw[red] (2.5,2.5) -- (2.5, 2.5) node[anchor=north, rotate=45] {[0:00, 3:20), [3:40, 24:00)};
    \filldraw[black] (0,0) circle (2pt) node [anchor=north west] {Restaurant};
    \filldraw[black] (5,5) circle (2pt) node [anchor=south] {Apartment};
\end{tikzpicture}  
\caption{Road Graph}
\label{fig:deliveryrg}
\end{subfigure}
\begin{subfigure}{0.49\textwidth}
  \begin{tikzpicture}[scale = 0.65]
    \draw[thick,->] (0,0) -- (8.5,0) node[anchor=north west] {};
    \draw[thick,->] (0,0) -- (0,7.5) node[anchor=south west] {};

    \draw [color =red] (0.2, 0.2) -- (2.2,2.2);
    \filldraw[black] (1.2,1.2) circle (0pt) node [anchor=south] {A};

    \draw [color =blue] (2.2,4.2) -- (4.2,6.2);
    \filldraw[black] (3.2,5.2) circle (0pt) node [anchor=south] {B};

    \draw [color =red] (4.2,6.2) -- (6.2,6.2);
    \filldraw[black] (5.2,6.2) circle (0pt) node [anchor=south] {C};

    \draw [->,color=red] (6.2,6.2) -- (7.5, 7.5);
    \filldraw[black] (6.6,6.6) circle (0pt) node [anchor=south] {D};

    \filldraw[black] (2.2,2.2) circle (2pt) node [anchor=west] {Road Blocked};

    \filldraw[black] (0.2,0) circle (0pt) node [anchor=north] {3:10};
    \filldraw[black] (2.2,0) circle (0pt) node [anchor=north] {3:20};
    \filldraw[black] (4.2,0) circle (0pt) node [anchor=north] {3:30};
    \filldraw[black] (6.2,0) circle (0pt) node [anchor=north] {3:40};
    \filldraw[black] (8.2,0) circle (0pt) node [anchor=north] {3:50};

    \filldraw[black] (0,0.2) circle (0pt) node [anchor=east] {3:20};
    \filldraw[black] (0,2.2) circle (0pt) node [anchor=east] {3:30};
    \filldraw[black] (0,4.2) circle (0pt) node [anchor=east] {3:40};
    \filldraw[black] (0,6.2) circle (0pt) node [anchor=east] {3:50};

    \filldraw[black] (4.25,-.7) circle (0pt) node [anchor=north] {Departure Time};
    \filldraw[black] (-1.1,3.75) circle (0pt) node [anchor=south, rotate=90] {Arrival Time};
    \end{tikzpicture}
    \caption{Arrival Time Function}
    \label{fig:deliveryatf}
    \end{subfigure}
    \caption{Food delivery robot example.  
}
  \label{fig:delivery}
\end{figure}

The right panel of Figure~\ref{fig:delivery} shows the optimal arrival time of the delivery robot, plotted as a function of departure time.
There is a family of trajectories, labeled as segment A in Figure~\ref{fig:delivery}, that represents the same plan along the faster road, just with different departure times. 
If the robot departs at 3:10, it can take the fast road to arrive at 3:20.
Similarly, if it departs at 3:15 it will arrive at 3:25.
However, the train blocking the faster road causes a breakpoint in the arrival time function at 3:20, where the segment labeled B corresponds to the family of plans in which the delivery robot is better off taking the slower road.
At 3:30 there is another breakpoint in the arrival time function, where segment C has a constant arrival time because it is faster for the delivery robot to wait for the train to pass and then take the faster road.  
At 3:40 the train crossing opens and the arrival time function of segment D once again increases proportionally to departure time.
This structure --- segments with slope of $0$ or $1$, corresponding to waiting or acting --- arises due to the binary nature of safety in our problem.

To plan for any start time efficiently, we represent the planning problem as having safe intervals on states and actions when it is possible for the agent to be safely in the state or begin executing the action.
We search using $g(t)$ and $f(t)$ values that are functions of time, rather than scalar values as is typical for A* and related searches.
The functions, track the cumulative effects of the deadlines and constraints imparted by each safe interval the plan encounters.
The result of planning is an any-start-time plan, which is itself a function of time.
Our approach to solving any-start-time planning is (with some nuance related to tie-breaking) a form of cost-algebraic A* \cite{edelkamp2005cost} where costs are functions of time that accumulate through function composition rather than scalar summation.

We begin by reviewing the foundational prior work for any-start-time planning: SIPP and the TDSP.  
We then present Augmented SIPP, a novel heuristic search algorithm that uses functions of time instead of scalars for the $g$ values of the SIPP search.
We then formally define the any-start-time planning problem and present RePEAT, an algorithm that solves the any-start-time planning problem by repeatedly running Augmented SIPP to build a set of optimal plans covering all departure times up to a monotonically increasing time horizon.
We then prove the optimality and complexity of these algorithms, adapting prior results for cost-algebraic heuristic search and the TDSP.
Finally, we evaluate the empirical performance of our methods and discuss some possible applications for any-start-time planning.

This paper greatly extends the preliminary work of \textcite{anystarttimesipp} and also incorporates elements from subsequent work by \textcite{rtsipp}, \textcite{thomas2025real}, and \textcite{Thomas_LCCRTS_2026}. 
Any-start-time planning has already been applied to multi-agent planning: to instantly replan (by replanning in advance) \cite{Hanou_Thomas_trains}, and to quantify and exploit flexibility in the agents' plans \cite{Hanou_flexsipp_2026}.

\section{Background}

We begin with SIPP, which concerns offline planning with moving obstacles and a known start time.  
We then describe the TDSP problem, which considers finding shortest paths on a graph with time-varying edge costs. 
While prior work has described situations where the TDSP can be solved efficiently, to our knowledge we are the first to observe that the subset of the TDSP with linearly-sized solutions corresponds to the SIPP state space that is popular in the planning community.
We call this intersection any-start-time planning.

\subsection{Safe Interval Path Planning}

SIPP \cite{phillips2011sipp} is a popular way to formalize the problem of navigating in a dynamic world (often moving obstacles).
In SIPP, the agent and obstacles move in discrete space and continuous time and the states and actions are either safe or unsafe at every point in time.
All future movements of dynamic obstacles in the environment are encoded by their impact on the time-dependent safety of states and actions.
We refer to this original continuous, infinite state space as the naive state space.
A naive state $s_{naive} = \langle x, t\rangle$, consists of a configuration $x$ and point in time $t$.
Likewise, a naive action $a_{naive} = \langle x, x', t\rangle$, consists of a source configuration $x$, destination configuration $x'$, and a departure time $t$, when the action begins.
Exploring this infinite naive state space is intractable, so \textcite{phillips2011sipp} discretize the naive state space into the SIPP state space by grouping neighboring safe states (and actions) together into safe intervals.    
In the SIPP state space, the time points in the naive states (and actions) are replaced by a continuous interval of times when the configuration (or action) is safe to occupy (or begin).

We depart from prior work in two ways. 
First, SIPP has generally been described as having safe intervals solely for states.
However, as is done in multi-agent pathfinding (MAPF) with continuous time \cite{andreychuk2022multi}, we also use safe intervals on actions.
The example in Figure~\ref{fig:boulder} shows the benefit of safe intervals on actions, as the safety of the agent depends on the action the agent takes: if it moves to the left, it will be safer for longer than if it moves to the right or stays in place. 
Generally, SIPP has been applied to problems that require some form of collision checking, which implicitly performs the same role as the action safe intervals.  
Although we do not explicitly consider how to generate safe intervals in this paper, the version of SIPP with safe intervals on both states and actions may be easier in many cases to generate safe intervals for than the version of SIPP that only has intervals on states.

Second, SIPP has generally been applied in the context of grid pathfinding.
However, this SIPP representation is more broadly applicable than the terms safe and path planning would suggest.
While a safe interval can often literally mean system safety (i.e., collision-free), it can be interpreted more broadly as representing any sort of dynamic applicability of, expiration of, or deadlines on, states and actions.
SIPP applies to any setting that can be described by binary safety of configurations and/or actions in continuous time. 
SIPP is a state-of-the-art approach in settings such as: 
MAPF\cite{Andreychuk_improveCCBS_2021,Li_mapfLNS2_2022, stern2019multi},
railway planning \cite{Hanou_Thomas_trains,li2021scalable},
and planning for robotic manipulators \cite{Kerimov_manipsipp_2025,li2019safe}.
Our novel enhancement of SIPP, Augmented SIPP (see Section~\ref{sec:augsipp}) is a drop in replacement for SIPP which returns extra information with no overhead (see Section~\ref{sec:asippoverhead}).

\subsubsection{Problem Statement: Safe Interval Path Planning}
In SIPP, we use a compact representation of the state space, where states and actions both have safe intervals.
A safe interval $i$ is denoted $i = [t_b, t_e) \in I$, where 
$t_b\in \mathbb{R}^{\geq 0}$ is the beginning of the safe interval, 
$t_e\in  \mathbb{R}^{>0} \cup \{\infty\}$ is the end, 
and $I$ is the subset of $\mathbb{R}^2$ corresponding to the space of valid safe intervals.\footnote{While it is not necessary to constrain SIPP to only right-open intervals, we have done so for clarity as it yields a consistent set of edge cases when intervals interact.}
For clarity, when referring to a state $s$ with interval $i$ we will use the shorthand $s_b$ to refer to the beginning of $s$'s safe interval ($t_b(i(s))$); likewise, for action $a$ the shorthand $a_e$ refers to the end of its safe interval.  
In this paper we assume that all safe intervals have the following properties, they are:

right-open and have positive duration,
\begin{equation}\label{eqn:sipprightopen}
    i = [t_b, t_e) = \{t \in  \mathbb{R}^{>0} | t_b \leq t < t_e\};
\end{equation}

continuously safe
\begin{equation}
           \forall t \in i: \safe(t);
\end{equation}

and maximal, meaning that the state/action is unsafe immediately before the interval or there is no before \begin{equation}
    t_b = 0 \lor \exists \delta \in \mathbb{R}^{>0}: \forall \epsilon \in [\delta, 0): \lnot \safe(t_b - \epsilon)
\end{equation}

and immediately after the interval
\begin{equation}
    \lnot \safe(t_e).
\end{equation}

A SIPP problem is a tuple $\langle S, \sigma, \delta, x_o, t_0, \isGoal \rangle$ where:
\begin{description}
    \item[$S$] is the state space,
    \item[$\sigma$] is the mapping of states to their set of successor edges,
    \item[$\delta : S \times S \rightarrow \mathbb{R}^{>0}$] is the edge duration function,
    \item[$x_o$] is the origin configuration,
    \item[$t_0$] is the starting time, and 
    \item[$\isGoal: S \rightarrow \{\true, \false\}$] is the goal predicate.
\end{description}
A state $s = \langle x, i \rangle \in S$ has configuration $x$ and safe interval $i$.
Each state represents an interval of time ($i$) that the configuration ($x$) is safe.
A configuration represents the atemporal component of the state, such as the location of the agent (e.g., the restaurant or apartment in our example in Figure~\ref{fig:delivery}, the configuration of a robotic arm, or a set of true fluents in PDDL planning).
Many states may be associated with the same configuration, but they will have different non-overlapping safe intervals.
For example, if the restaurant our food delivery robot picks up from closes each night and the robot is not allowed to stay at the restaurant while it is closed, then each day the restaurant is open would be represented its own state with a safe interval from the start to the end of business.\footnote{While a state representation that allows a single state to incorporate multiple safe intervals corresponding to the same configuration is certainly possible,
we find it clearer to allow only one interval per state.}

An edge $e = \langle u, v, i\rangle \in \sigma(u)$ has a source state $u$, a destination state $v$, and an interval $i$ during which it is safe to depart from $u$ to $v$.
We use an action $\langle u, v, t\rangle$ to indicate a transition from state $u$ to state $v$ starting at time $t$.
The cost of an action is its non-negative duration $\delta(u,v)$: an agent departing $u$ at time $t$ will arrive at $v$ at $t + \delta(u,v)$.
Wait actions $\langle u, u, t \rangle$ may have any non-negative duration, subject to the constraints of $u$'s safe interval.
All other actions have a time-independent constant duration.

The existence of an action implies that there exists a time $t$ within the action's safe interval such that the source is safe at $t$, the edge traversed by the action is safe at $t$, and the destination is safe at $t + \delta(u,v)$.
A plan $p$ is a sequence of SIPP edges, $p_1,\ldots,p_n$, where the destination of $p_i$ is the source of $p_{i+1}$.
A plan is \textbf{valid} at time $t$ if and only if an agent departing on or after time $t$ can traverse the edges $p_{1 \ldots n}$ safely, specifically no edge $p_i$ will make the agent arrive too late to be safe at the destination of $p_i$ and then safely take the next edge $p_{i+1}$.
The objective in SIPP is to find a shortest valid plan from the start configuration $x_o$ at $t_0$ to a state $s_g$ where $\mathit{isGoal}(s_g)$. 
SIPP is an offline problem: we assume that there is ample time to complete planning before execution begins at time $t_0$.
Furthermore, SIPP models an omniscient agent with complete knowledge of its environment, including state and edge safety (e.g., moving obstacles) now and in the future.  We retain these assumptions throughout this paper.

\subsubsection{SIPP Example: 1D Moving Obstacle}
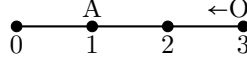
\begin{figure}[t]
  \centering
  \begin{tikzpicture}[scale = 1]
    \filldraw[black] (0,0) circle (2pt) node [anchor=north] {$0$};
    \filldraw[black] (1,0) circle (2pt) node [anchor=north] {$1$};
    \filldraw[black] (2,0) circle (2pt) node [anchor=north] {$2$};
    \filldraw[black] (3,0) circle (2pt) node [anchor=north] {$3$};
    \filldraw[black] (2.8,0) circle (0pt) node [anchor=south] {$\leftarrow$O};
    \filldraw[black] (1,0) circle (0pt) node [anchor=south] {A};
    \draw[thick,black] (0,0) -- (3,0);
  \end{tikzpicture}    
  \caption{Point obstacle $O$ moves with velocity $-1/s$, approaching point agent $A$.}
  \label{fig:boulder}
\end{figure}

A simple example of a SIPP problem instance is shown in Figure~\ref{fig:boulder}.
An obstacle $O$ approaches from the right of the agent $A$, moving to the left with unit speed.
$A$ is capable of moving at half the speed of $O$.  
The safe intervals of states represent when it is safe to be in the corresponding location, while the edge intervals represent when it is safe to initiate movement actions between them. 
In this example, the three SIPP states of interest are: $s_0 = \langle 0, [0, 3) \rangle$, $s_1 = \langle 1, [0, 2) \rangle$, and $s_2 = \langle 2, [0, 1) \rangle$, where the ends of the safe intervals correspond to the arrival times of the obstacle. 
The potential edges corresponding to the agent's routes to leave $s_1$ are: $e_{1, 0} = \langle s_1, s_0, [0, 1) \rangle$ (the obstacle arrives at $t=1$, marking the end of the safe interval), and $e_{1, 2} = \langle s_1, s_2, [0, 0) \rangle$ (leaving on or after $t=0$ will result in the agent being clobbered by the obstacle before finishing the move). 
Because $e_{1, 2}$ is never safe, only $e_{0, 1}$ is retained as a valid edge in the SIPP graph.

\subsubsection{Solving SIPP}

The fundamental observation used when solving SIPP is the temporal dominance relation (adapted from \textcite{phillips2011sipp}): 
\begin{theorem}\label{thm:sippdom}
    In a SIPP problem, arriving at a SIPP state as early as possible is optimal.
\end{theorem}
\begin{proof}
    Recall that each SIPP state $s \in S$ has a single maximal interval $i(s)$.
    Suppose that an optimal plan $p$ arrives at $s$ at time $t\in i(s)$. Consider another plan arriving at an earlier time $t'\in i: t' < t$. Then a plan $p'$ arriving at time $t'$ can just wait until $t$ and still be optimal by following the remainder of $p$.
\end{proof}
The SIPP algorithm of~\textcite{phillips2011sipp} exploits this by tracking the earliest arrival time at each state as the $g$-value of the nodes of an A$^*$ search.
This efficiently prunes the search to the compact SIPP state space rather than the infinite $\langle x, t\rangle$ naive state space of these problems.
A SIPP search node is a tuple $\langle s, g, h, p \rangle$ where: 
\begin{description}
    \item[$s$] is a state $\langle x , i \rangle$,
    \item[$g$] is the time the agent would arrive in $s$ if it followed the plan obtained by tracing back parent pointers,
    \item[$h$] is the scalar heuristic time-to-go from s, and
    \item[$p$] is a parent pointer.
\end{description}

\begin{algorithm}[t]
\caption{Safe Interval Path Planning (adapted from \textcite{phillips2011sipp})}
\label{alg:sipp}
\begin{algorithmic}[1]
\Function{SIPP}{$\sigma, \delta, x_o, t_0, \isGoal$}
\State $s_o \gets \langle x_o, \langle t_0, \infty \rangle\rangle$
\State $\open \gets \{\langle s_o, t_0, h(s_o), \varnothing \rangle\}$ \label{algln:sippopen}
\State $\closed \gets \{s_o \mapsto \htop(\open)\}$  \label{algln:sippclosed}
\While{$\lnot \hempty(\open)$} \label{algln:sipploop}
    \State $n \gets \hpop(\open)$ \Comment{$\min f(n) = g(n) + h(n)$} \label{algln:sippexpand}
    \If{$\isGoal(n)$} 
        \State \Return $n$ \label{algln:sippsuccess}
    \EndIf
    \For{$e \in \sigma(n)$} \label{algln:sippsuccessorgeneration}
        \State $s' = v(e)$
        \State $g \leftarrow \max(g(n) + \delta(e), e_b  + \delta(e), s'_b )$\label{algln:sippg}
        \If{ $(s' \notin \closed \lor g < g(\closed[s'])) \land (g < s'_e)  \land (g - \delta(e) < e_e) \land (g - \delta(e) \geq n_e)$} \label{algln:sippprune}
            \If{$g < g[\closed[s']]$}
                \State $\open \gets \open \setminus \closed[s']$ \label{algln:sipppruneremove}
            \EndIf
            \State $n' \leftarrow \langle s', g, h(s'), n \rangle$  
            \State $\closed[s'] \leftarrow n'$  \label{algln:sippaddclosed}
            \State $\open \leftarrow \open \cup \{n'\}$ \label{algln:sippaddopen}
        \EndIf
    \EndFor
\EndWhile
\State \Return Failure \label{algln:sippfail}
\EndFunction
\end{algorithmic}
\end{algorithm}

The SIPP algorithm is fundamentally A* \cite{hart:fbh} using the SIPP state space, as shown in Algorithm~\ref{alg:sipp}. 
The node priority function, $f(n) = g(n) + h(n)$, is an estimate of the total plan duration to goal.
We initialize open with the starting state (Line~\ref{algln:sippopen}).  
The closed list, called `generated' is a mapping of states to nodes generated by the search. 
It is initialized with the start state and the corresponding search node (Line~\ref{algln:sippclosed}).

We repeatedly expand the lowest $f(n) = g(n) + h(n)$ node on open (Line~\ref{algln:sipploop}) until we have expanded a goal (Line~\ref{algln:sippsuccess}) or open is exhausted, and the problem has been proved unsolvable (Line~\ref{algln:sippfail}).
When a node $n$ is selected for expansion, we generate its successors (Line~\ref{algln:sippsuccessorgeneration}), with an earliest arrival time ($g$) that is either constrained by $g(n)$ and the cost of the edge, by the safe interval of the edge, or by the safe interval of the destination state (Line~\ref{algln:sippg}).
The exact nature of the successor generation $\sigma$ is domain dependent, but we only consider a successor valid if there is a non-negative wait and positive duration move.  
This is why the successor generation is on a search node, not a state because the available successors depend on the time $(g(n))$.
We then prune based on Theorem~\ref{thm:sippdom}, keeping the earlier arriving node (Line~\ref{algln:sippprune},\ref{algln:sipppruneremove}).
If the resulting plan is safe (Line~\ref{algln:sippprune}) then the successor is added to the closed list (Line~\ref{algln:sippaddclosed}) and open (Line~\ref{algln:sippaddopen}).

\subsection{Time-Dependent Shortest Paths}

We make the observation that SIPP problems are a special case of time-dependent shortest path (TDSP) problems on implicit graphs, where there is a single known execution time and a restricted set of edge costs.  
TDSPs model networks that have time-dependent edge costs.
They are frequently applied to roads, communication, and other networks where congestion can vary with the time of day.
\textcite{COOKE1966} introduced a discrete-time version of the TDSP, along with a modified Bellman-Ford algorithm to solve it.  
\textcite{dreyfus1969} observed that Dijkstra's algorithm \cite{dijkstra1959note} could be trivially applied.
Furthermore, \citeauthor{dreyfus1969} observed that it was unnecessary to discretize time, instead labeling each node with the upper bound on its earliest arrival-time.
This is identical to the state-dominance relation used by the SIPP algorithm \cite{phillips2011sipp} and the one-to-many algorithm described by \citeauthor{dreyfus1969} adapts Dijkstra's algorithm in the same way that SIPP adapts A*.
We adapt the notation of \textcite{foschini2014}, which is based on the arrival time functions used by \textcite{ordarom1990}.

An instance of the TDSP consists of a tuple $\langle V, E, d, s_o\rangle$, where:
\begin{description}
    \item[$V$] is the set of vertices in the network,
    \item[$E$] is the set of edges in the network, 
    \item[$d$] is the delay function, where $d[e](t)$ is the time-dependent duration of edge $e$, and
    \item[$s_o$] is the origin vertex.
\end{description} 
The objective is to find the shortest paths from $s_o$ to all other vertices in $V$, for all time $t \in \mathbb{R}$.
The delay $d$ is a functional that returns the edge specific delay function $d[e]$. 
We denote the evaluation of this and other functions that return functions with square braces, while functions that return numbers evaluate with parentheses.  
For example, $d[e]$ is the entire delay function for $e$ while $d[e](t)$ is the specific edge duration value when departing at time $t$.
While the TDSP permits arbitrary delay functions,  often $d[e]$ is limited to be first-in-first-out (FIFO), which requires that the delay never have a slope less than $-1$, as that would result in a later departure arriving before an earlier departure.
It is also common to consider the subset of problems with piecewise linear delay functions.

Our focus is on the TDSP with time windows on nodes and edges \cite{halpern1974,halpern1977}.
Time windows are similar to the safe intervals defined for SIPP, however, their interpretation has differed when it comes to windows of states.  
Most recent work is focused on the case where the node windows are `parking bans': such a node can be temporarily closed for parking (waiting), but
may still be traversed.
\textcite{halpern1977} describes a manipulation to their graph that reduces the TDSP with safe intervals on nodes (which \citeauthor{halpern1977} called ``no passing-through periods'') to a TDSP with time windows by replacing each node with a pair of nodes connected by a single edge corresponding to the safe interval.
We note that the combination of the nodes with no-passing-through periods described by \citeauthor{halpern1977} and the edge functions of \textcite{halpern1974} that alternate between a constant duration and infinite duration directly corresponds to SIPP.

\textcite{sancho1992} provides a dynamic programming algorithm for the problem described by \citeauthor{halpern1974}.
\textcite{ordarom1990} formalize the problem of \textcite{halpern1977} and provides algorithms for a variety of extensions and modifications to the problem.
\textcite{foschini2014} study the complexity of the problems presented by \textcite{ordarom1990}.  
They prove that when edge delays are polynomial-sized piecewise-linear functions, the solution complexity is super-polynomial ($n^{\Theta(\log(n))}$ where $n$ is the number of vertices in the network).
However, they also show that the complexity of the solution is linear in the size of the network for a restricted family of edge slopes (we show this includes the SIPP case of $\{0, 1, \infty\}$, see Section~\ref{sec:complexity}).
\textcite{ElSherbeny2014} give an A* algorithm for the TDSP with time windows.
Modern work on TDSP has focused on rapidly answering queries in logistics domains, for instance, trucking with driving bans \cite{vandertuin2018,zhang2020} and railway scheduling \cite{haehn2020}.
To our knowledge, while settings equivalent to any-start-time planning have been suggested in remarks on the TDSP by \textcite{halpern1977} and \textcite{foschini2014}, such settings have not been specifically explored in prior work.

Like \textcite{ordarom1990} and \textcite{foschini2014}, we will reason in terms of earliest arrival time functions (ATFs) rather than delays. 
An earliest arrival time function (ATF) is a function $\mathbb{R} \rightarrow \mathbb{R} \cup \{\infty\}$ that assigns for each departure time the earliest arrival time.
We follow the notation of \citeauthor{foschini2014} and denote three types of functionals that map to ATFs:
\begin{description}
    \item[Edge ATFs] ($A[e]: \sigma(S) \rightarrow (\mathbb{R} \rightarrow \mathbb{R} \cup \{\infty\})$) map a single edge to the ATF for traversing the edge,
    \item[Plan ATFs] ($A[p]: S^n \rightarrow (\mathbb{R} \rightarrow \mathbb{R} \cup \{\infty\})$) map a single plan to the ATF for traversing the entire plan, and
    \item[Compound ATFs] ($A[s, d]: S \times S \rightarrow (\mathbb{R} \rightarrow \mathbb{R} \cup \{\infty\})$) map two states ($s$ and $d$) to the ATF of the set of optimal plans from $s$ to $d$.
\end{description}
In the following sections, we will describe each type of ATF in greater detail and explain how they can be used for any-start-time planning.

\section{Heuristic Search with Arrival Time Functions}\label{sec:augsipp}

With these ingredients in hand, in this section we describe Augmented SIPP, a generalization of the SIPP algorithm that returns a plan ATF along with the plan normally returned by SIPP.
Augmented SIPP is the inner loop of RePEAT, our algorithm for any-start-time planning.
However, it is also useful in its own right, as it provides more information than a standard SIPP search, such as when the plan will no longer be safe, and it illustrates how heuristic search can work with $g$ values that are functions.

Because it is built on the SIPP dominance relation (Theorem~\ref{thm:sippdom}), a SIPP search results in a plan where the agent departs as early as possible from each intermediate state.  
For example, if we imagine our delivery robot, it would be tailgating other cars and crossing the train tracks the moment the train clears.  
While this is one earliest arriving plan, it is not necessarily uniquely optimal.  
If the agent will have to wait for a traffic light after the train crossing anyway, there is no difference in arrival time between waiting before crossing the tracks or after.
In general, the sequence of actions taken by a SIPP plan is valid, or possibly even optimal, for more than just the single departure time that was planned for.
Augmented SIPP captures the whole family of related plans that have different departure times but pass through the same set of states and vary only in when the agent waits and for how long.
We do this by augmenting the SIPP search to use a plan ATF $g$-value, rather than the usual scalar $g$.

We begin by showing how to create an Augmented SIPP graph for a SIPP problem, this operation is the compilation from a SIPP problem (with safe intervals) to a TDSP problem (with ATFs).  
The Augmented SIPP graph is identical to the SIPP problem except that, for each edge, we compute an associated edge ATF based on the three safe intervals of 1) the edge, 2) its source state, and 3) its destination state.
If desired, the edge ATFs can just as easily be computed lazily as required by an implicit search. 
We then show how to  accumulate costs by calculating plan ATFs in the Augmented SIPP graph as the composition of the ATFs of the edges in the plan.
We then give the algorithm for Augmented SIPP, and show that it is complete, optimal, and optimally efficient because it is a type of Cost-algebraic A*\footnote{If the safe intervals are inexact, (e.g., an irregular obstacle modeled by a bounding circle) SIPP may become incomplete with regard to the real problem being modeled.}.

\subsection{Compiling the Augmented SIPP Graph: Calculating Edge ATFs}

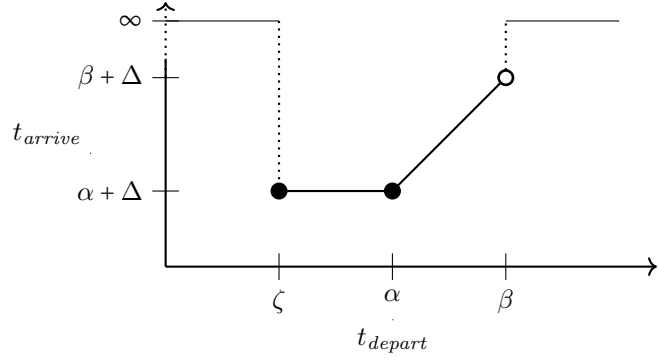
\begin{figure}[t]
    \centering
    \begin{tikzpicture}[scale=1]
    \draw[thick,->] (0,0) -- (6.5,0);
    \draw[thick] (0,0) -- (0,2.75);
    \draw[thick,dotted, ->] (0,2.5) -- (0,3.5);
    \draw (1.5 cm,5pt) -- (1.5 cm,-5pt) node[anchor=north] {$\zeta$};
    \draw (3 cm,5pt) -- (3 cm,-5pt) node[anchor=north] {$\alpha$};
    \draw (4.5 cm,5pt) -- (4.5 cm,-5pt) node[anchor=north] {$\beta$};
    \draw (5pt,3.25 cm) -- (-5pt,3.25 cm) node[anchor=east] {$\infty$};
    \draw (5pt, 1) -- (-5pt, 1) node[anchor=east] {$\alpha + \Delta$};
    \draw (5pt, 2.5) -- (-5pt, 2.5) node[anchor=east] {$\beta + \Delta$};
    \draw (-1, 1.5) -- (-1, 1.5) node[anchor=south east] {$t_{\textit{arrive}}$};
    \draw (3cm, -20pt) -- (3cm, -20pt) node[anchor=north] {$t_{\textit{depart}}$};
    \filldraw[black] (1.5,1) circle (3pt);
    \filldraw[black] (3,1) circle (3pt);
    \filldraw[black] (4.5,2.5) circle (3pt);
    
    \draw (0,3.25) -- (1.5,3.25);
    \draw [thick,dotted] (1.5,3.25) -- (1.5,1);
    \draw [thick] (1.5,1) -- (3,1);
    \draw [thick] (3,1) -- (4.5,2.5);
    \draw [thick,dotted] (4.5,2.5) -- (4.5,3.25);
    \draw (4.5,3.25) -- (6,3.25);
    \filldraw[white] (4.5,2.5) circle (2pt);
    \end{tikzpicture}
    \caption{An edge ATF with parameters $\langle \zeta,\alpha,\beta, \Delta \rangle$.}
    \label{fig:functional}
\end{figure}

Figure~\ref{fig:functional} shows the typical form of an ATF $A[e]$ for an edge $e \in E$ (or plan $p \in E^N$) in a SIPP state space.
$A[e]$ is piecewise linear and is parameterized by four real-valued numbers: $\zeta, \alpha, \beta$, and $\Delta$ (denoted $A[e] = A[\zeta, \alpha, \beta, \Delta]$).
These parameters define the piecewise linear function:
\begin{equation}
    \label{eqn:edgearrival}
    A[\zeta,\alpha,\beta,\Delta ](t) = \begin{cases}
        \infty &t < \zeta\\
        \alpha  +  \Delta &\zeta \leq  t  <  \min(\alpha, \beta)\\
        t + \Delta &\alpha \leq  t  <  \beta\\
        \infty &\beta \leq t
    \end{cases}, 
\end{equation}
where:
\begin{description}
    \item[$\zeta$]  is when the safe interval of the start state of the edge begins,
    \item[$\alpha$] is the earliest departure time when the agent can safely begin to traverse the edge, 
    \item[$\beta$] is the latest departure time before one of the states or edge becomes unsafe, and 
    \item[$\Delta$] is the duration required to make the state transition represented by the edge.  
\end{description}
An arrival time of $\infty$ is used to represent an unsafe departure time.

The form of the second and third cases may appear strange.
However, this is because Equation~\ref{eqn:edgearrival} is used for both edge ATFs and plan ATFs, and for a plan $\beta \leq \alpha$ can be true, in which case $t \in [\zeta, \beta)$ triggers the second case as $\alpha \leq t < \beta$ never holds.  
An example of how this can happen appears at the end of Section~\ref{sec:pathatfs}.

An ATF is \textbf{valid} iff the following constraints on the parameters are met:
\begin{align}
    \zeta& \leq \alpha \label{eqn:nonnegwait},\\
    \zeta& < \beta \label{eqn:positiveduration}, \text{ and}\\
    \Delta& \geq 0 \label{eqn:nonnegcost}.
\end{align}
Equation~\ref{eqn:nonnegwait} means we do not consider negative wait durations, Equation~\ref{eqn:positiveduration} follows from our representation of safe intervals as having positive duration (i.e., they are right-open, see Equation~\ref{eqn:sipprightopen}), and Equation~\ref{eqn:nonnegcost} is because we can not have negative plan durations.
Equation~\ref{eqn:nonnegcost} is not intended to imply that actions can have zero duration, but is useful when describing the ATF of a plan with no actions (we explicitly define this `Identity ATF' in Equation~\ref{eqn:atfifdentity}).

\begin{figure}[t]
     \centering
     \begin{subfigure}{0.5\textwidth}
         \centering
    \begin{tikzpicture}[scale=1]
    \draw[thick,->] (0,0) -- (6.5,0);
    \draw (0,5pt) -- (0,-5pt) node[anchor=north] {$0$};
    \draw (2,5pt) -- (2,-5pt) node[anchor=north] {$2$};
    \draw (4,5pt) -- (4,-5pt) node[anchor=north] {$4$};
    \draw (6,5pt) -- (6,-5pt) node[anchor=north] {$6$};
    \draw (3, -.6) -- (3, -.6) node[anchor=north] {$t$};

    \draw[thick] (0, 3) node[anchor=east] { $u$} -- (6, 3); 
    \filldraw[black] (0, 3) circle (3pt);
    \filldraw[black] (6, 3) circle (3pt);
    \filldraw[white] (6, 3) circle (2pt);

     \draw[thick] (0, 2) node[anchor=east] { $e$} -- (4, 2); 
    \filldraw[black] (0, 2) circle (3pt);
    \filldraw[black] (4, 2) circle (3pt);
    \filldraw[white] (4, 2) circle (2pt);

    \draw[thick] (0, 1) node[anchor=east] { $v$} -- (0, 1);
    \draw[thick] (3, 1) -- (6, 1); 
    \filldraw[black] (3, 1) circle (3pt);
    \filldraw[black] (6, 1) circle (3pt);
    \filldraw[white] (6, 1) circle (2pt);
    \end{tikzpicture}
    \caption{Timeline of safe intervals.}
    \label{fig:ecetimeline}
     \end{subfigure}
     \begin{subfigure}[b]{0.45\textwidth}
        \centering
        \begin{tikzpicture}[scale=1]
        \draw[thick,->] (0,0) -- (5,0);
        \draw[thick,->] (0,0) -- (0,3);
        \draw (.5,5pt) -- (.5,-5pt) node[anchor=north] {$0$};
        \draw (2.5,5pt) -- (2.5,-5pt) node[anchor=north] {$2$};
        \draw (4.5,5pt) -- (4.5,-5pt) node[anchor=north] {$4$};
        \draw (5pt, 0.5) -- (-5pt, 0.5) node[anchor=east] {$3$};
        \draw (5pt, 2.5) -- (-5pt, 2.5) node[anchor=east] {$5$};
        \draw (-.45, 1.5) -- (-.45, 1.5) node[anchor= east] {$t_{\textit{arr}}$};
        \draw (2.5, -.6) -- (2.5, -.6) node[anchor=north] {$t_{\textit{dep}}$};
        \filldraw[black] (.5,.5) circle (3pt);
        \filldraw[black] (2.5,.5) circle (3pt);
        \filldraw[black] (4.5,2.5) circle (3pt);
        
        \draw [thick] (.5,.5) -- (2.5,.5) -- (4.5,2.5);
        \filldraw[white] (4.5,2.5) circle (2pt);
        \end{tikzpicture}
        \caption{$A[e]=\langle 0, 2, 4, 1\rangle$}
        \label{fig:eceatf}
     \end{subfigure}
        \caption{An example SIPP ATF compilation for edge $e\in E$ from $u \in S$ to $v \in S$.}
        \label{fig:edgecompileexample}
\end{figure}
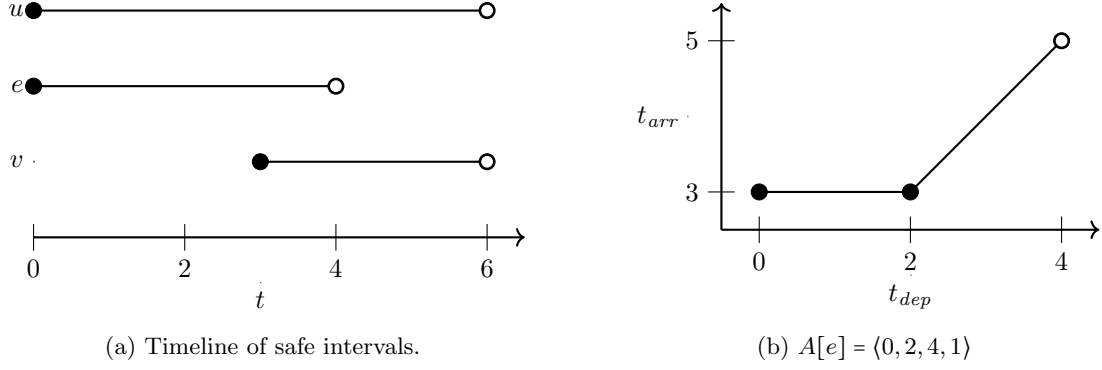

Figure~\ref{fig:edgecompileexample} shows an example compilation for a SIPP edge $e \in E$, with source $u \in S$, destination $v \in S$ and unit duration.
The timelines of the three safe intervals are shown in Figure~\ref{fig:ecetimeline}.  State $u$ is safe for the whole time. 
However, $e$ becomes unsafe at $t=4$ and $v$ is not safe until $t=3$.
We calculate the four values of the edge ATF as follows: 
$\zeta = 0$ because that is when the source $u$ becomes safe, 
$\alpha = 2$ because we cannot reach $v$ until it is safe at $t=3$ and the edge $e$ takes 1 time unit to traverse,
$\beta = 4$ because the edge $e$ becomes unsafe at $t=4$, and
$\Delta = 1$ because $e$ has unit duration.
The resulting ATF is shown in Figure~\ref{fig:eceatf}. 

Formally, given a SIPP problem, for every pair of states $u, v \in S$ with edge $x \in \sigma(u) \land v(x) = v$, we augment the SIPP graph with the corresponding edge ATFs with parameters:
\begin{align}
    &\zeta = u_b\label{eqn:zetacompute}\\
    &\begin{aligned}
    \alpha = max(&x_b, u_b, v_b - \delta(u, v)) \label{eqn:alphacompute} 
    \end{aligned}\\
    &\begin{aligned}
    \beta = min(&x_e,  u_e,  v_e - \delta(u, v)) \label{eqn:betacompute}
    \end{aligned}\\
    &\Delta = \delta(u, v)\label{eqn:deltacompute}
\end{align}
The resulting Augmented SIPP graph has the information of all the safe intervals of states and edges in the graph compiled into one edge ATF for each edge in the SIPP graph.  As usual in state-space search, in cases where the state space is defined implicitly (by a simulator, for example), one might want to construct this graph in a lazy as-needed manner.

\subsection{Augmented SIPP Plan Costs: Calculating Path ATFs}
\label{sec:pathatfs}

The structure of an edge ATF as illustrated in Figure~\ref{fig:functional} consists of a constant arrival time segment from $\zeta$ to $\alpha$ with arrival time $\alpha + \Delta$, which represents a set of optimal plans for different start times.
After this, the ATF increases linearly with the departure time, meaning that no waiting needs to be done along the plan until it is no longer safe at $\beta$.
This structure (Equation~\ref{eqn:edgearrival}, illustrated in Figure~\ref{fig:functional}) is maintained under composition.
Consider three plans, $a, b,$ and $c$, where the final state of plan $a$ is the initial state of $b$, and $c$ is exactly the concatenation of the plans $a$ and then $b$ (i.e., the same sequence of edges without any extras).  
We denote the composition operation in the usual way, so:
\begin{equation}
    \label{eqn:path2}
    A[c](t) =  (A[b] \circ A[a])(t) = A[b](A[a](t))
\end{equation}
Unfortunately for our use case, the notation of function composition results in the plan reading right-to-left.  
For simplicity, when not referring to a specific start time,
we usually omit the argument $t$.

$A[c]$ has parameters:
\begin{align}
    \zeta(c) &= \zeta(a)  \label{eqn:pathzeta}\\
    \alpha(c) &= max(\alpha(a), \alpha(b) - \Delta(a))  \label{eqn:pathalpha}\\
    \beta(c) &= min(\beta(a), \beta(b) - \Delta(a)) \label{eqn:pathbeta}\\
    \Delta(c) &= \Delta(a) + \Delta(b) \label{eqn:pathdelta}
\end{align}
For convenience, we define the Shifted Identity ATF, notated as:
\begin{equation}\label{eqn:atfifdentity}
    I[\mathit{shift}, t_0](t) = t + \mathit{shift} = A[t_0, t_0, \infty, \mathit{shift}](t)
\end{equation}
which for all $t \in [t_0, \infty)$ triggers the third case of Equation~\ref{eqn:edgearrival}.
The ATF for a plan of no edges is $I[0,t_0](t)$, which corresponds to a plan with zero duration where the earliest arrival time is the departure time.
This ATF is called the Identity ATF, because like $0$ is the identity of plan costs that are summed in traditional A*, the Identity ATF is the identity of plan costs that are the composition of plan ATFs.

For a single edge, all waiting must be done prior to taking the action to traverse the edge.
However, for a multistep plan, the agent can sometimes be required to wait at intermediate points.  
While it cannot be forced to wait for longer than the edge is safe, it can be forced to wait along an entire plan for longer than it can wait prior to traversing the first edge.
Formally, this means that while $\alpha < \beta$ is true for all edge ATFs individually, it need not be true for a plan ATF.

\begin{figure}[t]
     \centering
     \begin{subfigure}{\textwidth}
         \centering
    \begin{tikzpicture}[scale=1]
    \draw[thick,->] (0,0) -- (10.5,0);
    \draw (0,5pt) -- (0,-5pt) node[anchor=north] {$0$};
    \draw (2,5pt) -- (2,-5pt) node[anchor=north] {$2$};
    \draw (4,5pt) -- (4,-5pt) node[anchor=north] {$4$};
    \draw (6,5pt) -- (6,-5pt) node[anchor=north] {$6$};
    \draw (8,5pt) -- (8,-5pt) node[anchor=north] {$8$};
    \draw (10,5pt) -- (10,-5pt) node[anchor=north] {$10$};
    \draw (5, -.6) -- (5, -.6) node[anchor=north] {$t$};

    \draw[thick] (0, 1) node[anchor=east] {$e_1$} -- (2, 1); 
    \filldraw[black] (0, 1) circle (3pt);
    \filldraw[black] (2, 1) circle (3pt);
    \filldraw[white] (2, 1) circle (2pt);

     \draw[thick] (0, 2) node[anchor=east] {$e_2$} -- (10, 2); 
    \filldraw[black] (0, 2) circle (3pt);
    \filldraw[black] (10, 2) circle (3pt);
    \filldraw[white] (10, 2) circle (2pt);

    \draw[thick] (0, 3) node[anchor=east] {$e_3$} -- (0, 3);
    \draw[thick] (8, 3) -- (10, 3); 
    \filldraw[black] (8, 3) circle (3pt);
    \filldraw[black] (10, 3) circle (3pt);
    \filldraw[white] (10, 3) circle (2pt);
   
    \end{tikzpicture}
    \caption{Timeline of Edge Safe Intervals.}
    \label{fig:edgeatfexample}
     \end{subfigure}
     \begin{subfigure}[b]{0.3\textwidth}
        \centering
        \begin{tikzpicture}[scale=1]
        \draw[thick,->] (0,0) -- (3,0);
        \draw[thick,->] (0,0) -- (0,5);
        \draw (.5,5pt) -- (.5,-5pt) node[anchor=north] {$0$};
        \draw (2.5,5pt) -- (2.5,-5pt) node[anchor=north] {$2$};
        \draw (5pt, 0.5) -- (-5pt, 0.5) node[anchor=east] {$1$};
        \draw (5pt, 2.5) -- (-5pt, 2.5) node[anchor=east] {$3$};
        \draw (5pt, 4.5) -- (-5pt, 4.5) node[anchor=east] {$5$};
        \draw (-.25, 3.5) -- (-.25, 3.5) node[anchor= east] {$t_{\textit{arr}}$};
        \draw (1.5, -10pt) -- (1.5, -10pt) node[anchor=north] {$t_{\textit{dep}}$};
        \filldraw[black] (.5,.5) circle (3pt);
        \filldraw[black] (2.5,2.5) circle (3pt);
        
        \draw [thick] (.5,.5) -- (2.5,2.5);
        \filldraw[white] (2.5,2.5) circle (2pt);
        \end{tikzpicture}
        \caption{$A[e_1]=\langle 0, 0, 2, 1\rangle$}
        \label{fig:ae1}
     \end{subfigure}
     \hfill
     \begin{subfigure}[b]{0.3\textwidth}
        \centering
        \begin{tikzpicture}[scale=1]
        \draw[thick,->] (0,0) -- (3,0);
        \draw[thick,->] (0,0) -- (0,5);
        \draw (.5,5pt) -- (.5,-5pt) node[anchor=north] {$0$};
        \draw (2.5,5pt) -- (2.5,-5pt) node[anchor=north] {$2$};
        \draw (5pt, 0.5) -- (-5pt, 0.5) node[anchor=east] {$1$};
        \draw (5pt, 2.5) -- (-5pt, 2.5) node[anchor=east] {$3$};
        \draw (5pt, 4.5) -- (-5pt, 4.5) node[anchor=east] {$5$};
        \draw (-.25, 3.5) -- (-.25, 3.5) node[anchor= east] {$t_{\textit{arr}}$};
        \draw (1.5, -10pt) -- (1.5, -10pt) node[anchor=north] {$t_{\textit{dep}}$};
        
        \filldraw[black] (.5,1.5) circle (3pt);
        \filldraw[black] (2.5,3.5) circle (3pt);
        \draw [thick] (.5,1.5) -- (2.5,3.5);
        \filldraw[white] (2.5,3.5) circle (2pt);
        \end{tikzpicture}
        \caption{$A[e_1e_2]=\langle 0, 0, 2, 2\rangle$}
        \label{fig:ae2}
     \end{subfigure}
     \hfill
     \begin{subfigure}[b]{0.3\textwidth}
        \centering
        \begin{tikzpicture}[scale=1]
        \draw[thick,->] (0,0) -- (3,0);
        \draw[thick,->] (0,0) -- (0,5);
        \draw (.5,5pt) -- (.5,-5pt) node[anchor=north] {$0$};
        \draw (2.5,5pt) -- (2.5,-5pt) node[anchor=north] {$2$};
        \draw (5pt, 0.5) -- (-5pt, 0.5) node[anchor=east] {$5$};
        \draw (5pt, 2.5) -- (-5pt, 2.5) node[anchor=east] {$7$};
        \draw (5pt, 4.5) -- (-5pt, 4.5) node[anchor=east] {$9$};
        \draw (-.25, 3.5) -- (-.25, 3.5) node[anchor= east] {$t_{\textit{arr}}$};
        \draw (1.5, -10pt) -- (1.5, -10pt) node[anchor=north] {$t_{\textit{dep}}$};
        
        \filldraw[black] (.5,4.5) circle (3pt);
        \filldraw[black] (2.5,4.5) circle (3pt);
        \draw [thick] (.5,4.5) -- (2.5,4.5);
        \filldraw[white] (2.5,4.5) circle (2pt);
        \end{tikzpicture}
        \caption{$A[e_1e_2e_3]=\langle 0, 6, 2, 3\rangle$}
        \label{fig:ae3}
        \end{subfigure}
        \caption{Example demonstrating $\alpha > \beta$ in a plan ATF.}
        \label{fig:alphabetaexample}
\end{figure}
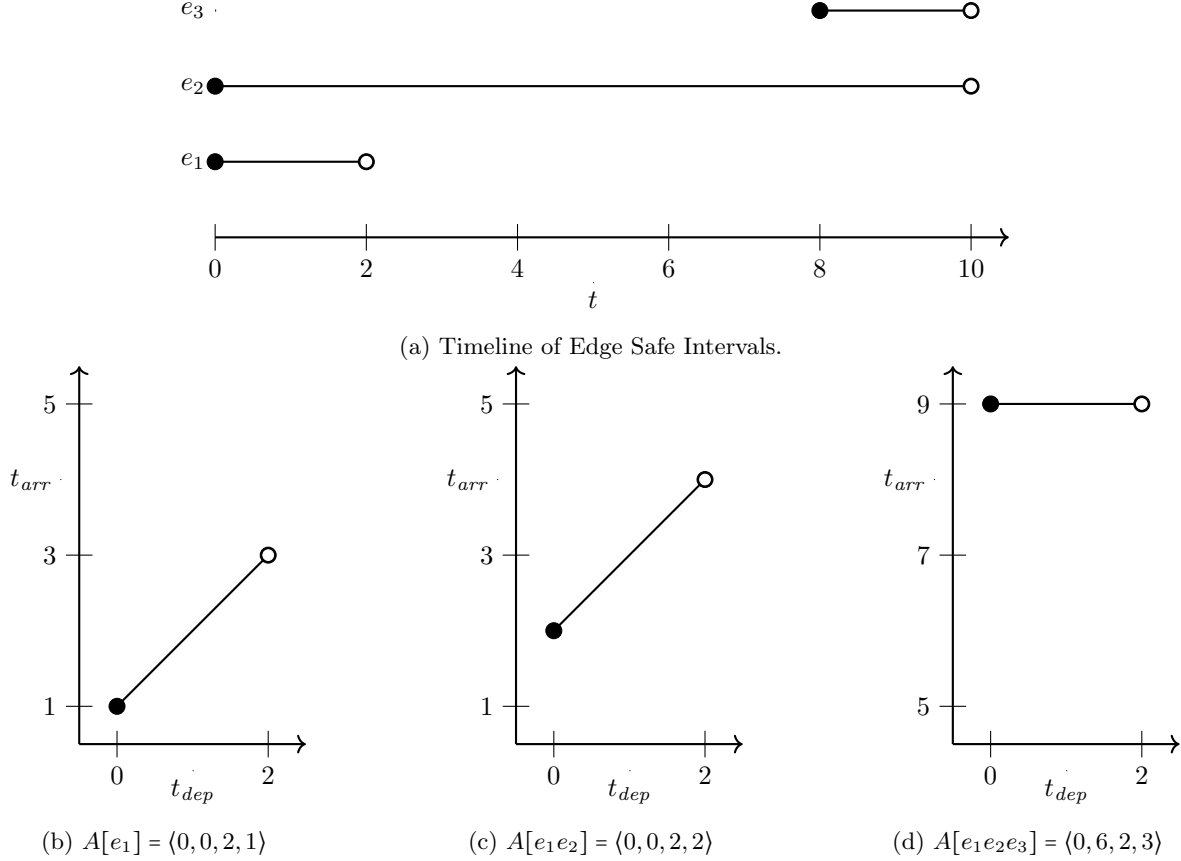

A toy example where $\beta < \alpha$ is shown in Figure~\ref{fig:alphabetaexample}.
Consider a plan ($e_1e_2e_3$) consisting of three edges ($e_1,e_2,e_3$), Figure~\ref{fig:edgeatfexample} shows the safe intervals of the three edges:
$e_1$ is safe in the interval $[0,2)$, $e_2$ in $[0,10)$, and $e_3$ in $[8,10)$.
If all edges have unit travel duration, the agent is forced to wait for at least six time units for the final edge to become accessible, but it can wait for a maximum of two time units before taking the first edge, thus our plan ATF is parameterized by $\langle 0, 6, 2, 3\rangle$, giving us the earliest arrival time of $6+3=9$.
We illustrate the composition step by step.  Initially, the ATF of $e_1$ (Figure~\ref{fig:ae1}) is constrained by the interval of $e_1$, $\zeta = \alpha = 0$ because the edge is immediately traversable with duration $\Delta = 1$, and $\beta = 2$ because that is when $e_1$ becomes unsafe.
The ATF of plan $e_1e_2$ (Figure~\ref{fig:ae2}), i.e., the plan through $e_1$ then $e_2$, is still constrained entirely by the safe interval of $e_1$, so the only difference is the increased duration from taking two edges, $\Delta = 2$.
However, the ATF of the full plan, $e_1e_2e_3$ (Figure~\ref{fig:ae3}) is still constrained by $e_1$'s safe interval to depart between $0$ and $2$, however $e_3$ can not be traversed until $t=8$, requiring at least $6$ units of waiting in addition to the $2$ units of moving to traverse $e_1e_2$, so $\Delta = 2$ and $\alpha = 6$ but $\beta = 2$.
So, if we evaluate $A[e_1e_2e_3](1)$ (Equation~\ref{eqn:edgearrival} at $t = 1$), we are in the second case ($ 0 \leq 1 < \min(6, 2)$ so our earliest arrival time is $\alpha + \Delta = 6 + 2$.

\subsection{Augmented SIPP: Heuristic Search With ATFs}
\begin{algorithm}[t]
\caption{Augmented SIPP}
\label{alg:asipp}
\begin{algorithmic}[1]
\Function{AugmentedSIPP}{$\sigma, \delta, x_o,t_{0}, \isGoal$}
\State $s_o \gets \langle x_o, \langle t_0, \infty \rangle \rangle$
\State $\open \gets \{\langle s_o, I[0, t_0], h[s_o], \varnothing \rangle\}$\label{algln:asippopeninit}
\State generated $\gets \{s_o \mapsto \mathit{top}(\open)\}$  \label{algln:asippclosedinit}
\While{$\lnot \mathit{empty}(\open)$}
    \State $n \leftarrow \mathit{pop}(\open)$ \Comment{$\min f(n)$}\label{algln:nlasippopen}
    \If{$\isGoal(n)$} 
        \State \Return $n, g[n], \open \cup \{n\}, generated$  \label{algln:retval}
    \EndIf
    \For{$e \in \sigma(n)$}\label{algln:asippsuccessorgeneration}
        \State $s' \leftarrow v(e)$
        \State $\zeta \leftarrow \zeta(n)$  \label{algln:asippzeta}
        \State $\alpha \leftarrow \max(\alpha(e) - \Delta(n), \alpha(n))$ \label{algln:asippalpha}
        \State $\beta \leftarrow \min(\beta(e)  - \Delta(n), \beta(n))$
        \State $\Delta \leftarrow \Delta(n) + \delta(e)$ \label{algln:asippdelta}
        \State $g \leftarrow A[\zeta, \alpha, \beta, \Delta]$
        \State $n' \leftarrow \langle s', g, h[s'], n \rangle$
        \If{$s' \notin \closed \lor g[n'] \trianglelefteq  g[\closed(s')]$} \label{algln:asippprune}
            \If{$g(n) \geq \beta(f[n'])$}\label{algln:asippvalidcheck}
                \State $\mathit{continue}$
            \EndIf
            \If{$g[n'] \trianglelefteq  g[\closed(s')]$}
                \State $\open \gets \open \setminus \closed[s']$
            \EndIf
            \State $\closed[s'] \leftarrow n'$\label{algln:asippclosed}
            \State $\open \leftarrow \open \cup\{n'\}$ \label{algln:newnode}
        \EndIf
    \EndFor
\EndWhile
\State \Return Failure \label{algln:asippend}
\EndFunction
\end{algorithmic}
\end{algorithm}

We now present Augmented SIPP, an algorithm that solves SIPP while also returning a plan ATF along with the resulting plan.
Augmented SIPP is a type of cost-algebraic heuristic search where edge costs are edge ATFs, plan costs (e.g. $g$, $h$, and $f$) are plan ATFs, and cost accumulates through ATF composition.
In SIPP the $f$ value of a search node is $g$ the earliest arrival time (along the plan represented by the node) + $h$ the heuristic time to reach the goal.
Augmented SIPP, using the same tie-breaking and $h$ as SIPP, will select the same node as SIPP for every expansion, in effect conducting an identical search that keeps ATF $g$-values rather than scalar $g$-values. 
This is because they both select the next node to expand based on an identical criteria: earliest arrival time starting at $t_0$.
(In the following subsection, we will, after some background on cost-algebraic heuristic search, formalize this similarity by describing Augmented SIPP as a cost-algebraic A*.)

An Augmented SIPP search node is a tuple $\langle s, g, h, p\rangle$ where:
\begin{description}
    \item[$s$] is a state,$\langle x, i\rangle$ (identical to SIPP),
    \item[$g$] is the plan ATF corresponding to the plan to $s$,
    \item[$h$] is an edge ATF time-to-go from $s$, and
    \item[$p$] is a parent pointer (identical to SIPP).
\end{description}
The path ATF priority function of a node is the composition of the ATFs of $h$ and $g$:
\begin{equation}
    f[n] = h[s(n)] \circ g[n],
\end{equation}
the scalar priority of a node in the open list is:
\begin{equation}
    f(n) = f[n](t_0),
\end{equation}
and the scalar earliest arrival time at a node is:
\begin{equation}\label{eqn:gn}
    g(n) = g[n](t_0).
\end{equation}
The heuristic ATF $h[n]$ may be calculated from a scalar heuristic $h(n)$ using the identity ATF (Equation~\ref{eqn:atfifdentity}),
\begin{equation}
    h[n] = I[h(n), t_0].
\end{equation}

Pseudocode for Augmented SIPP is presented in Algorithm~\ref{alg:asipp}.
Augmented SIPP differs from SIPP in three ways: first, the successor generation (Lines~\ref{algln:asippzeta}-\ref{algln:asippdelta}) propagates a plan ATF $g$, following Equations~\ref{eqn:pathzeta}-\ref{eqn:pathdelta}; second, we return not just the node representing the plan, but also the ATF of the plan when one is found (and optionally the open and closed list for use by an outer loop any-start-time search); and third, we use $\trianglelefteq$ to highlight that, in Augmented SIPP (and in cost-algebraic heuristic search in general), the total ordering used to order the open list (Line~\ref{algln:nlasippopen}) and the partial ordering used to prune successors (Line~\ref{algln:asippprune}) need not always be the same. 
We explore this in more detail in Section~\ref{sec:issuewithtiebreaking}.

\section{Theoretical Properties of Augmented SIPP}
In this section, we explore the theoretical properties of Augmented SIPP.
We begin by observing that, when using the same ordering, Augmented SIPP is SIPP, and retains all of the same theoretical properties SIPP enjoys as an A* search such as completeness, optimality, and cost-optimality \cite{dechteroptimalastar}.
\begin{observation}\label{obs:asippissipp}
Augmented SIPP and SIPP perform identical searches.
If they use the same total-ordering of nodes and the same $h(n)$ then both searches will expand a sequence of nodes corresponding to the same states, with the same earliest-arrival-times ($g$-values), in the same order.  
\end{observation}
However, Augmented SIPP is capable of orderings that SIPP is not.  
While both can find optimal plans (that arrive as early as possible when departing at time $t_0$), Augmented SIPP can find the optimal plan with the most waiting (by a secondary ordering maximizing $\alpha$), or the optimal plan that remains valid for the longest (by a secondary ordering maximizing $\beta$), these sort of orderings are critical for ensuring that our approach to any-start-time planning, RePEAT, makes efficient progress when constructing the any-start-time plan.

In this section, we will assume that the heuristic function is admissible; admissibility and consistency have their usual meaning.
However, because $h$ can be an ATF, they are universally quantified over time.
\begin{definition}
    A heuristic $h[n]$ is \textbf{admissible}, iff for the optimal heuristic $h^*[n]$:
    \begin{equation}
        \forall t \in [0, \infty) : h[n](t) \leq h^*[n](t).
    \end{equation}
\end{definition}

In the remainder of this section, we start by formalizing Observation~\ref{obs:asippissipp}, then we focus on showing that these same properties hold even when we are searching with orderings that are impossible for a regular SIPP search.
Many of these arise from Augmented SIPP being a form of cost-algebraic A*, so we will restate the definitions of cost-algebraic heuristic search from \textcite{edelkamp2005cost} and \textcite{Holte_Zilles_optimalitycostalgebra}, and then  show that Augmented SIPP is a cost-algebraic A*. 
There will be a moment of dramatic tension when we see that the secondary ordering that we want to use for Augmented SIPP does not obey the SIPP dominance relation (Theorem~\ref{thm:sippdom}), but we will resolve this by showing that this problematic secondary sort can be viewed as a tie-breaking strategy for a well-behaved sort, and therefore the search remains complete, admissible, and optimally efficient.
Our findings demonstrate a nuanced point: for a cost-algebraic heuristic search, it is not always correct to prune successors using the same total ordering as is used to order the open list!

\subsection{Augmented SIPP is SIPP} \label{sec:asippissipp}

We now formalize Observation~\ref{obs:asippissipp}.
First we show that $g[n](t_0) = \alpha(n) + \Delta(n)$ for all search nodes.
\begin{lemma}\label{lmm:aplusb}
    For all nodes $n$ generated by an Augmented SIPP search, $g[n](t_0) = \alpha(n) + \Delta(n)$.
\end{lemma}
\begin{proof}
    Observe that for all nodes $n$ of the Augmented SIPP search:
    \begin{equation*}
        \zeta(n) = t_0
    \end{equation*} (by Equation~\ref{eqn:pathzeta}, see Algorithm~\ref{alg:asipp} line~\ref{algln:asippopeninit} and Equation~\ref{eqn:atfifdentity}),
    and because SIPP successors must have a non-negative wait:
    \begin{equation*}
        \zeta(n) \leq \alpha(n).
    \end{equation*}
    This means that $g[n](t_0)$ is determined by either the second ($\zeta \leq t \leq \min(\alpha, \beta)$) or third ($\alpha, \leq t < \beta$) case of Equation~\ref{eqn:edgearrival}, so $g[n](t_0) = \max(\alpha(n) + \Delta(n), t_0 + \Delta(n))$. 
    Since we know that $t_0 \leq \alpha(n_{\mathit{aug}})$, we get:
    \begin{equation*}
        g[n](t_0) = \alpha(n) + \Delta(n).
    \end{equation*}
\end{proof}
Now we show that the Augmented SIPP $g$-function evaluated at $t_0$ (from Lemma~\ref{lmm:aplusb}) is equivalent to the SIPP $g$-value.
To do this, we will examine two searches run on the same SIPP problem.
Augmented SIPP, with the $i$th node generated denoted $n_{\mathit{aug}}^i$, and SIPP with the $i$th node generated denoted $n_{\mathit{sipp}}^i$.
\begin{lemma}\label{lmm:equivg}
    Assuming that Augmented SIPP and SIPP generate sequences of nodes corresponding to the same sequence of states and actions, 
    the Augmented SIPP $g$-function evaluated at $t_0$ is the SIPP $g$-value (i.e., $g[n_{\mathit{aug}}](t_0) = g(n_{\mathit{sipp}})$).
\end{lemma}
\begin{proof}
    We prove by induction from the search root.
    At the root, $g(n_{\mathit{sipp}}^0) = t_0$ (Algorithm~\ref{alg:sipp}, line~\ref{algln:sippopen}) and $g[n_{\mathit{aug}}^0] = I[0,t_0]$ (Algorithm~\ref{alg:asipp}, line~\ref{algln:asippopeninit}) so $g[n_{\mathit{aug}}^0](t_0) = t_0$ (Equation~\ref{eqn:atfifdentity}).  
    Proving the base case.

    For the inductive step, consider the nodes: $n_{\mathit{aug}}^i, n^{j}_{\mathit{aug}}, n^{i}_{\mathit{sipp}},$ and $n^{j}_{\mathit{sipp}}$.
    Where $n^j_{\mathit{aug}}$ is the successor of $n^i_{\mathit{aug}}$, $n^j_{\mathit{sipp}}$ is the successor of $n^i_{\mathit{sipp}}$, $n^i_{\mathit{aug}}$ and $n^i_{\mathit{sipp}}$ both correspond to the same plan (sequence of SIPP states), and $n^j_{\mathit{aug}}$, $n^j_{\mathit{sipp}}$ also both correspond to the same plan with the addition of their shared successor edge $e$, and $g[n^i_{\mathit{aug}}](t_0) = g(n^i_{\mathit{sipp}})$.  
    For Augmented SIPP,
    \begin{equation*}
        g[n^j_{\mathit{aug}}](t_0) = (A[e] \circ g[n^i_{\mathit{aug}}])(t_0),
    \end{equation*} 
    which has parameters ($\langle \zeta, \alpha, \beta, \Delta \rangle$):
    \begin{align*}
        \langle &\zeta(g[n^i_{\mathit{aug}}]),\\
        &\max(\alpha(g[n^i_{\mathit{aug}}]), \alpha(e) - \Delta(g[n^i_{\mathit{aug}}])),\\
        &\min(\beta(g[n^i_{\mathit{aug}}]), \beta(e) - \Delta(g[n^i_{\mathit{aug}}])),\\
        &\Delta(g[n^i_{\mathit{aug}}]) + \Delta(e)\rangle
    \end{align*} (see Equations~\ref{eqn:pathzeta}-\ref{eqn:pathdelta}).
    So, by Lemma~\ref{lmm:aplusb}, 
    \begin{equation*}
        g[n^j_{\mathit{aug}}](t_0) = \alpha(n^j_{\mathit{aug}}) + \Delta(n^j_{\mathit{aug}}).
    \end{equation*}

    We first look at the value of $\alpha(n^j_{\mathit{aug}})$.
    By Equation~\ref{eqn:pathalpha}, $\alpha(n^j_{\mathit{aug}}) = \max(\alpha(n^i_{\mathit{aug}}) + \Delta(n^i_{\mathit{aug}}), \alpha(e) ) - \Delta(n^i_{\mathit{aug}})$.
    We know $\alpha(n^i_{\mathit{aug}}) + \Delta(n^i_{\mathit{aug}}) = g[n^i_{\mathit{aug}}](t_0) = g(n^i_{\mathit{sipp}})$, so by Equation~\ref{eqn:alphacompute}:
    \begin{align*}
        \alpha(n^j_{\mathit{aug}}) = \max(&g(n^i_{\mathit{sipp}}),\\
                &s_b(n^i_{\mathit{sipp}}),\\
                &e_b,\\
                &s_b(n^j_{\mathit{sipp}}) - \delta(e) - \Delta(n^i_{\mathit{aug}}).
    \end{align*}
    Because $g(n)$ is on or after the beginning of the safe interval of the corresponding state ($s_b(n)$), the first term in the max dominates the second for all SIPP search nodes.

    Now we look at the value of $\Delta(n^j_{\mathit{aug}})$
    By Equation~\ref{eqn:pathdelta}, $\Delta(n^j_{\mathit{aug}}) = \Delta(n^i_{\mathit{aug}}) + \Delta(e)$, and $\Delta(e) = \delta(e)$ (Equation~\ref{eqn:deltacompute}), so:
    \begin{equation*}
         g[n^j_{\mathit{aug}}](t_0) = \max(g(n^i_{\mathit{sipp}}), e_b, s_b(n^j_{\mathit{sipp}}) - \delta(e)) - \Delta(n^i_{\mathit{aug}}) + \Delta(n^i_{\mathit{aug}}) + \delta(e),
    \end{equation*}
    which simplifies to:
    \begin{align*}
         g[n^j_{\mathit{aug}}](t_0) = \max(&g(n^i_{\mathit{sipp}}) + \delta(e),\\
         &e_b + \delta(e),\\
         &s_b(n^j_{\mathit{sipp}}),
    \end{align*}
    which is exactly the SIPP $g$-value (Algorithm~\ref{alg:sipp}, Line~\ref{algln:sippg}).
\end{proof}
Now we prove the equivalence of the searches done by SIPP and Augmented SIPP.
\begin{theorem}\label{thm:searchasippissipp}
    Given the same SIPP problem, with the same scalar $h(n)$, and the same node ordering and pruning (e.g., $\leq_{f(n)}$) with the same tie breaking, SIPP and Augmented SIPP expand the same states, in the same order, with equivalent $g$-values.  
\end{theorem}
\begin{proof}
    The candidate successors $\sigma(n)$ are identical for SIPP (Algorithm~\ref{alg:sipp}, Line~\ref{algln:sippsuccessorgeneration}) and Augmented SIPP (Algorithm~\ref{alg:asipp}, Line~\ref{algln:asippsuccessorgeneration}).  
    They are pruned using equivalent rules, Algorithm~\ref{alg:sipp}, Line~\ref{algln:sippprune} vs. Algorithm~\ref{alg:asipp}, Line~\ref{algln:asippprune}-\ref{algln:asippvalidcheck}.
    The $g$-values (Lemma~\ref{lmm:equivg}), $h$-values, and open-list ordering are identical, so SIPP and Augmented SIPP select from the pruned collection of successors, ordering on the same criteria, and therefore making the equivalent selection at each step of the search.
\end{proof}

In fact, when substituting Augmented SIPP in for SIPP, it is possible to run a regular SIPP search, and then compute the ATF of only the actual returned plan while that plan is being backed up from the goal node.
However, we have described a full search with ATFs because in any-start-time planning we will use a node ordering that is not possible to use with SIPP (because it requires the ATFs during search) in addition to making use of the ATF of the eventually found plan.

However, it is more informative to derive these properties of Augmented SIPP by examining it as an instance of a cost-algebraic A* \cite{edelkamp2005cost}.
This description demonstrates that the correctness and optimality of Augmented SIPP arise from using ATFs as edge costs and the composition operation that accumulates those costs along a plan.
Additionally, the claim of Observation~\ref{obs:asippissipp} does not apply when we use Augmented SIPP with a tie-breaking strategy that is not possible to do with SIPP alone, which is exactly what we will do when using it for any-start-time planning.
Thus, it is useful to understand that the theoretical guarantees of Augmented SIPP do not have to rely on its relationship to conventional SIPP search.

\subsection{Background: Cost-algebraic Heuristic Search}

Cost-algebraic A* is a generalization of A* to a more general notion of cost than the scalar sum of costs used by A* to solve shortest-path problems.
We begin by reviewing some definitions.\footnote{Definitions~\ref{def:monoid}, 
\ref{def:totalorder}, 
 \ref{def:sum}, \ref{def:isotone}, and \ref{def:costalgebra} are from \textcite{Holte_Zilles_optimalitycostalgebra} and~\ref{def:least} is from \textcite{edelkamp2005cost}.}

A cost algebra is a 4-tuple $\langle \Omega, \bullet, \preceq, \mathbf{1} \rangle$ where
$\Omega$ is the set of possible edge and path costs, 
$\bullet$ is the operator for computing a path’s cost from the costs of its edges, 
$\preceq$ is the ordering on costs, and
$\mathbf{1}$ is the cost of the empty path.
The use of $\bullet$ rather than the usual $+$, and $\mathbf{1}$ rather than the usual $0$ highlights that the assumptions of the usual setting do not necessarily hold.

\begin{definition}[Monoid]\label{def:monoid}
Let $\Omega$ be a set and $\bullet$ : $\Omega \bullet \Omega \rightarrow \Omega$ a binary operator. 
A \textbf{monoid} is a triple $\langle \Omega, \bullet, \mathbf{1} \rangle$ such that $\mathbf{1} \in \Omega$ and
for all a, b, c $\in \Omega$:
    \begin{description}
    \item[associativity] $a \bullet (b \bullet c) = (a \bullet b) \bullet c$
    \item[identity] $\mathbf{1} \bullet a = a \bullet \mathbf{1} = a$
    \end{description}
\end{definition}
\noindent Monoids are similar to groups, except without the requirement of inverses.
\begin{definition}[Total order]\label{def:totalorder}
A total order $\preceq$ is a binary relation on set $\Omega$ where, $\forall a,b,c \in \Omega$:
\begin{enumerate}
    \item $a \preceq a$,
    \item $a \preceq b \land b \preceq c \implies a \preceq c$,
    \item $a \preceq b \land b\preceq a \implies a = b$, and
    \item $a \preceq b \lor b \preceq a$.
\end{enumerate}
\end{definition}
\noindent If $\Omega$ is a set and $\preceq$ is a total order on $\Omega$ then a $\prec$ b denotes $(a \preceq b) \land (a \neq b)$, and $a \succeq b$ and $a \succ b$ are alternative ways of writing $b\preceq a$ and $b \prec a$ respectively.
\begin{definition}[Cost product]\label{def:sum} 
If $\langle \Omega, \bullet, \mathbf{1} \rangle$ is a monoid and $\langle a_1, \dotsb , a_n \rangle \in \Omega^n$ is a sequence of length $n$, then
\begin{equation*}
    \prod_{i = 1}^{n}a_i = \begin{cases}
    \mathbf{1}, &n = 0\\
    a_1, &n = 1\\
    a_1 \bullet \dotsb \bullet a_n, &n > 1. 
    \end{cases}
\end{equation*}
\end{definition}
\begin{definition}[Least]\label{def:least}
    The least operator $\bigsqcup$ on a set of costs $A$ with ordering $\preceq$ is defined: $\bigsqcup A = c$ with $A \subseteq \Omega$ and $c \in A$ such that $\forall a \in A :c \preceq a$.
\end{definition}
\begin{definition}[Isotonicity]\label{def:isotone}
    A monoid $\langle \Omega, \bullet, \mathbf{1} \rangle$ with total order $\preceq$ on $\Omega$ is \textbf{isotone} iff a $\preceq$ b implies both $(a \bullet c) \preceq (b \bullet c)$ and $(c \bullet a) \preceq (c \bullet b)$ for all a, b, c $\in \Omega$.
\end{definition}
\begin{definition}[Cost algebra]\label{def:costalgebra}
    A \textbf{cost algebra} is a 4-tuple $\langle \Omega, \bullet, \preceq, \mathbf{1} \rangle$ such that:
    \begin{itemize}
        \item $\langle \Omega, \bullet, \mathbf{1} \rangle$ is a monoid, 
        \item $\preceq$ is a total order on $\Omega$,
        \item $\mathbf{1} = \bigsqcup \Omega$, and
        \item $\langle \Omega, \bullet, \mathbf{1} \rangle$ with $\preceq$ is isotone.
    \end{itemize}
\end{definition}

For example, the typical shortest path optimization setting, 
for which A* has been shown to be optimally efficient \cite{dechteroptimalastar,Holte_Zilles_optimalitycostalgebra}, 
is the cost-algebra $\langle \mathbb{R}^{\geq 0}, +, \leq, 0 \rangle$ defined by the monoid  $\langle \mathbb{R}^{\geq 0}, +, 0 \rangle$ of non-negative real-valued costs accumulated through summation and the total order $\leq$ over these costs.

\subsection{Pruning in Cost-algebraic Heuristic Search}
Pruning is an important part of a practical heuristic search; however, it is not mentioned in the foundational work on cost-algebraic heuristic search (Domain-independent Dynamic Programming \cite{KUROIWADIDP} considers Cost-Algebraic dynamic programming models and separates pruning from search ordering).
To prune a cost-algebraic search, we define a partial order $\trianglelefteq$ which determines if the $g$ of the newly generated node dominates the $g$ of the existing node (e.g., Algorithm~\ref{alg:asipp}, Line~\ref{algln:asippprune}).
Usually, we are used to seeing the same total order ($\preceq$) used for pruning.
For example, in the typical A* setting, $\leq$ is used to order open and prune successors.

However, this is not the case in similar settings (e.g., multi-objective search which is not a cost algebra because it requires a partial order, see \citeauthor{wolffmulti} for a discussion of the difficulty of pruning in that setting).
So we now show that in the cost-algebraic setting we can use the same ordering to prune:
\begin{theorem}\label{thm:costalgpruning}
    Pruning a cost-algebraic search using its total ordering ($\preceq$) does not prune all optimal plans.
\end{theorem}
\begin{proof}
    This follows in a straightforward manner from isotonicity.
    Consider two nodes, $a$ and $b$, both representing plans up to some shared state $s$.  
    Because $a \preceq b$, we prune $b$.  Suppose we did not prune $b$. The costs of plans for every successor $c$ of $b$ and $a$ are $a \times c$ or $b \times c$, respectively, for the plans via $a$ or $b$.  
    By isotonicity, $a \times c \preceq b \times c$, and the same follows inductively for all pairs of plans, which share the same subplan after $a$ and $b$.
    Therefore, it is safe to have pruned $b$ because if there were an optimal plan through node $b$, there must also be one through $a$ that was not pruned. 
\end{proof}

\subsection{Augmented SIPP is a Cost-algebraic A*}\label{sec:ordering}
We now show that Augmented SIPP, with the usual SIPP search orderings, is a Cost-algebraic A*.
In contrast to the typical scalar-based optimization settings for A*, the Augmented SIPP monoid consists of the ATF costs that are accumulated through function composition, rather than scalar costs accumulated through summation.

Formally, we define the set of costs: 
\begin{definition}
    $\Omega_{t_0}$ is the set of all valid plan ATFs (recall that an ATF is valid if it satisfies Equations~\ref{eqn:nonnegwait}-\ref{eqn:nonnegcost}).
\end{definition}

\newcommand{\leqf}{\leq_{\downarrow f}}
\newcommand{\leqfg}{\leq_{\downarrow f \uparrow g}}
\newcommand{\leqfa}{\leq_{\downarrow f \uparrow \alpha}}
\newcommand{\leqfab}{\leq_{\downarrow f \uparrow \alpha \uparrow \beta}}
We consider the following total orderings over $\Omega_{t_0}$  with the first two being commonly used by SIPP (and heuristic searches in general \cite{asai2017tie}) and the second two being used by our approach to any-start-time planning.  
The primary sort of all four prefers low $f(n)$ nodes, and they differ in their tie-breaking strategies:
\begin{description}
    \item[$\leqf$] breaks ties arbitrarily,
    \item[$\leqfg$] breaks ties preferring high $g(n)$,
    \item[$\leqfa$] breaks ties preferring high $\alpha(n)$, and 
    \item[$\leqfab$] breaks further ties preferring high $\beta(n)$.
\end{description}
While not required for admissibility, breaking ties on $f$ in favor of goals is required for the optimal efficiency of A* \cite{dechteroptimalastar} so we always do that.\footnote{Often this tie-breaking is implicit in $h$, where only goals will have $h = 0$, but we make no such assumption.}
High $g$ (or equivalently low $h$) is a common tie-breaking strategy for A* because it selects nodes that are estimated to be closer to the goal, reducing the portion of the final $f$-layer that is searched.
The second pair of orderings, $\leqfa$ and $\leqfab$ are specialized to the ATF nodes used by Augmented SIPP.
$\leqfa$ effectively breaks ties towards nodes that represent plans that are optimal for longer, and $\leqfab$ selects among those the nodes that are valid for longest.  
Intuitively, this ordering is a natural match for any-start-time planning (and in Section~\ref{sec:time} we formalize this intuition by showing that this node ordering leads to desirable runtime performance).
We first show that Augmented SIPP with any of the first two orderings is a cost-algebraic A*, in the next section we show why this is not the case for $\leqfa$ and $\leqfab$.

The monoid for Augmented SIPP is:
$\langle \Omega_{t_0}, \circ, I[0, t_0]\rangle$, consisting of the set of plan ATFs with $\zeta = t_0$, the function composition operator, and the identity ATF (Equation~\ref{eqn:atfifdentity}) at $t_0$ and with no shift.
Four potential cost algebras for Augmented SIPP are:
\begin{description}
    \item[$C_{\leqf}$] $\coloneqq \langle \Omega_{t_0}, \circ, \leqf, I[0, t_0] \rangle$, 
    \item[$C_{\leqfg}$] $\coloneqq \langle \Omega_{t_0}, \circ, \leqfg, I[0, t_0] \rangle$, 
    \item[$C_{\leqfa}$] $\coloneqq \langle \Omega_{t_0}, \circ, \leqfa, I[0, t_0] \rangle$, and
    \item[$C_{\leqfab}$] $\coloneqq \langle \Omega_{t_0}, \circ, \leqfab, I[0, t_0] \rangle$. 
\end{description}

We now prove that $C_{\leqf}$ and $C_{\leqfg}$ are cost-algebras as defined in Definition~\ref{def:costalgebra}.
Both share the same Augmented SIPP monoid:
\begin{lemma}\label{lmm:asippmonoid}
    $\langle \Omega_{t_0}, \circ, I[0](t) \rangle$ is a monoid.
\end{lemma}
\begin{proof}
    We show that ATF composition is associative, and that $I[0, t_0]$ is the identity.
    For $a, b, c \in \Omega_{t_0}$, plan ATF composition (Equation~\ref{eqn:path2}) consists of four operations that are all associative, e.g., $\zeta(a \circ (b \circ c)) = \zeta(c) = \zeta((a \circ b) \circ c)$. 
    Likewise, the $max$ of the $\alpha$ update (Equation~\ref{eqn:pathalpha}), $min$ of the $\beta$ update (Equation~\ref{eqn:pathbeta}) and summation of the $\Delta$ update (Equation~\ref{eqn:pathdelta} are all associative.
    Therefore, because each of the four parts of the composition operation is associative, the composition operation as a whole is associative.
    The identity ATF, $I[0](t)$, is indeed the identity:
    \begin{align*}
        I[0, t_0] \circ A[\zeta, \alpha, \beta, \Delta] &= \langle t_0, \max(\zeta, \alpha - 0), \min(\infty, \beta - 0), \Delta + 0 \rangle\\ &= A[\zeta, \alpha, \beta, \Delta]
    \end{align*}
\end{proof}

All four are total orders:
\begin{lemma}\label{lmm:asipptotalorder}
     $\leqf$, $\leqfg$, $\leqfa$, and $\leqfab$ are total orders on $\Omega_{t_0}$.
\end{lemma}
\begin{proof}
    Let $\leq_*$ be any of $\leqf$, $\leqfg$, or $\leqfa$. 
    All of $f(n)$, $g(n)$, $\alpha(n)$, and $\beta(n)$ are scalar earliest arrival times, and $\leq$ of scalar values is a total order.
    Either both $a$ and $b$ are not goals and the ordering is $\leq$ of scalar values, or at least one is a goal and the ordering selects one of the goals, fulfilling the requirements for a total ordering.
\end{proof}

The identity ATF is the identity of $\Omega_{t_0}$ for all four orderings:
\begin{lemma}\label{lmm:asippidentitymin}
    $I[0, t_0]$ $\leqf$, $\leqfg$, $\leqfa$, and $\leqfab$ $a$ for all $a \in \Omega_{t_0}.$
\end{lemma}
\begin{proof}
    No ATF in $\Omega_{t_0}$ may arrive before it departs.  Specifically, for all $t \geq t_0$, $I[0, t_0](t) = t$ and $A[a](t_0) \geq t_0 + \Delta(a)$ (where $\Delta(a) \geq 0$ so $A[a](t_0) \geq I[0, t_0](t)$.
\end{proof}

Now we show that $\leqf$ and $\leqfg$are isotone: 
\begin{lemma}\label{lmm:asippisotoneleqf}
    $\langle \Omega_{t_0}, \circ, I[0, t_0]\rangle$ with $\leqf$ is isotone.
\end{lemma}
\begin{proof}
    Let $a$, $b$, and $c\in \Omega_{t_0}$ and $a \leqf b$.
    The expression $c \circ a  \leqf c \circ b$ holds because our ATFs are first-in first-out (you can not arrive before you depart), while $a \circ c \leqf b \circ c$ is a vacuous truth as there are never such successors generated as we are performing a forward search.
\end{proof}
\begin{lemma}\label{lmm:asippisotoneleqfg}
    $\langle \Omega_{t_0}, \circ, I[0, t_0]\rangle$ with $\leqfg$ is isotone.
\end{lemma}
\begin{proof}
    Let $a$, $b$, and $c\in \Omega_{t_0}$ and $a \leqfg b$.
    The expression $a \circ c \leqf b \circ c$ is a vacuous truth as there are no such successors generated as we are performing a forward search.
    If $f(a) \neq f(b)$ then Lemma~\ref{lmm:asippisotoneleqf} applies, otherwise $f(a) = f(b)$ and $g(a) \geq g(b)$.  Assume in contradiction $\lnot g(c \circ b) \leqfg g(c \circ a)$.  Plugging this into Equation~\ref{eqn:gn} and Equations~\ref{eqn:pathalpha},\ref{eqn:pathdelta} we get $g(c\circ a) = \max(\alpha(a), \alpha(c) - \Delta(a)) + \Delta(c) + \Delta(a)$, which simplifies to $g(c\circ a) = \max(g(a), \alpha(c)) + \Delta(c)$.  Likewise $g(c\circ b) = \max(g(b), \alpha(c)) + \Delta(c)$, so substituting into our contradictory assumption we get $\max(g(a), \alpha(c)) + \Delta(c) < \max(g(b), \alpha(c)) + \Delta(c)$, a contradiction.
\end{proof}

Therefore,
\begin{theorem}\label{thm:asippiscostalg}
    Augmented SIPP with orderings$\leqf$ and $\leqfg$ are cost-algebraic A*s.
\end{theorem}
\begin{proof}
    By Definition~{\ref{def:costalgebra} and Lemmas~\ref{lmm:asippmonoid}-\ref{lmm:asippisotoneleqfg}}, the search space of Augmented SIPP is a cost-algebra.  
    Augmented SIPP (Algorithm~\ref{alg:asipp}) conforms to the generalized definition of A* given by \textcite{edelkamp2005cost}, with the addition of pruning, which is safe (Theorem~\ref{thm:costalgpruning}). 
      Therefore, Augmented SIPP is a Cost-algebraic A*.
\end{proof}
This means that Augmented SIPP inherits the properties of a cost-algebraic A*, including that:
\begin{corollary}
    Augmented SIPP is admissible,
\end{corollary}
\begin{proof}
    Cost-algebraic A* is admissible \cite{edelkamp2005cost}, and 
    Augmented SIPP is a cost-algebraic A* (Theorem~\ref{thm:asippiscostalg}). 
\end{proof}
and that:
\begin{corollary}
    Augmented SIPP with isotone ordering is optimally efficient.
\end{corollary}
\begin{proof}
    Cost-algebraic A* with isotone ordering is optimally efficient \cite{Holte_Zilles_optimalitycostalgebra}, and 
    Augmented SIPP is a cost-algebraic A* (Theorem~\ref{thm:asippiscostalg}). 
\end{proof}

The optimality and correctness of Augmented SIPP do not rely on it being identical to the SIPP algorithm.  
Instead, it is a property of searching with ATFs.

\subsection{Pruning vs Node Ordering in Cost-algebraic Search With Non-isotone Tie-breaking}\label{sec:issuewithtiebreaking}

So far, we have shown that Augmented SIPP with isotone node ordering ($\leqf$, $\leqfg$) behaves as expected.
We now examine the properties of Augmented SIPP when using other search orderings $\leqfa$ or $\leqfab$ that we will use below in our approach to any-start-time planning.
First, we will demonstrate that this ordering is not isotone and that it violates the SIPP dominance relationship (Theorem~\ref{thm:sippdom}).
Then, we show how to modify the pruning done by Augmented SIPP to avoid these issues.
In short, we can only prune nodes when their ATF is dominated, a weaker partial order than the total order used to order $\open$. 

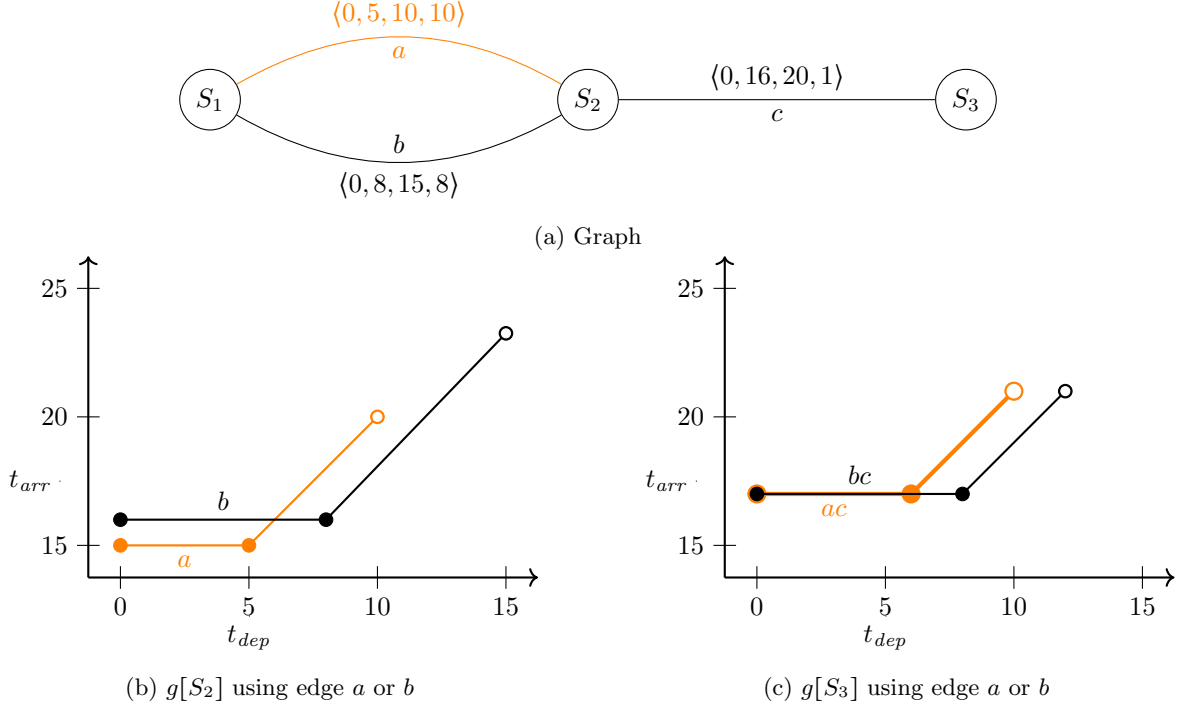
\begin{figure}
    \centering
    \begin{subfigure}{0.9\linewidth}
        \centering
        \begin{tikzpicture}[>=Stealth, node distance=5cm]
        \node[circle, draw] (S1) {$S_1$};
        \node[circle, draw, right of=S1] (S2) {$S_2$};
        \node[circle, draw,  right of=S2] (S3) {$S_3$};
        
        \draw[-, bend left=30, orange] (S1) to node[above] {$\langle 0, 5, 10, 10 \rangle$} node[below] {$a$} (S2);
        \draw[-, bend right=30] (S1) to node[below] {$\langle 0, 8, 15, 8 \rangle$} node [above] {$b$} (S2);
        
        \draw[-] (S2) to node[above] {$\langle 0, 16, 20, 1 \rangle$} node[below] {$c$} (S3);
        \end{tikzpicture}
        \caption{Graph}
        \label{fig:dominanceexamplegraph}
    \end{subfigure}
    \begin{subfigure}[b]{0.49\textwidth}
        \centering
        \begin{tikzpicture}[scale=.85]
        \draw[thick,->] (0,0) -- (7,0);
        \draw[thick,->] (0,0) -- (0,5);
        \draw (.5,5pt) -- (.5,-5pt) node[anchor=north] {$0$};
        \draw (2.5,5pt) -- (2.5,-5pt) node[anchor=north] {$5$};
        \draw (4.5,5pt) -- (4.5,-5pt) node[anchor=north] {$10$};
        \draw (6.5,5pt) -- (6.5,-5pt) node[anchor=north] {$15$};
        \draw (5pt, 0.5) -- (-5pt, 0.5) node[anchor=east] {$15$};
        \draw (5pt, 2.5) -- (-5pt, 2.5) node[anchor=east] {$20$};
        \draw (5pt, 4.5) -- (-5pt, 4.5) node[anchor=east] {$25$};
        \draw (-.45, 1.5) -- (-.45, 1.5) node[anchor= east] {$t_{\textit{arr}}$};
        \draw (2.5, -.6) -- (2.5, -.6) node[anchor=north] {$t_{\textit{dep}}$};

        \draw [thick, orange] (.5,.5) -- (2.5,.5)  node[midway, below] {$a$} -- (4.5,2.5);
        \filldraw[orange] (.5,.5) circle (3pt);
        \filldraw[orange] (2.5,.5) circle (3pt);
        \filldraw[orange] (4.5,2.5) circle (3pt);
        \filldraw[white] (4.5,2.5) circle (2pt);

        \draw [thick] (.5,.9) -- (3.7,.9)  node[midway, above] {$b$} -- (6.5,3.8);
        \filldraw[black] (.5,.9) circle (3pt);
        \filldraw[black] (3.7,.9) circle (3pt);
        \filldraw[black] (6.5,3.8) circle (3pt);
        \filldraw[white] (6.5,3.8) circle (2pt);
        \end{tikzpicture}
        \caption{$g[S_2]$ using edge $a$ or $b$}
        \label{fig:dominanceexample1exp}
    \end{subfigure}
    \hfill
    \begin{subfigure}[b]{0.49\textwidth}
        \centering
        \begin{tikzpicture}[scale=.85]
        \draw[thick,->] (0,0) -- (7,0);
        \draw[thick,->] (0,0) -- (0,5);
        \draw (.5,5pt) -- (.5,-5pt) node[anchor=north] {$0$};
        \draw (2.5,5pt) -- (2.5,-5pt) node[anchor=north] {$5$};
        \draw (4.5,5pt) -- (4.5,-5pt) node[anchor=north] {$10$};
        \draw (6.5,5pt) -- (6.5,-5pt) node[anchor=north] {$15$};
        \draw (5pt, 0.5) -- (-5pt, 0.5) node[anchor=east] {$15$};
        \draw (5pt, 2.5) -- (-5pt, 2.5) node[anchor=east] {$20$};
        \draw (5pt, 4.5) -- (-5pt, 4.5) node[anchor=east] {$25$};
        \draw (-.45, 1.5) -- (-.45, 1.5) node[anchor= east] {$t_{\textit{arr}}$};
        \draw (2.5, -.6) -- (2.5, -.6) node[anchor=north] {$t_{\textit{dep}}$};
        
        \draw [orange, ultra thick] (.5,1.3) -- (2.9,1.3)  node[midway, below] {$ac$} -- (4.5,2.9);
        \filldraw[orange] (.5,1.3) circle (4pt);
        \filldraw[orange] (2.9,1.3) circle (4pt);
        \filldraw[orange] (4.5,2.9) circle (4pt);
        \filldraw[white] (4.5,2.9) circle (3pt);

        \draw [thick] (.5,1.3) -- (3.7,1.3)  node[midway, above] {$bc$} -- (5.3,2.9);
        \filldraw[black] (.5,1.3) circle (3pt);
        \filldraw[black] (3.7,1.3) circle (3pt);
        \filldraw[black] (5.3,2.9) circle (3pt);
        \filldraw[white] (5.3,2.9) circle (2pt);
        \end{tikzpicture}
        \caption{$g[S_3]$ using edge $a$ or $b$}
        \label{fig:dominanceexample2exp}
    \end{subfigure}
\caption{Counterexample to SIPP dominance for Augmented SIPP with $\leqfab$ (or $\leqfa$).}
\label{fig:dominanceexample}
\end{figure}

Figure~\ref{fig:dominanceexample} shows an example Augmented SIPP problem that demonstrates that the Augmented SIPP monoid with $\leqfab$ is not isotone through a violation of the SIPP dominance relationship.
The graph shown in Figure~\ref{fig:dominanceexamplegraph} consists of three states: $S_1$, $S_2$, and $S_3$.
For simplicity we have $h[n] = I[0, 0]$ which is equivalent to the scalar heuristic $h(n) = 0$.
The starting state is $S_1$ and the goal state is $S_3$.
There are two edges connecting $S_1$ to $S_2$; the ATFs of these two paths are shown in Figure~\ref{fig:dominanceexample1exp}.  
Edge $a$ arrives at $S_2$ earlier (i.e., visually in Figure~\ref{fig:dominanceexample1exp} it has a lower $f(n) = h[n] \circ g[n] = g[n]$ shown as $t_{\mathit{arr}}$ in Figure~\ref{fig:dominanceexample}
), however the plan is invalid earlier (i.e.,it has an earlier $\beta$).
On the other hand, the other edge $b$ from  $S_1$ arrives at $S_2$ later but with a plan that is valid for longer.
Because it has a higher $g(n)$, Augmented SIPP with ordering and pruning using $\leqfab$ would prune $b$ in favor of $a$ (Algorithm~\ref{alg:asipp},~Line~\ref{algln:asippprune}). 
However, if we were to keep it and extend these two paths by the edge $c$ we see that $c$ introduces a stricter constraint (in the form of a high $\alpha$) than either $a$ or $b$, erasing their difference in terms of earliest arrival time but leaving the plan using $b$ better because it remains valid longer.
This contradicts the SIPP dominance relation (Theorem~\ref{thm:sippdom}): the node ($a$) that arrives earlier at $S_2$ is worse, because though both have the same earliest arrival time at the goal, the plan using $b$ stays optimal for longer, and then stays executable for longer.
In fact, the ATF of $bc$ dominates the ATF of $ac$, so it would actually be correct to prune $ac$ (if the search were to continue to a goal past $S_3$).
More precisely, the reason that the Augmented SIPP monoid is anisotonic with respect to the total ordering $\leqfab$ is because our ordering is trying to maximize $\beta(n)$, a value that is successively minimized by the ATF composition operation (Equation~\ref{eqn:pathbeta}) along a path. 

This result is perhaps counterintuitive, as we would not generally expect a tie-breaking strategy (here $\uparrow\alpha \uparrow\beta$) to require a weakening of the pruning rules.
In summary, 
$\leqfab$ can be viewed as a tie-breaking strategy for $\leqf$; in this view, we can prune using $\leqfab$, but we are only guaranteed to find a least $f$ solution (not a greatest $\beta$, greatest $\alpha$, least $f$ solution).
This illustrates that not all optimizations that are correct when used with A* (e.g., pruning using the same ordering as the search) can also be used with all other variants of cost-algebraic A*.

In order to guarantee that the found plan respects the full $\leqfab$ ordering, we weaken the pruning used by Augmented SIPP, so that we no longer prune non-dominated nodes.
This is making a tradeoff to do less pruning in the initial Augmented SIPP search so that the solution it finds is more informative, which leads to fewer Augmented SIPP searches needing to be run to find the any-start-time plan.  
When we potentially devote extra (beyond using $\leqfg$) expansions, we are eliminating those expansions from later rounds of any-start-time search.

Formally, we define two partial orders of dominance between ATFs:
\begin{definition}\label{def:atfdom}
    An ATF $A$ \textbf{dominates} another ATF $A'$ on interval $i$ if and only if $\forall t \in i:  A(t) \leq A'(t)$.
\end{definition}
Note that case 1 and case 4 of Equation~\ref{eqn:edgearrival} produce the same symbol $\infty$, so if both $A(t)$ and $A'(t)$ evaluate to $\infty$ then $A(t) \leq A'(t)$ holds at $t$.
In general, this comparison can also be avoided by considering only the interval where for all $t$ at least one of $A(t)$ and $A'(t)$ is finite.
\begin{definition}\label{def:atfstrongdom}
    An ATF $A$ \textbf{strongly dominates} another ATF $A'$ on interval $i$ if and only if $\forall t \in i:  A(t) < A'(t)$.
\end{definition}
By this definition, $A$ being finite for all $t \in i$ is a necessary condition for the strong dominance of $A$.  
To avoid confusion with the scalar comparisons ($\leq$) and the cost-algebraic total orders ($\preceq$) we will use the notation $A \trianglelefteq_{ATF} A'$ to mean that $A$ dominates $A'$ on the interval $[0, \infty)$.
Likewise, we will use $A \triangleleft_{ATF} A'$ to mean that $A$ strongly dominates $A'$ on the interval $[0, \infty)$.
These definitions apply for all three types of ATFs we use in this paper: edge, plan, and compound.

We will use \aSIPPab\ to refer to the instantiation of Augmented SIPP where we:
\begin{enumerate}
    \item select nodes from $\open$ based on the total ordering $\leqfab$, and
    \item prune generated nodes representing the same state as a previously-generated node based on the partial ordering $\trianglelefteq_{ATF}$.
\end{enumerate}
Note that
$\leqfab$ respects $\leqf$, being merely a tie-breaking strategy among equal $f$ nodes.

We now explore the optimality of \aSIPPab.
This hinges on two results: first, that pruning using $\trianglelefteq_{ATF}$ retains optimal plans, and second, that when ordering open by $\leqfab$ the first goal node expanded is an optimal plan.

\begin{lemma}\label{lmm:absipppruneokay}
    Pruning \aSIPPab{} using $\trianglelefteq_{ATF}$ does not prune all optimal plans.
\end{lemma}
\begin{proof}
    Observe that $\trianglelefteq_{ATF}$ is isotonic.  
    If, for two nodes $a$ and $b$ to their shared state $s$, $a \trianglelefteq_{ATF} b$, then the plan for $a$ arrives at least as early in $s$ as the plan for $b$ for all possible arrival times, so it will be at least as good as $b$ for any continuing plans, so it is safe to prune $b$. 
\end{proof}
One view is that ordering on $f$ is a tie-breaking strategy for the partial order $\trianglelefteq_{ATF}$, so we can always prune dominated solutions.

Now consider the first goal node expanded:
\begin{theorem}\label{thm:absippoptimal}
    The first goal node ($n_g$) selected for expansion by \aSIPPab\ is optimal, assuming $h$ is admissible.
\end{theorem}
\begin{proof}
    Suppose, in contradiction, that we select for expansion the first goal node $n_g$, but it is not optimal, so there is some other optimal plan with node $n^*$ (and $n^* <_{\downarrow f \uparrow \alpha \uparrow \beta} n_g$) which can not have already been expanded ($n_g$ is the first goal node) or pruned (Lemma~\ref{lmm:absipppruneokay}), so either it is in open, or it is the (hypothetical) descendant of a node $n$ on open with $n \leqf n^*$.
    We know that $n_g \leqfab n$ because it was selected to be expanded first, therefore $n_g \leqf n^*$, so $n_g$ and $n^*$ must have the same earliest arrival time ($f(n)$), because $n_g \leqf n^*$ and $n^* <_{\downarrow f \uparrow \alpha \uparrow \beta} n_g$ could only both be true if they have the same $f$.
    Because there are no zero-duration edges (and $h$ is admissible), $n = n^*$, as otherwise $n$ would have already been expanded. 
    But this is a contradiction, because $n_g \leqfab n$ but $n^* <_{\downarrow f \uparrow \alpha \uparrow \beta} n_g$.
\end{proof}

\subsection{How \aSIPPab{} Aids Any-start-time Planning}

We conclude our exploration of Augmented SIPP's theoretical properties by showing how it enables progress in any-start-time planning.
The explanation begins with Augmented SIPP's property that each optimal plan to a goal is represented by some node on $open$, and that node will have a plan ATF that is a non-overestimate of the cost of the plan (this is similar to Lemma~1 from \textcite{Holte_Zilles_optimalitycostalgebra} and \textcite{hart:fbh}):
\begin{theorem}\label{thm:optimism}
Given a SIPP problem $\langle S, \sigma, \delta, x_o, t_0, \isGoal \rangle$ and any admissible 
heuristic function $h$ over $S$, then always in Line~\ref{algln:nlasippopen} of Augmented SIPP using $h$,
for all $t \in i(x_o)$, and any hypothetical optimal plan $p^*$ represented by node $n^*$, with corresponding goal state $x_g$, there exists a node $n_f$ on $open$, with frontier state $s(n_f)$, such that:
a)~$s(n_f) \in p^*$ (i.e., there is a node on open corresponding to a state along the plan) and 
b)~$f[n_f] \trianglelefteq_{ATF} f[n^*]$.
\end{theorem}
\begin{proof}
By induction on iterations of the Augmented SIPP search.  
When open contains only the initial node, $s_0$ must be on the plan so (a) holds, and there is only a single node on open, with the identity as its $g$ and the heuristic is admissible so (b) holds.
The inductive step considers successor generation, during which $n_f$ is removed from $open$ and expanded and a child $n_f'$, with $s(n_f) \in p^*$ is generated. 
If $p^*$ exists, then (ignoring pruning) such a node $n_f'$ must exist so (a) holds, and the heuristic is admissible so (b) holds.
If we do prune, we can not prune all such nodes (Lemma~\ref{lmm:absipppruneokay}), so (a) and (b) still hold.
\end{proof}
During any point in the search, Theorem~\ref{thm:optimism} implies that the open list taken as a whole is an optimistic underestimate of the any-start-time plan, (i.e., if we construct a set containing all the ATFs of all the nodes on open, and minimize over the $f[n](t)$ of the ATFs in that set at particular time $t$, the minimizer is a heuristic underestimate of the optimal plan cost (assuming an admissible $h$)).

Augmented SIPP finds a solution that is optimal for at least one (of infinitely many) departure times, which is the starting time of the search $t_0=0$.
Clearly, a solution is optimal at least until its $\alpha$ breakpoint (assuming $\alpha \leq \beta$) because the arrival time remains constant, and any plan that arrives earlier would have had a lower $f$ than our original solution, meaning it would have been returned instead.  
However, if we pause the Augmented SIPP search at the moment a goal state (call its node $\mathit{soln}$) is selected for expansion and examine the optimal $f$-layer (the set of nodes $n \in \textit{open}$ with $f(n) = f(\textit{soln})$), we may find a variety of plans with a variety of ATFs, all with the same $f$ (perhaps optimistic due to an admissible $h$).
However with \aSIPPab\ the returned solution dominates its entire $f$-layer, meaning that \aSIPPab\ returns a solution that is optimal until another solution with a later arrival time:
\begin{theorem}\label{thm:asippdom}
    A solution node $n$ that is returned by \aSIPPab\ using an admissible heuristic has $f[n]$ that dominates all other nodes in its $f$-layer for the interval $[0,\beta(n))$.
\end{theorem}
\begin{proof}
    \aSIPPab\ is optimal (Theorem~\ref{thm:absippoptimal}), so for all other (hypothetical) nodes $n_f$ in its $f$-layer, $f(n) = f(n_f)$ (they are in the same $f$-layer) and $f[n] \trianglelefteq f[n_f]$, so $\forall t \in [0,\infty): f[n](t) \leq f[n_f](t)$.
\end{proof}
Solutions found with this ordering dominate those found with the typical $\leqfg$ ordering.
Crucially, the solution ATF also enables us to reason about how long the solution is optimal, given the rest of the open list when the search completes.
Informally, Theorem~\ref{thm:asippdom} shows our solution dominates others in its f-layer, so our solution is optimal until its $\beta$, or until its ATF intersects with another (higher f) node's ATF, which can only happen on the rising segment, in the interval $(\alpha, \beta)$.
The details of this are explored in subsection \ref{sec:reftimeupdate}.
But, because the Augmented SIPP solution dominates its entire f-layer, it is guaranteed to be optimal for a strictly positive duration, in both departure time and arrival time.
Which is to say, it is guaranteed to make progress in constructing an any-start-time plan.

Now that we have an understanding of planning with ATFs using Augmented SIPP, we can turn our focus back to the central problem of this paper: any-start-time planning.

\section{Solving Any-start-time Planning}

In this section, we formally define the any-start-time planning problem and provide an algorithm to solve it.
The Any-start-time planning problem~(@SIPP) is a straightforward modification of the SIPP problem: given a SIPP problem instance $\langle S, \sigma, \delta, x_o, \isGoal \rangle$ (excluding the start time, $t_0$), the objective is now to find a shortest plan to the goal configuration $x_g$ (where $\isGoal(x_g)$) from the initial state $s_o$ for all safe departure times $t$ from $s_o$, where $t \in i(x_o)$.
We will represent a solution to an @SIPP problem as a compound ATF $A[x_o,x_g]$, which is a set of optimal plans that can be queried at any departure time for the optimal plan and its corresponding arrival time.  (We will see an example shortly.)
Specifically, we will notate the query returning a search node representing a plan to the goal that is optimal at $t$ as $A[s_o,x_g][t]$, and the corresponding arrival time at $t$ as $A[s_o,x_g](t)$, consistent with our usage of parentheses to indicate functions that return scalar values.

In brief: to solve @SIPP, we will repeatedly run \aSIPPab\ with a monotonically increasing starting time, storing the resulting ATFs in a compound ATF (cATF).
The cATF can be efficiently queried for a plan at any start time.
The final piece to solving @SIPP is knowing for what start times to call \aSIPPab.
We show how to compute a monotonically increasing `reference time' that tracks the time up to which, based on the current set of \aSIPPab\ results, the @SIPP plan is provably optimal. 
Together, these form an anytime algorithm that we call RePEAT, in part because it repeatedly computes optimal plans for later and later reference times.

\subsection{Worked Example: Delivery Robot Any-start-time Planning}\label{sec:@SIPPexample}

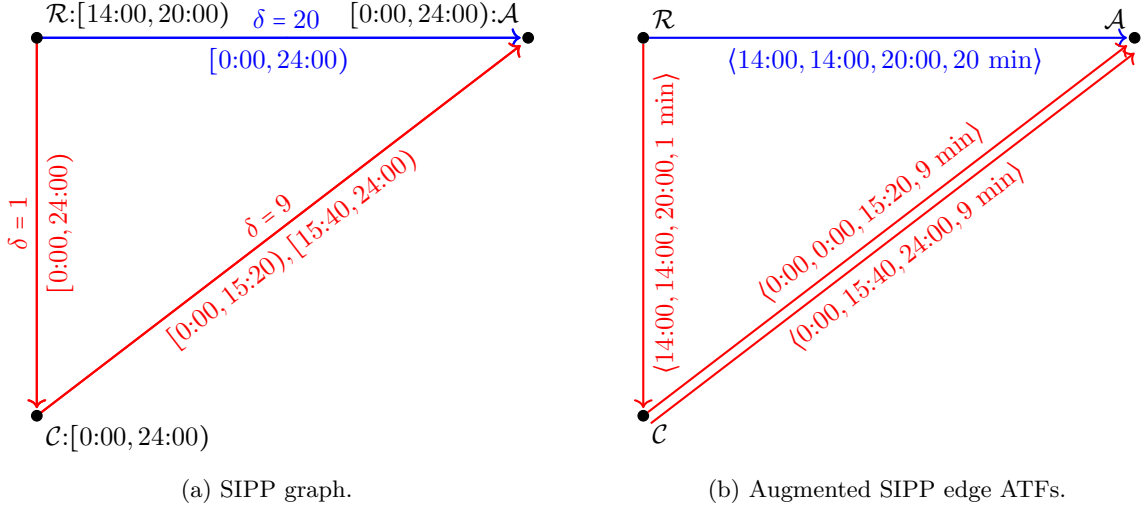
\begin{figure}[t]
  \centering
  \begin{subfigure}{0.49\textwidth}
  \centering
\begin{tikzpicture}[scale = 1]
    \draw[thick,blue,->] (0,5) -- node[anchor=north] {$[0\text{:}00, 24\text{:}00)$} (6.4,5);
    \draw[thick,blue,->] (0,5) -- (6.4,5);
    \draw[thick,blue] (3.3,5) -- node[anchor=south] {$\delta = 20$} (3.3,5);

    \draw[thick,red,->] (0,5) -- node[anchor = north, rotate = 90] {$[0\text{:}00, 24\text{:}00)$} (0,0.1);
        \draw[thick,red,->] (0,5) -- node[anchor = south, rotate = 90] {$\delta = 1$} (0,0.1);

     \draw[thick,red,->] (0,0)  -- node[anchor = north, rotate = 37.56859] {$[0\text{:}00, 15\text{:}20),[15\text{:}40, 24\text{:}00)$} (6.4,4.9);
     \draw[thick,red,->] (0,0)  -- node[anchor = south, rotate = 37.56859] {$\delta = 9$} (6.4,4.9);
     
    \filldraw[black] (0,5) circle (2pt) node [anchor=south west] {$\mathcal{R}\text{:}[14\text{:}00,20\text{:}00)$};
    \filldraw[black] (6.5,5) circle (2pt) node [anchor=south east] {$[0\text{:}00, 24\text{:}00)\text{:}\mathcal{A}$};
    \filldraw[black] (0,0) circle (2pt) node [anchor=north west] {$\mathcal{C}\text{:}[0\text{:}00, 24\text{:}00)$};
\end{tikzpicture}  
\caption{SIPP graph.}\label{fig:delivery2sipp}
\end{subfigure}
   \begin{subfigure}{0.49\textwidth}
  \centering
\begin{tikzpicture}[scale = 1]
    \draw[thick,blue,->] (0,5) -- node[anchor=north] {$\langle 14\text{:}00, 14\text{:}00, 20\text{:}00, 20\text{ min}\rangle$} (6.4,5);

    \draw[thick,red,->] (0,5) -- node[anchor = north, rotate = 90] {$\langle 14\text{:}00, 14\text{:}00, 20\text{:}00, 1\text{ min}\rangle$} (0,0.1);

     \draw[thick,red,->] (0,0)  -- node[anchor = south, rotate = 37.56859] {$\langle 0\text{:}00, 0\text{:}00, 15\text{:}20, 9\text{ min}\rangle$} (6.4,4.9);
     \draw[thick,red,->] (0.1,-.10)  -- node[anchor = north, rotate = 37.56859] {$\langle 0\text{:}00, 15\text{:}40, 24\text{:}00, 9\text{ min}\rangle$} (6.5,4.8);

     \filldraw[white] (0,5) circle (2pt) node [anchor=south west] {$\mathcal{R}\text{:}[2\text{:}00,8\text{:}00)$};
    \filldraw[white] (6.5,5) circle (2pt) node [anchor=south east] {$[0\text{:}00, 24\text{:}00)\text{:}\mathcal{A}$};
    \filldraw[white] (0,0) circle (2pt) node [anchor=north west] {$\mathcal{C}\text{:}[0\text{:}00, 24\text{:}00)$};
    
    \filldraw[black] (0,5) circle (2pt) node [anchor=south west] {$\mathcal{R}$};
    \filldraw[black] (6.5,5) circle (2pt) node [anchor=south east] {$\mathcal{A}$};
    \filldraw[black] (0,0) circle (2pt) node [anchor=north west] {$\mathcal{C}$};
\end{tikzpicture}  
\caption{Augmented SIPP edge ATFs.}\label{fig:delivery2edge}
\end{subfigure}
\caption{State $\mathcal{C}$ has been added to the delivery robot example (Figure~\ref{fig:delivery}) to illustrate calculating the ATF of a multi-edge plan.}\label{fig:delivery2}
\end{figure}

To concretely see how RePEAT works, we give an example of calculating the reference time and tracking a compound ATF solution constructed from a sequence of \aSIPPab\ solutions for the delivery robot scenario used earlier (Figure~\ref{fig:delivery2}).
Like our original delivery robot scenario (Figure~\ref{fig:delivery}), this problem models a delivery robot navigating from a restaurant $\mathcal{R}$ to an apartment $\mathcal{A}$ to deliver a meal.
However, to create the SIPP graph shown in Figure~\ref{fig:delivery2sipp} we have added another state $\mathcal{C}$ that represents the intersection immediately before the rail crossing so that some of our final plans will consist of multiple edges (thereby making this example more illuminating).
We use the 24-hour clock to signify times, and give durations in minutes.
We restrict ourselves to a single day, so the interval of a configuration that is always safe is $[0\text{:}00, 24\text{:}00)$. 

The SIPP graph of our problem has three states: 
\begin{align*}
    &\langle \mathcal{R}, [14\text{:}00, 20\text{:}00) \rangle\\
    &\langle \mathcal{C}, [0\text{:}00, 24\text{:}00) \rangle\\
    &\langle \mathcal{A}, [0\text{:}00, 24\text{:}00) \rangle
\end{align*}
representing the restaurant ($\mathcal{R}$), which has a safe interval representing the dinner service hours of $[14\text{:}00, 20\text{:}00)$,  
the railroad crossing ($\mathcal{C}$), 
and the apartment ($\mathcal{A}$), respectively.
The restaurant is the starting state, the crossing and apartment are always safe, and the apartment is the goal state.  
The action of moving from $\mathcal{C}$ to $\mathcal{A}$ has a pair of safe intervals split by the train. 
Formally, the actions are: 
\begin{align*}
e_1 &= \langle \mathcal{R}, \mathcal{C}, [0\text{:}00, 24\text{:}00) \rangle\\
e_2 &= \langle \mathcal{C}, \mathcal{A}, [0\text{:}00, 15\text{:}20)\rangle\\ 
e_3 &= \langle \mathcal{C}, \mathcal{A}, [15\text{:}40, 24\text{:}00) \rangle\\
e_4 &= \langle \mathcal{R}, \mathcal{A}, [0\text{:}00, 24\text{:}00) \rangle.
\end{align*}
The travel times are: 
\begin{align*}
    \delta(\mathcal{R}, \mathcal{C}) &= 1 \text{ min}\\
    \delta(\mathcal{C}, \mathcal{A}) &= 9 \text{ min}\\
    \delta(\mathcal{R}, \mathcal{A}) &= 20 \text{ min}.
\end{align*}
We now calculate the parameters $\langle \zeta, \alpha, \beta, \Delta \rangle$ of our edge ATFs of the Augmented SIPP graph (Figure~\ref{fig:delivery2edge}), using Equations \ref{eqn:zetacompute}-\ref{eqn:deltacompute}:
\begin{align*}
A[e_1] &= \langle 14\text{:}00, 14\text{:}00, 20\text{:}00, 1\text{ min}\rangle\\   
A[e_2] &= \langle 0\text{:}00, 0\text{:}00, 15\text{:}20, 9\text{ min}\rangle\\
A[e_3] &= \langle 0\text{:}00, 15\text{:}40, 24\text{:}00, 9\text{ min}\rangle\\
A[e_4] &= \langle 14\text{:}00, 14\text{:}00, 20\text{:}00, 20\text{ min}\rangle
\end{align*}

Search nodes consist of a state (for clarity, we list the full plan), its plan ATF which takes the place of the traditional scalar $g$ value, a scalar $f$ value to order open, and the parent edge to act as a parent pointer.  
For this example, we will use $h(n) = $ the cost-to-go assuming all states and actions are always safe.  
Augmented SIPP's open list will be sorted by the earliest expected arrival time~($f(n) =  \alpha(n) + \Delta(n) + h(n)$).
We use $t_0 = 14\text{:}00$, our initial state is $\mathcal{R}$, and the $\isGoal$ predicate is: $\mathit{isGoal(n): \mathit{state}(n)== \mathcal{A}}$.
The initial ATF is the identity ($I[0, 14\text{:}00]$, and $f$ is $14\text{:}00 + 0\text{:}00 + 10 \text{ min} = 0\text{:}10$. The initial open list is:
\begin{equation*}
    \text{open}_0 = \{\langle \mathcal{R}, \langle 14\text{:}00, 14\text{:}00, \infty, 0 \text{ min} \rangle, f = 14\text{:}10, \varnothing \rangle\}
\end{equation*} 
We pop the front of open and expand the node corresponding to the plan $\mathcal{R}$.
The successors of $\mathcal{R}$ are generated from the edges $e_1$ and $e_2$ that originate in $\mathcal{R}$.  
Each successor is generated by applying Equations~\ref{eqn:zetacompute}-\ref{eqn:deltacompute} yielding two successors that are added to open, the open list after the first expansion is: 
\begin{align*}
    \text{open}_1 = \{
\langle \mathcal{RC}&, \langle 14\text{:}00, 14\text{:}00, 20\text{:}00, 1 \min\rangle, f = 14\text{:}10, e_1\rangle,\\
\langle \mathcal{RA}&, \langle 14\text{:}00, 14\text{:}00, 20\text{:}00, 20 \min\rangle, f = 14\text{:}20, e_4\rangle
\}
\end{align*}
Popping the front of open, we expand $\mathcal{RC}$, yielding two more potential successors.
Both successors correspond to plans through states $\mathcal{RCA}$, however the first follows edges $e_1,e_2$ and the second $e_1,e_3$.
Note that the second is not pruned, even though it has a worse $g(n)$ because it is not dominated.
We arrive at $\beta(e_1, e_2) = 15\text{:}19$ from Equation~\ref{eqn:pathbeta}, $\beta(e_1, e_2) = \min(\beta(e_1), \beta(e_2) - \Delta(e_1)$, which is $\beta(e_1, e_2) = \min(20\text{:}00, 15\text{:}20 - 1$ min. 
\begin{align*}
    \text{open}_2 = \{
\langle \mathcal{RCA}&, \langle 14\text{:}00, 14\text{:}00, 15\text{:}19, 10 \text{ min}\rangle, f = 14\text{:}10, e_2\rangle,\\
\langle \mathcal{RA}&, \langle 14\text{:}00, 14\text{:}00, 20\text{:}00, 20 \min\rangle, f = 14\text{:}20, e_4\rangle,\\
\langle \mathcal{RCA}&, \langle 14\text{:}00, 15\text{:}39, 20\text{:}00, 10 \text{ min}\rangle, f = 15\text{:}49, e_3\rangle
\}
\end{align*}
Popping the front of open, we expand $\mathcal{RCA}$.
$\mathcal{RCA}$ is a goal, so we add its ATF to our under-construction cATF any-start-time solution:
\begin{align*}
    \textit{cATF}_0 = \{&[14\text{:}00, 15\text{:}19) \mapsto \langle\mathcal{RCA}, \langle 14\text{:}00, 14\text{:}00, 15\text{:}19, 10 \text{ min}\rangle, f = 14\text{:}10, e_2\rangle \} 
\end{align*}
We know our initial plan is optimal until $15\text{:}19$ because no plans in open have an earlier arrival time until the plan we found expires at $15\text{:}19$, so we run Augmented SIPP again, this time with a start time of $15\text{:}19$.
Our first expansion is the same (except for the difference in starting time), resulting in:
\begin{align*}
    \text{open}_3 = \{
\langle \mathcal{RC}&, \langle 15\text{:}19, 15\text{:}19, 20\text{:}00, 1 \min\rangle, f = 15\text{:}20, e_1\rangle,\\
\langle \mathcal{RA}&, \langle 15\text{:}19, 15\text{:}19, 20\text{:}00, 20 \min\rangle, f = 15\text{:}39, e_4\rangle
\}
\end{align*}
As with the first search, we expand $\mathcal{RC}$ however, this time the start time is too late to follow edge $e_1$, so we generate only the successor following edge $e_3$.
\begin{align*}
    \text{open}_4 = \{
\langle \mathcal{RA}&, \langle 15\text{:}19, 15\text{:}19, 20\text{:}00, 20 \min\rangle, f = 15\text{:}39, e_4\rangle,\\
\langle \mathcal{RCA}&, \langle 15\text{:}19, 15\text{:}39, 20\text{:}00, 10 \text{ min}\rangle, f = 15\text{:}49, e_3\rangle
\}
\end{align*}
We then expand $\mathcal{RA}$, which is a goal so we add it to our set of solutions:
\begin{align*}
    \textit{cATF}_1 = \{
            &[14\text{:}00, 15\text{:}19) \mapsto \langle \mathcal{RCA}, \langle 14\text{:}00, 14\text{:}00, 15\text{:}19, 10 \text{ min}\rangle, f = 14\text{:}10, e_2\rangle,\\
            &[15\text{:}19, 15\text{:}29) \mapsto \langle \mathcal{RA}, \langle 15\text{:}19, 15\text{:}19, 20\text{:}00, 20 \min\rangle, f = 15\text{:}39, e_4\rangle
    \} 
\end{align*}
Figure~\ref{fig:catf1} shows cATF$_1$, the any-start-time plan so far for our example.
The solid portions of the plan have been proven to be optimal, while the optimality of the dotted portion is unknown by the search (but known by us to be suboptimal).

\begin{figure}[t]
    \centering
    \begin{subfigure}{0.49\textwidth}
    \begin{tikzpicture}[scale = 0.65]
    \draw[thick,->] (0,0) -- (8.5,0) node[anchor=north west] {};
    \draw[thick,->] (0,0) -- (0,7.5) node[anchor=south west] {};
    \draw [color =red] (0.2, 0.2) -- node [anchor=south east] {$e_2$} (2.1,2.1);
    \draw [color =blue] (2.1,4.1) -- node [anchor=south east] {$e_4$} (4.1,6.1);
    \draw [thick,dotted,color=blue,->] (4.1,6.1) -- (5.4,7.4);

    \filldraw[black] (0.2,0) circle (0pt) node [anchor=north] {15:10};
    \filldraw[black] (2.2,0) circle (0pt) node [anchor=north] {15:20};
    \filldraw[black] (4.2,0) circle (0pt) node [anchor=north] {15:30};
    \filldraw[black] (6.2,0) circle (0pt) node [anchor=north] {15:40};
    \filldraw[black] (8.2,0) circle (0pt) node [anchor=north] {15:50};

    \filldraw[black] (0,0.2) circle (0pt) node [anchor=east] {15:20};
    \filldraw[black] (0,2.2) circle (0pt) node [anchor=east] {15:30};
    \filldraw[black] (0,4.2) circle (0pt) node [anchor=east] {15:40};
    \filldraw[black] (0,6.2) circle (0pt) node [anchor=east] {15:50};

    \filldraw[black] (4.25,-.6) circle (0pt) node [anchor=north] {Departure Time};
    \filldraw[black] (-1.5,3.75) circle (0pt) node [anchor=south, rotate=90] {Arrival Time};
    \end{tikzpicture}
    \caption{$cATF_1$}\label{fig:catf1}
    \end{subfigure}
    \begin{subfigure}{0.49\textwidth}
    \begin{tikzpicture}[scale = 0.65]
    \draw[thick,->] (0,0) -- (8.5,0) node[anchor=north west] {};
    \draw[thick,->] (0,0) -- (0,7.5) node[anchor=south west] {};
    \draw [color =red] (0.2, 0.2) -- node [anchor=south east] {$e_2$} (2.1,2.1);
    \draw [color =blue] (2.1,4.1) -- node [anchor=south east] {$e_4$} (4.1,6.1);
    \draw [color =red, ->] (4.1,6.1) -- node [anchor=south] {$e_3$} (6.1,6.1) -- (7.5, 7.5);

    \filldraw[black] (0.2,0) circle (0pt) node [anchor=north] {15:10};
    \filldraw[black] (2.2,0) circle (0pt) node [anchor=north] {15:20};
    \filldraw[black] (4.2,0) circle (0pt) node [anchor=north] {15:30};
    \filldraw[black] (6.2,0) circle (0pt) node [anchor=north] {15:40};
    \filldraw[black] (8.2,0) circle (0pt) node [anchor=north] {15:50};

    \filldraw[black] (0,0.2) circle (0pt) node [anchor=east] {15:20};
    \filldraw[black] (0,2.2) circle (0pt) node [anchor=east] {15:30};
    \filldraw[black] (0,4.2) circle (0pt) node [anchor=east] {15:40};
    \filldraw[black] (0,6.2) circle (0pt) node [anchor=east] {15:50};

    \filldraw[black] (4.25,-.6) circle (0pt) node [anchor=north] {Departure Time};
    \filldraw[black] (-1.5,3.75) circle (0pt) node [anchor=south, rotate=90] {Arrival Time};
    \end{tikzpicture}  
    \caption{$cATF_2$}\label{fig:catf2}
    \end{subfigure}
    \caption{Compound ATFs from the Delivery Robot Example (see Figure~\ref{fig:delivery2})}
    \label{fig:catfs}
\end{figure}
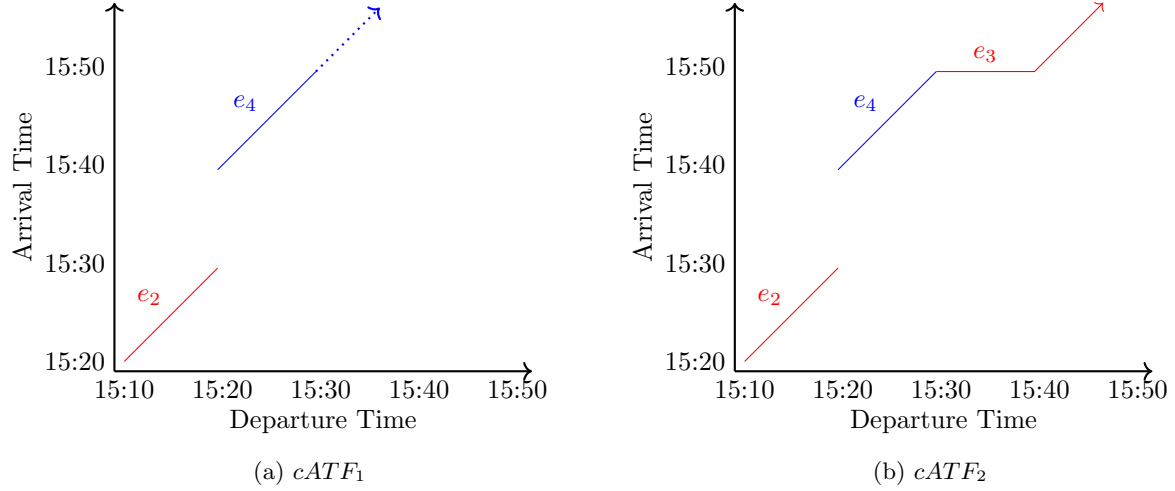

We can see that this plan will be optimal until $15\text{:}29$, at which time it will be as fast or faster to wait on edge $e_3$ rather than taking edge $e_4$.
Thus, we start our third and final Augmented SIPP search with $t_0 = 15\text{:}29$, this search is identical to the previous one (except for the starting time) except $\mathcal{RA}$ is pruned because it is dominated by the solutions already in the cATF.
Thus, the only unpruned plan left is $\mathcal{RCA}$ via edge $e_3$, which we add to the cATF.
Leading to our final set of solutions:
\begin{align*}
    \textit{cATF}_2 = \{
            &[14\text{:}00, 15\text{:}19) \mapsto \langle \mathcal{RCA}, \langle 14\text{:}00, 14\text{:}00, 15\text{:}19, 10 \text{ min}\rangle, f = 14\text{:}10, e_2\rangle,\\
            &[15\text{:}19, 15\text{:}29) \mapsto \langle \mathcal{RA}, \langle 15\text{:}19, 15\text{:}19, 20\text{:}00, 20 \min\rangle, f = 15\text{:}39, e_4\rangle,\\
            &[15\text{:}29, 20\text{:}00) \mapsto \langle \mathcal{RCA}, \langle 15\text{:}29, 15\text{:}39, 20\text{:}00, 10 \text{ min}\rangle, f = 15\text{:}49, e_3\rangle
    \} 
\end{align*}
Figure~\ref{fig:catf2} shows cATF$_2$, the any-start-time plan for our example.  With this intuitive understanding in place, we now introduce the algorithm formally.

\subsection{Solving @SIPP: RePEAT}

We describe an algorithm for solving @SIPP, called RePEAT (because it is Repeatedly Planning for Ever-Ascending Times).
It works by repeatedly calling \aSIPPab\ with a monotonically increasing starting time.
The result is a set of plans that contains the optimal plan for any start time.  (In contrast, the objective of Augmented SIPP was to find a single plan that is optimal for a specified starting time.)
Although @SIPP is defined as an offline problem, one way of thinking about RePEAT is as an anytime algorithm: the longer it runs, the further into the future we have optimal plans in our any-start-time plan.  Assuming it can progress the reference time at least as fast as time passes, RePEAT will always have an optimal solution when signaled to act.
The example in Section~\ref{sec:@SIPPexample} followed the form of RePEAT, however in this section we describe RePEAT in detail, particularly how a compound ATF can be represented efficiently and how to calculate the time until which an \aSIPPab\ plan is optimal.

\begin{algorithm}[t]
\caption{RePEAT}
\label{alg:repeat}
\begin{algorithmic}[1]
\Function{RePEAT}{$\sigma, \delta, x_o ,\isGoal$}
    \State $\trefm \leftarrow 0$ \label{algln:tref0}
    \State $\mathit{solutions} \leftarrow \{
    [0, \infty) \mapsto \langle \varnothing, I[\infty, 0], f = \infty, \varnothing \rangle
    \}$ \label{algln:solutionsinit}
    \While{$\trefm < t_{e}(\mathit{start})$}
        \State $\mathit{solution}, \mathit{ATF}, \open, \closed \leftarrow $ \Call{\aSIPPab}{$\sigma, \delta, x_o, \trefm,\isGoal$} \label{algln:runasipp}
        \State $\trefm \leftarrow $ \Call{updateReferenceTime}{$\mathit{solution}, \open$} \label{algln:updatereftime}
        \State $\mathit{solutions} \leftarrow$ \Call{compoundATFadd}{$\mathit{solution}, \mathit{solutions}$} \label{algln:catfadd}
    \EndWhile
    \State \Return $\mathit{solutions}$ \label{algln:peatend}
\EndFunction
\end{algorithmic}
\end{algorithm}

The pseudocode for RePEAT is shown in Algorithm~\ref{alg:repeat}.
The reference time \tref\ is the starting time for our \aSIPPab\ searches.
After each \aSIPPab\ search we calculate a new, later \tref\ to use as the starting time of the next round of search.
We describe how to perform this reference time update in Section~\ref{sec:reftimeupdate}.
We initially have a reference time of $\trefm = 0$ (Line~\ref{algln:tref0}), indicating that our set of plans is empty and not known to be optimal for any start time.
For simplicity, we use $0$ here as the relative time after the beginning of the interval of interest; it works just as well to initialize \tref\ to an absolute starting time or the beginning of the safe interval of the initial state.   
Our cATF $\mathit{solutions}$ is initialized to be infinity to reflect that we have no current plans that reach the goal; to accomplish this, we insert a fake solution with infinite cost that covers the entire timeline (Line~\ref{algln:solutionsinit}).
We then repeatedly run \aSIPPab\ with $t_0 = \trefm$ (Line~\ref{algln:runasipp}).
After each round of \aSIPPab, we increase the reference time to be the earliest departure time when our solution is not provably optimal (Line~\ref{algln:updatereftime}), then we add the latest \aSIPPab\ solution to our set of plans (Line~\ref{algln:catfadd}).

In the following subsections, we detail the two key operations for RePEAT beyond \aSIPPab: how the cATF is efficiently implemented and how the reference time is updated.
Note that the reference time update makes use of the open list of \aSIPPab\ as it was immediately before the goal node was selected.
This is why we have \aSIPPab\ returning the open list along with the solution (Algorithm~\ref{alg:asipp}, Line~\ref{algln:retval}).
RePEAT terminates when it has proven that it has optimal plans for the entire interval of the starting state.
When used in an anytime manner, the reference time is the departure time until which RePEAT has found optimal solutions.

\subsection{Representing Any-Start-Time Plans: Compound ATFs}

The edge and plan ATFs used by Augmented SIPP have a straightforward representation consisting of the four real-valued parameters $\zeta, \alpha, \beta,\text{ and }\Delta$.
As we have seen in our examples, the optimal plan may be different for different starting times,
so compound ATFs represent the arrival time function of a set of plans.
Formally, consider a set of plans $\mathbb{P} \subseteq E^N$ and denote the search node
associated with a specific plan as $p \in \mathbb{P}$.  
Each plan has an associated plan ATF ($A[p]$), and a cATF is the ATF of the entire set of plans.
In the context of any-start-time planning, we are finding the cATF of plans between the start and goal states, which we denote by $A[s, d]: s,d \in S$. 
The cATF converges from being an overestimate to exact, and in order of departure time as more Augmented SIPP searches are run.

We will use cATF to refer both to the set of plans $\mathbb{P}$ and its associated ATF $A[\mathbb{P}]$.
While it would be trivial to implement a cATF as a simple set of plans, if we do not track when each plan is optimal, then we would need to try all of the plans each time we use the cATF to compute an arrival time, potentially slowing down our search.
In this section, we describe our representation of a cATF and how to efficiently compute the following operations:
\begin{description}
    \item[plan insertion] inserting a new plan $p$ into $\mathbb{P}$, i.e., $\mathbb{P} \cup \{p\}$,
    \item[arrival time query] the earliest arrival time departing at time $t$, i.e., $\min_{p\in \mathbb{P}}A[p](t)$, and
    \item[plan domination check] if a plan is dominated by a cATF, i.e., $A[s,d] \trianglelefteq A[p]$.
\end{description}

\begin{figure}[b]
    \centering
    \begin{tikzpicture}[scale=1]
    \draw[thick,->] (0,0) -- (6.5,0) node[anchor=north west] {$t_{depart}$};
    \draw[thick,->] (0,3) -- (0,3.5) ;
    \draw[thick] (0,0) -- (0,2.5);
    \draw[thick,dotted] (0,2.5) -- (0,3);
    \draw[thick] (1.5 cm,5pt) -- (1.5 cm,-5pt) node[anchor=north] {$\zeta$};
    \draw[thick] (3 cm,5pt) -- (3 cm,-5pt) node[anchor=north] {$\alpha$};
    \draw[thick] (4.5 cm,5pt) -- (4.5 cm,-5pt) node[anchor=north] {$\beta$};
    \draw (5pt, 1) -- (-5pt, 1) node[anchor=east] {$A[p](\min(\alpha, \beta))$};
    \draw (5pt, 2.5) -- (-5pt, 2.5) node[anchor=east] {$A[p](\beta)$};
    
    \draw (1.5,1) -- (3,1);
    \draw (3,1) -- (4.5,2.5);
    \filldraw[black] (1.5,1) circle (3pt);
    \filldraw[black] (3,1) circle (3pt);
    \filldraw[black] (4.5,2.5) circle (3pt);
    \filldraw[white] (4.5,2.5) circle (2pt);
    \end{tikzpicture}
    \caption{A SIPP plan ATF with parameters $\zeta,\alpha,\beta$, and $\Delta$.}
    \label{fig:apsis}
\end{figure}
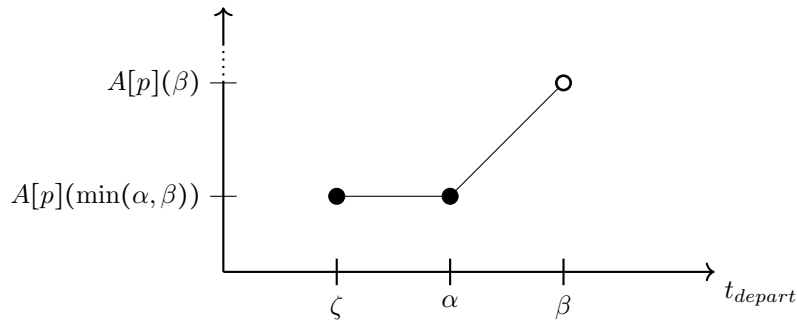

We represent a cATF as an ordered mapping of time intervals to plans ($I \mapsto \mathbb{P}$).
Recall that each plan $p\in \mathbb{P}$ has an ATF $A[p]$ of the general form shown in Figure~\ref{fig:apsis}, consisting of up to two segments: the flat segment with interval $[\zeta, \min(\alpha, \beta))$, and if $\alpha < \beta$ the rising segment with interval $[\alpha, \beta)$. 
Each interval in a cATF maps to a search node, so that the plan and its ATF corresponding to times in that interval can be quickly retrieved.
These intervals capture when each plan in the cATF is optimal, and allow for the quick retrieval of those plans.
The reason that each plan is not represented by a single interval is that an ATF might be optimal for several disjoint intervals of departure time; an example of this is shown in Figure~\ref{fig:catfaddexample}.

Formally, a cATF mapping to a set of plans $\mathbb{P}$ has the following properties:
\begin{description}
    \item[Coverage] For all times $t$, there is one and only one interval in  $\mathbb{P}$ containing $t$.
    \item[Correctness] For all intervals $i \in I$, with plan $p_i = I[i]$, 
    $\forall t \in i: A[p_i](t) = A[\mathbb{P}](t)$.
\end{description}
In RePEAT (Algorithm~\ref{alg:repeat}), $\mathit{solutions}$ is a cATF initialized to contain a single interval that maps to an ATF that is infinite for all times (Line~\ref{algln:solutionsinit}).
This satisfies coverage as all times map to the infinite ATF and correctness because no plans have been found for any time, so the earliest arrival time via that empty set of plans is always infinity.

\subsubsection{Plan Insertion}

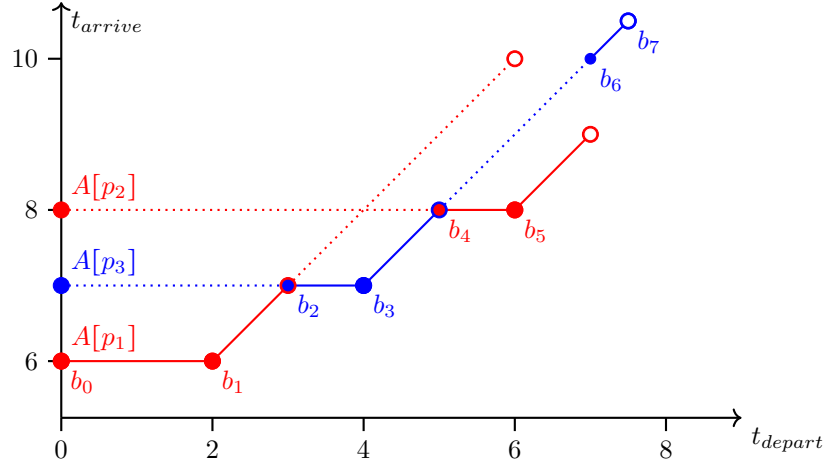
\begin{figure}[t]
    \centering
    \begin{tikzpicture}[scale = 1]
        \draw[thick,->] (0,-0.5) -- (9,-0.5) node[anchor=north west] {$t_{depart}$};
        \draw[thick,->] (0,-0.5) -- (0,5) node[anchor=north west] {$t_{arrive}$};
        
        \draw[thick] (0,-0.5) -- ++(0,-5pt) node[anchor=north] {$0$};
        \draw[thick] (2,-0.5) -- ++(0,-5pt) node[anchor=north] {$2$};
        \draw[thick] (4,-0.5) -- ++(0,-5pt) node[anchor=north] {$4$};
        \draw[thick] (6,-0.5) -- ++(0,-5pt) node[anchor=north] {$6$};
        \draw[thick] (8,-0.5) -- ++(0,-5pt) node[anchor=north] {$8$};

        \draw[thick] (0,0.25) -- (-5pt,0.25) node[anchor=east] {$6$};
        \draw[thick] (0,2.25) -- (-5pt,2.25) node[anchor=east] {$8$};
        \draw[thick] (0,4.25) -- (-5pt,4.25) node[anchor=east] {$10$};

        \draw[thick,red] (0, 0.25) node[anchor=south west] {$A[p_1]$} -- (2, 0.25)  -- (3, 1.25);
        \draw[thick,red,dotted] (3, 1.25)  -- (6, 4.25);
        \filldraw[red] (0,0.25) circle (3pt);
        \filldraw[red] (2,0.25) circle (3pt);
        \filldraw[red] (6,4.25) circle (3pt);
        \filldraw[white] (6,4.25) circle (2pt);

        \draw[thick,red,dotted] (0, 2.25) node[anchor=south west] {$A[p_2]$} -- (5, 2.25);
        \draw[thick,red] (5, 2.25) -- (6, 2.25)  -- (7, 3.25);
        \filldraw[red] (0,2.25) circle (3pt);
        \filldraw[red] (6,2.25) circle (3pt);
        \filldraw[red] (7,3.25) circle (3pt);
        \filldraw[white] (7,3.25) circle (2pt);

        \draw[thick,blue, dotted] (0, 1.25) node[anchor=south west] {$A[p_3]$} -- (3, 1.25);
        \draw[thick,blue] (3, 1.25) -- (4, 1.25)  -- (5, 2.25);
        \draw[thick,blue,dotted] (5, 2.25)  -- (7, 4.25);
        \draw[thick,blue] (7, 4.25)  -- (7.5, 4.75);
        \filldraw[blue] (0, 1.25) circle (3pt);
        \filldraw[blue] (4, 1.25) circle (3pt);
        \filldraw[blue] (7.5, 4.75) circle (3pt);
        \filldraw[white] (7.5, 4.75) circle (2pt);
        
        \filldraw[red] (0, 0.25) circle (3pt) node[below right] {$b_0$};
        
        \filldraw[red] (2, 0.25) circle (3pt) node[below right] {$b_1$};
        
        \filldraw[red] (3, 1.25) circle (3pt);
        \filldraw[blue] (3, 1.25) circle (2pt) node[below right] {$b_2$};
        
        \filldraw[blue] (4, 1.25) circle (3pt) node[below right] {$b_3$};

        \filldraw[blue] (5, 2.25) circle (3pt);
        \filldraw[red] (5, 2.25) circle (2pt) node[below right] {$b_4$};
        
        \filldraw[red] (6, 2.25) circle (3pt) node[below right] {$b_5$};

        \filldraw[blue] (7, 4.25) circle (2pt) node[below right] {$b_6$};
        
        \filldraw[blue] (7.5, 4.75) circle (3pt) node[below right] {$b_7$};
        \filldraw[white] (7.5, 4.75) circle (2pt);

    \end{tikzpicture}
    \caption{Compound ATF plan insertion example: $p_3$ is inserted into $\{ p_1, p_2 \}$.  The three plan ATFs $A[p_1],A[p_2],$ and $A[p_3]$ are shown with segments strongly dominated by $\{p_1, p_2, p_3 \}$ dotted.  The breakpoints of $\{p_1, p_2, p_3 \}$ are $b_{0},...,b_{7}$.}
    \label{fig:catfaddexample}
\end{figure}

Plan insertion works by computing the lower hull of the cATF and the new plan ATF.
For example, Figure~\ref{fig:catfaddexample} shows the addition of a new plan ($p_3$, in blue) to an existing cATF ($\{ p_1, p_2 \}$, in red).
The plan ATFs in this example are:
\begin{align*}
    A[p_1] =& \langle 0, 2, 6, 4 \rangle\\
    A[p_2] =& \langle 0, 6, 7, 2 \rangle\\
    A[p_3] =& \langle 0, 4, 7.5, 3 \rangle
\end{align*}
And the cATF is originally:
\begin{align*}
    A[\{ p_1, p_2 \}] = \{ 
    [&0, 4) \mapsto p_1,\\
    [&4, 7) \mapsto p_2,\\
    [&7, \infty) \mapsto I[\infty, 0]
    \}  
\end{align*}
We add the plan $p_3$, taking the lower hull, i.e., finding the new mapping resulting from minimizing over the set with the addition of the new node:
\begin{align*}
    A[\{ p_1, p_2, p_3 \}] = \{ 
    [&0, 3) \mapsto p_1,\\
    [&3, 5) \mapsto p_3,\\
    [&5, 7) \mapsto p_2,\\
    [&7, 7.5) \mapsto p_3,\\
    [&7.5, \infty) \mapsto I[\infty, 0]
    \}  
\end{align*} 
In particular, note that $p_3$ is dominant for two distinct intervals: the first until $t = 5$ because the shorter (in movement duration) plan $p_2$ requires waiting until $t=6$ and the second after $p_2$ becomes unsafe at $t=7$.
Handling this splitting simply and efficiently is one reason that our representation considers the flat and rising segments separately.


\begin{algorithm}[t]
\caption{add plan ATF $p$ to compound ATF $c$}
\label{alg:insertplan}
\begin{algorithmic}[1]
\Function{insert}{$p, c$} \label{alg:compoundatfadd}
    \If{$\beta(p) \leq \alpha(p)$} \label{algln:mksegments}
        \State \Call{insertSegment}{$[\zeta(p), \beta(p)), p$}
    \Else
        \State \Call{insertSegment}{$[\zeta(p), \alpha(p)), p$}
        \State \Call{insertSegment}{$[\alpha(p), \beta(p)), p$}
    \EndIf
\EndFunction

\Function{insertSegment}{$c, i, p$}
    \State $c \gets c \cup \{ i \mapsto p\}$\label{algln:catfinsert} \Comment{Keep track of overlaps}
    \For{Segment $s \in c: s $ overlaps $i$} \label{algln:overlaps}
        \State $\mathit{incumbent, newcomer} \gets \mathit{candidates}(s)$ \Comment{Consider the two overlapping segments} 
        \If{$\mathit{incumbent} \trianglelefteq_{ATF} \mathit{newcomer}$} \label{algln:incdom}
            \State $c \gets c \setminus \mathit{newcomer}$ \label{algln:removenewcomer}
        \ElsIf{$\mathit{newcomer} \trianglelefteq_{ATF} \mathit{incumbent}$} \label{algln:newdom}
            \State $c \gets c \setminus \mathit{incumbent}$ \label{algln:removeincumbent}
        \Else \Comment{Segments must intersect}
            \State $t \gets \mathit{intersection}(\mathit{incumbent}, \mathit{newcomer})$\label{algln:intersection} \Comment{Finds the `x' value where two line segments intersect}
            \State $t_b, t_e \gets \mathit{interval}(s)$
            \State $a = ATF[\mathit{incumbent}], b = ATF[\mathit{newcomer}]$
            \If{$a(t_b) > b(t_b)$}
                \State $a = ATF[\mathit{newcomer}], b =  ATF[\mathit{incumbent}]$ \Comment{Order the segments}
            \EndIf
            \State $\mathit{newSegments} \gets \{ [t_b, t) \mapsto a , [t, t_e) \mapsto b \}$ \label{algln:newsegments}
            \State $c \gets (c \setminus \mathit{incumbent} ) \setminus \mathit{newcomer} $
            \State $c \gets c \cup \mathit{newSegments} $
        \EndIf
    \EndFor
\EndFunction
\end{algorithmic}
\end{algorithm}

Algorithm~\ref{alg:insertplan} shows the pseudocode of the ATF insertion operation.
For our implementation, we have used a boost::icl::split\_interval\_map as the underlying data structure for the compound ATF \cite{boost_icl}.
This data structure takes care of finding the time intervals of overlaps between the new segment being inserted and the segments already in the compound ATF.
What remains is to decide which parts of the two segments are retained.
To add a plan ATF we divide it into its segment(s) (Line~\ref{algln:mksegments}).
If $\beta \leq \alpha$, then there is only one segment, rather than the usual two.

We then separately insert each segment into the compound ATF (Line~\ref{algln:catfinsert}).
The split interval map breaks the timeline into time intervals, one for each segment, except when an existing `incumbent' segment overlaps the newly added segment.  
We consider these time-overlapping segments (Line~\ref{algln:overlaps}) and first check if the incumbent (Line~\ref{algln:incdom}) or the newcomer (Line~\ref{algln:newdom}) dominates for the entire overlap.
In these cases, the dominant one can be retained and the other removed (Lines~\ref{algln:removenewcomer},\ref{algln:removeincumbent}).
Otherwise, we are guaranteed that the two line segments intersect, so we calculate the time $t$ when this intersection occurs (Line~\ref{algln:intersection}) and replace the two overlapping segments with the appropriate non-overlapping sub-segments (Line~\ref{algln:newsegments}).

The time complexity of the insert operation is linear in the number of segments in the overlap (Line~\ref{algln:overlaps}), and logarithmic in the size of the compound ATF due to the binary search to find the extent of the interval, thus the worst case complexity of the compound ATF insertion operation is $O(|c|\ln{(|c| + |p|)})$.

\subsubsection{Calculating Arrival Times with a Compound ATF}
Querying a compound ATF is a straightforward operation: we perform a binary search over the ordered set of time intervals to find the one containing the query time.
We then compute the arrival time using the ATF mapped from that interval.
The runtime of this operation is $O(\log(c))$ because of the binary search.
We will explore the on-line performance of querying compound ATFs in the empirical evaluation, but we expect this operation to be very fast in practice.

\subsubsection{Compound ATF Dominance}
The operation to compute if an ATF is dominated by another ATF is the same whether those ATFs are plan ATFs or compound ATFs.  
It is sufficient to consider only the breakpoints of the two ATFs when checking for dominance (Definition~\ref{def:atfdom}), because the ATFS are monotonic non-decreasing, and can only increase discontinuously at breakpoints.
Any change in which ATF is dominant (called a minimization breakpoint) is guaranteed to be preceded and followed by breakpoints on the same segments.
See Section~\ref{sec:theory} for a detailed exploration of the theoretical properties of compound ATFs, which is focused on understanding their breakpoints.
With this, checking plan domination amounts to iterating through the breakpoints of the two ATFs, and checking if one always arrives at least as early as the other.
Because we maintain both ATFs (e.g., $A, A'$) in order, the runtime is $O(|A| + |A'|)$.  
For a pair of plan ATFs this is $O(1)$, but if there is a compound ATF, in the worst case, we might need to check all of the breakpoints to check dominance.

\subsection{Updating the Reference Time}\label{sec:reftimeupdate}

\begin{algorithm}[t]
\caption{updateReferenceTime}
\label{alg:updatereftime}
\begin{algorithmic}[1]
\Function{updateReferenceTime}{solution, open}
    \If{$|\open| = 0$} \label{algln:emptycase}
        \State \Return $\infty$
    \EndIf
    \State $\mathit{lowerBound} \leftarrow \alpha(\mathit{solution})$ \label{algln:urtlower}
    \While{$\lnot \mathit{empty}(\mathit{open}) \land \mathit{lowerBound} < \beta(\mathit{solution})$} 
        \State $n \leftarrow \mathit{pop}(\mathit{open})$ \label{algln:urtpopopen} 
        \State $\mathit{lowerBound} \leftarrow \alpha(n) + \Delta(n) + h(n) - \Delta(solution)$ \label{algln:urtupdateopen}
        \If{$\alpha(n) > \mathit{lowerBound}$} \label{algln:thecheck} 
            \State \Return $\mathit{lowerBound}$ \label{algln:urtfoundbetter}
        \EndIf
    \EndWhile
    \State \Return $\beta(\mathit{solution})$ \label{algln:urtnocandidates}
\EndFunction
\end{algorithmic}
\end{algorithm}

The reference time (\tref) tracks the earliest time that we do not yet know that we have an optimal solution for.
A simple approach might observe that any solution, $n$, that is optimal at \tref{} is optimal until at least $\alpha(n)\geq \trefm$; however, this simple approach would only lead us to a non-decreasing reference time. 
Our algorithm for updating the reference time is described in Algorithm~\ref{alg:updatereftime}.
It finds a reference time that increases monotonically with each solution returned by Augmented SIPP.

When Augmented SIPP terminates, it returns a solution with node $n$, that:
\begin{itemize}
    \item is optimal at the reference time (Theorem~\ref{thm:absippoptimal}), and at least until $\alpha(n)$ (Equation~\ref{eqn:edgearrival}), and
    \item dominates all other plans that are optimal at the reference time (Theorem~\ref{thm:asippdom}), but
    \item does not necessarily dominate the remaining plans (that may be optimal at later times than the reference time), particularly after $\beta(n)$.
\end{itemize}
So the objective of our reference time update is to calculate the latest new reference time \tref'  such that the planner can prove that the plan represented by $n$ is optimal until at least \tref'. 
Our approach does this by examining where the ATF of other nodes on open intersects with the ATF of $n$.
With an admissible heuristic, the earliest intersection is the earliest time the plan corresponding to node $n$ could possibly no longer be optimal. 

Algorithm~\ref{alg:updatereftime} shows the reference time update.  
If open is empty (Line~\ref{algln:emptycase}), then no (non-dominated) plan was found by Augmented SIPP, so the new reference time is $\trefm' =\infty$.
Otherwise, the reference time will be in the interval $\trefm' \in (\alpha(\mathit{solution}), \beta(\mathit{solution}))$.  
We then look at each node $n$ on open, geometrically we know that the intersection must occur at the intersection of the rising segment of $\mathit{solution}$ and the flat segment of $n$, (two segments with the same slope being unable to intersect, and being constrained to the rising segment of $\mathit{solution}$), so we calculate the (hypothetical) intersection point in departure time as $\mathit{lowerBound}$ (Line~\ref{algln:urtupdateopen}).
If the hypothetical intersection is really an intersection (Line~\ref{algln:thecheck}), then we return the departure time of the intersection point as the new reference time.
Otherwise, if we find an intersection time that would be after $\beta(\mathit{solution})$, then we know there will be no intersection and return $\beta(\mathit{solution})$.

In Algorithm~\ref{alg:updatereftime}, we do an in-order traversal of open, meaning if open is implemented as a heap our worst-case runtime would be $O(|\open|\log (|open|))$ if we had to try every node on open.
It would also be valid to do an unordered traversal, tracking the earliest intersection, which could presumably be implemented in $O(|\open|)$.
We do the in-order traversal with the assumption that it is usually not actually necessary to traverse all of open.
Either way, we are not changing the overall runtime because we had to construct open during the search.

We now prove some theoretical properties of RePEAT:
\begin{corollary}\label{cor:trefinc}
    \tref\ increases monotonically with new solutions.
\end{corollary}
\begin{proof}
    By Theorem \ref{thm:asippdom}, our solution is optimal up-to and including $\trefm$, and dominates all other nodes that are optimal at \tref.  
    Therefore, the earliest arrival time departing at the new reference time, $\trefm'$, is later than the earliest arrival time departing at the former reference time, \tref.  
    Compound arrival time functions are monotonic non-decreasing (we prove this assertion in Lemma~\ref{lmm:monotone}), so $\trefm < \trefm'$.
\end{proof}
This allows us to show that:
\begin{theorem}
    RePEAT is complete and optimal
\end{theorem}
\begin{proof}
RePEAT has an optimal solution until at least \tref. 
\tref\ is initialized to 0, then \aSIPPab\ finds an optimal solution for \tref.  
Corollary~\ref{cor:trefinc} shows that \tref\ increases monotonically as we find new solutions. 
So RePEAT is optimal until \tref, and \tref\ is monotonic increasing, so by induction RePEAT finds an optimal solution for all time.
\end{proof}

\subsection{Replanning SIPP}

\begin{algorithm}[t]
\caption{Replanning SIPP}
\label{alg:rsipp}
\begin{algorithmic}[1]
\Function{RSIPP}{$\mathit{startConfig}, \mathit{goalConfig}$}
\State $validPlan \leftarrow False$
\State Wait until $t_{0}$
\While{true}
    \State p $\leftarrow$ \Call{SIPP}{$\langle \mathit{startConfig}, t_{\mathit{now}} \rangle, \mathit{goalConfig}$}\label{algln:rsippcallsipp}
    \If{p still valid}
        \State \Return p\label{algln:rsippend}
    \EndIf
\EndWhile
\EndFunction
\end{algorithmic}
\end{algorithm}

As an alternative to full any-start-time planning, one might consider a simple replanning approach, which we call RSIPP.  
RSIPP approximates an online solution to @SIPP by using an offline planner at query time.  
It does no computation ahead of time, but just waits until the start time.  
It then runs SIPP starting from the current time $t_{now}$ (Line~\ref{algln:rsippcallsipp}).
If the found plan is still valid after planning completes, we return it, otherwise we loop back and run SIPP again.
If the problem is simple enough and the planner fast enough, this should return a valid plan in time to be used.
However, it does not truly solve the @SIPP problem, as we are only returning a single plan.

\section{On the Complexity of Any-start-time Planning}\label{sec:theory}
In this section, we explore the theoretical complexity of any-start-time planning and RePEAT.  
We start by analyzing the space complexity of any-start-time plans, finding them at most linear in the size of the input graph.
Then, in Subsection~\ref{sec:time}, we prove the time complexity of RePEAT quasi-quadratic in the size of the input graph.  (We will experimentally evaluate the runtime of RePEAT on SIPP problems in Section~\ref{sec:empiricaleval}.)
These theoretical results highlight the utility of any-start-time planning as the subset of the TDSP where the any-start-time plan is relatively easy to compute.

\subsection{The Space Complexity of Any-Start-Time Plans}\label{sec:complexity}

In this subsection, we discuss the complexity of the solution to an any-start-time problem.
If we are to use cATFs to represent any-start-time plans, the crucial question is how many distinct breakpoints will need to be represented.  
Can the required number ever grow multiplicatively?
We will find that the compound ATF solving an any-start-time problem is linear in the size of the input graph.  This is a direct result of any-start-time planning being a special case of the TDSP.
We define the size of the input graph to mean the number of safe intervals on states and actions.  
In an infinite state space, this is the size of the local search space that encompasses all states explored by any of the implicit Augmented SIPP searches.

\textcite{foschini2014} find that a TDSP with piecewise linear edge durations has a super-polynomial solution complexity.
However, when the edge ATFs consist of segments with slopes constrained to $0, 1$, or $\infty$, the solution is linear in the size of the graph.
This complexity was also proved by \textcite{dehne2012shortest}, who explored a similar setting where the objective is to find the time-dependent shortest path to all other states.
In this section, we adapt the proof to the specific context of any-start-time planning and then describe a simple compilation from an any-start-time planning problem to a time-dependent shortest path problem.
The space of these compiled any-start-time problems is exactly the set of TDSP problems \citeauthor{foschini2014} found to be linear.

The proof relies on counting the breakpoints in the resulting piecewise linear cATF.
We begin by more precisely defining plans and the features of their ATFs.
An Augmented SIPP \textbf{plan} $p$ is an interleaved ordered sequence of states and edges $s_0,e_1,s_1,...,e_i,s_i,...,e_n,s_n$ with $e \in E$ and $s \in S$.  
Note that in this scheme the sequence of states uses zero-based indexing, while the sequence of edges uses one-based indexing.
For notational simplicity, we often choose to drop one of the components of a plan.
For example, we often only consider the edges of a plan $p$, which we will denote by $e_1,...,e_n \in p$, as the safety information of the states has been compiled into the edge ATFs.    
A plan $p$ is \textbf{valid} at time $t$ if and only if, for all edges $e_i$ of the plan, $s_{i-1}$ is the source of $e_i$ and $s_{i}$ is the destination of $e_i$ and the plan ATF evaluated at time $t$ is finite ($A[p](t) \in \mathbb{R}$).
A \textbf{partial plan} $p_{i,j} := s_{i-1},e_i,s_i,...,e_j,s_j \in p$ is a subsequence of the full plan $p$ and $e_i \in p$ denotes the $i$th edge of $p$ (counting from $e_1$), and $s_i \in p$ denotes the $i$th state (counting from $s_0$)
.
\begin{definition}
    A \textbf{breakpoint} of an ATF is a point where the slope of the ATF changes or the underlying plan changes. 
\end{definition}

An edge ATF has up to three breakpoints: at $\zeta$, $\alpha$, and $\beta$.
Similarly, a plan ATF also has three breakpoints at $\zeta$, $\alpha$, and $\beta$, each of which corresponds to a breakpoint of some edge along the plan, with their breakpoints shifted by the ATF composition operation to be with respect to the start-time of the entire plan.  
For this reason, we say that a breakpoint of a plan ATF is the image of a breakpoint of one of the edge ATFs along the plan.
Because composition does not change $\zeta$ (Equation~\ref{eqn:pathzeta}), it is always an image of $\zeta(e_1)$ where $e_1 \in p$.
\begin{definition}\label{def:alphaconstraint}\label{def:betaconstraint}
We say that a breakpoint $\omega$ of a plan ATF $A[p]$ is an image that is \textbf{constrained} by edge breakpoint $b \in A[e_i], e_i \in p$ if and only if $(b - \Delta(A[e_i])) = \omega$, i.e., it has won the successive maximizations or minimization from the composition operation (Equations~\ref{eqn:pathalpha},\ref{eqn:pathbeta}) and appeared in the result. 
\end{definition} 

\begin{figure}[t]
\centering
\begin{tikzpicture}[scale=0.6]
\draw (-0.3 cm,1pt) -- (-0.3 cm,-1pt) node[anchor=north] {(a)};
\draw[thick,->] (0,0) -- (6.5,0) node[midway, below] {t};
\draw[thick,->] (0,0) -- (0,4) node[anchor=south west] {A[s,d](t)};
\filldraw[black] (3,2) circle (2pt) node [anchor=north] {b};
\filldraw[red] (0.5,1) circle (2pt) node [anchor=north] {$\zeta_1$};
\filldraw[red] (2,1) circle (2pt) node [anchor=north] {$\alpha_1$};
\filldraw[red] (4.5,3.5) circle (2pt) node [anchor=north] {$\beta_1$};
\filldraw[blue] (1,2) circle (2pt) node
[anchor=north] {$\zeta_2$};
\filldraw[blue] (5,2) circle (2pt) node [anchor=north] {$\alpha_2$};
\filldraw[blue] (6,3) circle (2pt) node [anchor=north] {$\beta_2$};
\draw [color =red] (0.5,1) -- (2,1);
\draw [color =red] (2,1) -- (3,2);
\draw [color =red,dashed] (3,2) -- (4.5,3.5);
\draw [color =blue,dashed] (1,2) -- (3,2);
\draw [color =blue] (3,2) -- (5,2);
\draw [color =blue] (5,2) -- (6,3);
\end{tikzpicture}
\begin{tikzpicture}[scale=0.6]
\draw (-0.3 cm,1pt) -- (-0.3 cm,-1pt) node[anchor=north] {(b)};
\draw[thick,->] (0,0) -- (6.5,0) node[midway, below] {t};
\draw[thick,->] (0,0) -- (0,4) node[anchor=south west] {A[s,d](t)};
\filldraw[black] (2.5,2) circle (2pt) node [anchor=south] {b};
\filldraw[red] (0.5,1) circle (2pt) node [anchor=north] {$\zeta_1$};
\filldraw[red] (2,1) circle (2pt) node [anchor=north] {$\alpha_1$};
\filldraw[red] (2.5,1.5) circle (2pt) node [anchor=north] {$\beta_1$};
\filldraw[blue] (1,2) circle (2pt) node
[anchor=north] {$\zeta_2$};
\filldraw[blue] (5,2) circle (2pt) node [anchor=north] {$\alpha_2$};
\filldraw[blue] (6,3) circle (2pt) node [anchor=north] {$\beta_2$};
\draw [color =red] (0.5,1) -- (2,1);
\draw [color =red] (2,1) -- (2.5,1.5);
\draw [color =red,dashed] (2.5,1.5) -- (2.5,2);
\draw [color =blue,dashed] (1,2) -- (2.5,2);
\draw [color =blue] (2.5,2) -- (5,2);
\draw [color =blue] (5,2) -- (6,3);
\end{tikzpicture}
\begin{tikzpicture}[scale=0.6]
\draw (-0.3 cm,1pt) -- (-0.3 cm,-1pt) node[anchor=north] {(c)};
\draw[thick,->] (0,0) -- (6.5,0) node[midway, below] {t};
\draw[thick,->] (0,0) -- (0,4) node[anchor=south west] {A[s,d](t)};
\filldraw[black] (3,2.5) circle (2pt) node [anchor=south] {b};
\filldraw[red] (0.5,1) circle (2pt) node [anchor=north] {$\zeta_1$};
\filldraw[red] (2.5,1) circle (2pt) node [anchor=north] {$\alpha_1$};
\filldraw[red] (3,1.5) circle (2pt) node [anchor=north] {$\beta_1$};
\filldraw[blue] (1,2) circle (2pt) node
[anchor=north] {$\zeta_2$};
\filldraw[blue] (2.5,2) circle (2pt) node [anchor=north] {$\alpha_2$};
\filldraw[blue] (3.5,3) circle (2pt) node [anchor=north] {$\beta_2$};
\draw [color =red] (0.5,1) -- (2.5,1);
\draw [color =red] (2.5,1) -- (3,1.5);
\draw [color =red,dashed] (3,1.5) -- (3,2.5);
\draw [color =blue,dashed] (1,2) -- (2.5,2);
\draw [color =blue,dashed] (2.5,2) -- (3,2.5);
\draw [color =blue] (3,2.5) -- (3.5,3);
\end{tikzpicture}
  \caption{\label{fig:compoundatf}
   Examples of the interaction of two plan ATFs $A[p_1]$ and $A[p_2]$ in forming the cATF $A[s,d]$}.
\end{figure}
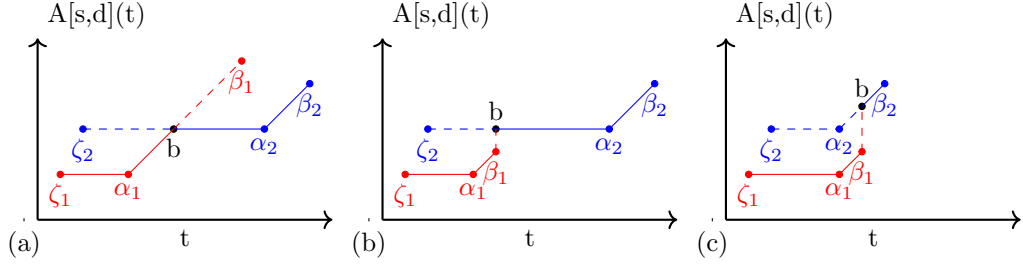

A cATF, being composed of many plan ATFs, may have more breakpoints.  
We follow \citeauthor{foschini2014} and partition the breakpoints of a cATF into two types.
\begin{definition}
    A \textbf{primitive image} is a breakpoint in a cATF $A[s,d]$ that is also a breakpoint in one of the plan ATFs ($A[p] \in A[s,d]$) in the cATF.
\end{definition}
\begin{definition}
    A \textbf{minimization breakpoint} is a breakpoint in a cATF $A[s,d]$ that occurs when the minimizing plan ATF $A[p_1],A[p_2] \in A[s,d]$ changes.
\end{definition}
For example, Figure~\ref{fig:compoundatf} shows several examples of cATFs.  The breakpoints of each plan ATF $A[p]$ are labeled $\zeta_i$, $\alpha_i$ or $\beta_i$.
Regions minimizing $A[s,d]$ are shown as a solid line while the remainder of the plan ATF is dashed.
Minimization breakpoints are black dots labeled $b$. 
Discontinuities in $A[s,d]$ are highlighted with a dashed line.
Figure~\ref{fig:compoundatf}a shows the unique configuration leading to a distinct breakpoint $b$ that was not originally a breakpoint of either constituent plan.  
The breakpoints $\zeta_1, \alpha_1, \alpha_2$ and, $\beta_2$ are primitive images, as they are the time-shifted breakpoints of an edge ATF.  
Breakpoint $b$ is a minimization breakpoint caused by the intersection of two ATFs.
In other cases, such as Figures~\ref{fig:compoundatf}b and c, a minimization breakpoint occurs at the same departure time as a primitive image $\beta$ when an ATF expires.

\begin{definition}
    A breakpoint $b$ is \textbf{distinct} from a set of breakpoints $B$ if and only if all breakpoints in $B$ occur at times different from $b$.
\end{definition}

Our proof of the linear complexity of the Any-start-time solution (see Section~\ref{sec:aspcomplex}) will constrain the number of each type of breakpoint.
In overview, each breakpoint of an edge ATF in the Augmented SIPP graph can result in only one primitive image in the solution and every pair of minimization breakpoints must have at least one primitive image between them.  
This means the solution has breakpoints linear in the number of edges in the graph.

\subsubsection{Illustrating the Possible Types of ATF Breakpoints}

To make it easier to follow the complexity proof, we illustrate 
examples of each combination of plan ATFs that lead to each type of compound ATF breakpoint: $\zeta$ image (see Figure~\ref{fig:zetabreakex}), $\beta$ image (see Figure~\ref{fig:betabreakex}), $\alpha$ image (see Figure~\ref{fig:alphabreakex}), and minimization breakpoints (see Figure~\ref{fig:minbreakex}).
The key aspect we are concerned with is: what is the slope immediately prior to the breakpoint and what is the slope immediately after.
For each type of breakpoint, only certain combinations of these two slopes are possible.
Individually, the slopes present in ATFs are: $\{ -\infty, 0, 1, \infty\}$, where $-\infty$ corresponds to the discontinuous decrease prior to $\zeta$ (where the ATF is $\infty$ because the source state is unsafe) and $\infty$ is the discontinuous increase that can happen when a plan expires (at $\beta$).

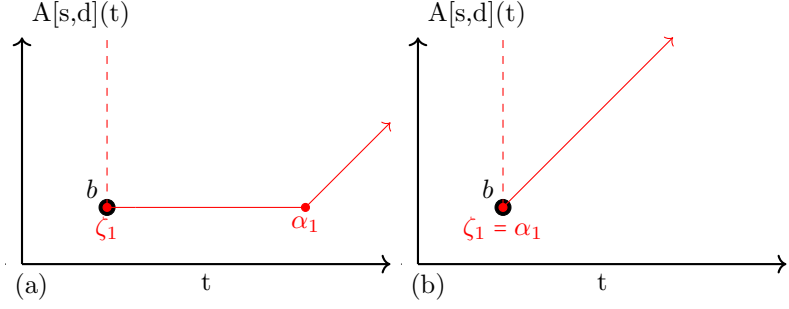
\begin{figure}[t]
\centering
\begin{tikzpicture}[scale=0.75]
\draw (-0.3 cm,0pt) -- (-0.3 cm,-0pt) node[anchor=north west] {(a)};
\draw[thick,->] (0,0) -- (6.5,0) node[midway, below] {t};
\draw[thick,->] (0,0) -- (0,4) node[anchor=south west] {A[s,d](t)};
\filldraw[black] (1.5,1) circle (4pt) node [anchor=south east] {$b$};
\draw [red, dashed] (1.5,1) -- (1.5,4);

\filldraw[red] (1.5,1) circle (2pt) node [anchor=north] {$\zeta_1$};
\filldraw[red] (5,1) circle (2pt) node [anchor=north] {$\alpha_1$};
\draw [color =red] (1.5,1) -- (2,1);
\draw [color =red] (2,1) -- (5,1);
\draw [color =red, ->] (5,1) -- (6.5,2.5);
\end{tikzpicture}
\begin{tikzpicture}[scale=0.75]
\draw (-0.3 cm,0pt) -- (-0.3 cm,-0pt) node[anchor=north west] {(b)};
\draw[thick,->] (0,0) -- (6.5,0) node[midway, below] {t};
\draw[thick,->] (0,0) -- (0,4) node[anchor=south west] {A[s,d](t)};
\filldraw[black] (1.5,1) circle (4pt) node [anchor=south east] {$b$};
\draw [red, dashed] (1.5,1) -- (1.5,4);

\filldraw[red] (1.5,1) circle (2pt) node [anchor=north] {$\zeta_1 = \alpha_1$};
\draw [color =red, ->] (1.5,1) -- (4.5,4);
\end{tikzpicture}
  \caption{\label{fig:zetabreakex}
   Examples of the two possible configurations for a $\zeta$ image, $b$.}
\end{figure}
Figure~\ref{fig:zetabreakex} shows the two combinations of slopes that are possible for a $\zeta$ image.
In both cases, the prior slope is $-\infty$ because there is no plan prior to $\zeta$.
Because the plan is valid, there can not be a slope of $\infty$ immediately after (Equation~\ref{eqn:positiveduration}).
Figure~\ref{fig:zetabreakex}a shows the first possible case, where the prior slope is $-\infty$ and $\zeta < \alpha$ so the after slope is $0$.
In this case the $\zeta$ image is distinct because it does not share the same time with other breakpoints.
Figure~\ref{fig:zetabreakex}b shows the second possible case, where the prior slope is still $-\infty$, but $\zeta = \alpha$ so the after slope is $1$.
In this case the $\zeta$ image and the $\alpha$ image share the same time, so the $\zeta$ image is not distinct.

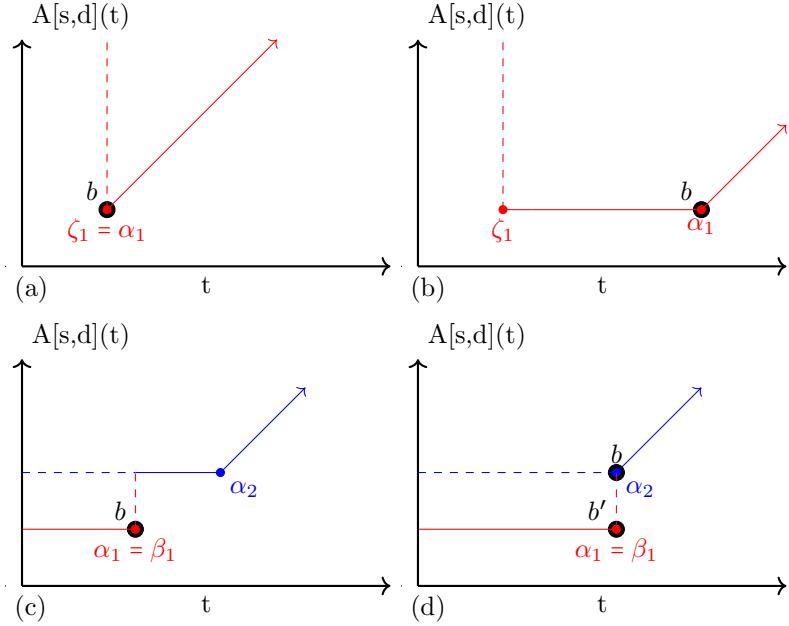
\begin{figure}[t]
\centering
\begin{tikzpicture}[scale=0.75]
\draw (-0.3 cm,0pt) -- (-0.3 cm,-0pt) node[anchor=north west] {(a)};
\draw[thick,->] (0,0) -- (6.5,0) node[midway, below] {t};
\draw[thick,->] (0,0) -- (0,4) node[anchor=south west] {A[s,d](t)};
\filldraw[black] (1.5,1) circle (4pt) node [anchor=south east] {$b$};
\draw [red, dashed] (1.5,1) -- (1.5,4);

\filldraw[red] (1.5,1) circle (2pt) node [anchor=north] {$\zeta_1 = \alpha_1$};
\draw [color =red, ->] (1.5,1) -- (4.5,4);
\end{tikzpicture}
\begin{tikzpicture}[scale=0.75]
\draw (-0.3 cm,0pt) -- (-0.3 cm,-0pt) node[anchor=north west] {(b)};
\draw[thick,->] (0,0) -- (6.5,0) node[midway, below] {t};
\draw[thick,->] (0,0) -- (0,4) node[anchor=south west] {A[s,d](t)};
\filldraw[black] (5,1) circle (4pt) node [anchor=south east] {$b$};
\draw [red, dashed] (1.5,1) -- (1.5,4);

\filldraw[red] (1.5,1) circle (2pt) node [anchor=north] {$\zeta_1$};
\filldraw[red] (5,1) circle (2pt) node [anchor=north] {$\alpha_1$};
\draw [color =red] (1.5,1) -- (2,1);
\draw [color =red] (2,1) -- (5,1);
\draw [color =red, ->] (5,1) -- (6.5,2.5);
\end{tikzpicture}

\begin{tikzpicture}[scale=0.75]
\draw (-0.3 cm, 0pt) -- (-0.3 cm,-0pt) node[anchor=north west] {(c)};
\draw[thick,->] (0,0) -- (6.5,0) node[midway, below] {t};
\filldraw[black] (2,1) circle (4pt) node [anchor=south east] {$b$};

\filldraw[red] (2,1) circle (2pt) node [anchor=north] {$\alpha_1=\beta_1$};
\filldraw[blue] (3.5,2) circle (2pt) node [anchor=north west] {$\alpha_2$};
\draw [color =red] (0,1) -- (2,1);
\draw [red, dashed] (2,1) -- (2, 2);
\draw [color =blue,dashed] (0,2) -- (2,2);
\draw [color =blue] (2,2) -- (3.5,2);
\draw [color =blue,->] (3.5,2) -- (5,3.5);
\draw[thick,->] (0,0) -- (0,4) node[anchor=south west] {A[s,d](t)};
\end{tikzpicture}
\begin{tikzpicture}[scale=0.75]
\draw (-0.3 cm, 0pt) -- (-0.3 cm,-0pt) node[anchor=north west] {(d)};
\draw[thick,->] (0,0) -- (6.5,0) node[midway, below] {t};
\filldraw[black] (3.5,1) circle (4pt) node [anchor=south east] {$b'$};
\filldraw[black] (3.5,2) circle (4pt) node [anchor=south] {$b$};

\filldraw[red] (3.5,1) circle (2pt) node [anchor=north] {$\alpha_1=\beta_1$};
\filldraw[blue] (3.5,2) circle (2pt) node [anchor=north west] {$\alpha_2$};
\draw [color =red] (0,1) -- (3.5,1);
\draw [red, dashed] (3.5,1) -- (3.5, 2);
\draw [color =blue,dashed] (0,2) -- (3.5,2);
\draw [color =blue,->] (3.5,2) -- (5,3.5);
\draw[thick,->] (0,0) -- (0,4) node[anchor=south west] {A[s,d](t)};
\end{tikzpicture}
  \caption{\label{fig:alphabreakex}
   Example of the four possible configurations for an $\alpha$image, $b$.}
\end{figure}
Figure~\ref{fig:alphabreakex} shows the four combinations of slopes that are possible for an $\alpha$ image.
There are two possible prior slopes.
First if $\alpha = \zeta$, then we are in the same case as the example shown in Figure~\ref{fig:zetabreakex}b.
Otherwise, $\alpha > \zeta$ and the second configuration that results in an $\alpha$ image is that $\alpha < \beta$, an example of which is shown in Figure~\ref{fig:alphabreakex}b.  
Here the prior slope is $0$ and the after slope is $1$ and the $\alpha$ image is distinct.
The third example is when $\alpha = \beta$ and is shown in Figure~\ref{fig:alphabreakex}c.
In this example, the after slope is $\infty$ because of the plan expiring at $\alpha$.
This $\alpha$ image is not distinct because there must also be a corresponding $\beta$ image.
There is also a minimization breakpoint $b'$, which (to torture our notation) is $\lim_{\epsilon \to 0}\epsilon$ time units `after' the $\alpha$ and $\beta$ images, which is to say though it has different slopes ($\infty$ prior and $0$ after in this case) it is not distinct from the $\alpha$ image either.
The fourth case is shown in Figure~\ref{fig:alphabreakex}d, where the prior slope is $\infty$, and the after slope is $1$, in this case the $\alpha$ image, $\alpha_2$ is not distinct from the $\alpha_1$ and $\beta_1$ images for the same reason as the previous example.
The remaining cases are: 5) when the after slope is $-\infty$, which is not possible (see Equation~\ref{eqn:nonnegwait}); and 6) when the after slope is $0$, which is not possible, as either the prior slope is $0$ and there is no breakpoint, or the prior slope is one of $-\infty$ and $\infty$ and the breakpoint is the only time that the corresponding ATF is not dominated so that zero duration interval can be safely removed from the compound ATF (though it may still be optimal for a later time interval, for example consider if $A[p_3]$ in Figure~\ref{fig:catfaddexample} were shifted later by one unit in arrival time).

\begin{figure}[t]
\centering
\begin{tikzpicture}[scale=0.75]
\draw (-0.3 cm, 0pt) -- (-0.3 cm,-0pt) node[anchor=north west] {(a)};
\draw[thick,->] (0,0) -- (6.5,0) node[midway, below] {t};
\filldraw[black] (2,1) circle (4pt) node [anchor=south east] {$b$};

\filldraw[red] (2,1) circle (2pt) node [anchor=north] {$\beta_1 \leq \alpha_1$};
\filldraw[blue] (3.5,2) circle (2pt) node [anchor=north west] {$\alpha_2$};
\draw [color =red] (0,1) -- (2,1);
\draw [red, dashed] (2,1) -- (2, 2);
\draw [color =blue,dashed] (0,2) -- (2,2);
\draw [color =blue] (2,2) -- (3.5,2);
\draw [color =blue,->] (3.5,2) -- (5,3.5);
\draw[thick,->] (0,0) -- (0,4) node[anchor=south west] {A[s,d](t)};
\end{tikzpicture}
\begin{tikzpicture}[scale=0.75]
\draw (-0.3 cm,0pt) -- (-0.3 cm,-0pt) node[anchor=north west] {(b)};
\draw[thick,->] (0,0) -- (6.5,0) node[midway, below] {t};
\draw[thick,->] (0,0) -- (0,4) node[anchor=south west] {A[s,d](t)};
\filldraw[black] (3,2) circle (4pt) node [anchor=south east] {$b$};
\filldraw[red] (0.5,1) circle (2pt) node [anchor=north] {$\zeta_1$};
\filldraw[red] (2,1) circle (2pt) node [anchor=north] {$\alpha_1$};
\filldraw[red] (3,2) circle (2pt) node [anchor=north] {$\beta_1$};
\filldraw[blue] (1,2) circle (2pt) node [anchor=north] {$\zeta_2$};
\filldraw[blue] (5,2) circle (2pt) node [anchor=north] {$\alpha_2$};
\filldraw[blue] (6,3) circle (2pt) node [anchor=north] {$\beta_2$};
\draw [color =red] (0.5,1) -- (2,1);
\draw [color =red] (2,1) -- (3,2);
\draw [color =blue,dashed] (1,2) -- (3,2);
\draw [color =blue] (3,2) -- (5,2);
\draw [color =blue] (5,2) -- (6,3);
\end{tikzpicture}

\begin{tikzpicture}[scale=0.75]
\draw (-0.3 cm,0pt) -- (-0.3 cm,-0pt) node[anchor=north west] {(c)};
\draw[thick,->] (0,0) -- (6.5,0) node[midway, below] {t};
\draw[thick,->] (0,0) -- (0,4) node[anchor=south west] {A[s,d](t)};
\filldraw[black] (3,2) circle (4pt) node [anchor=south east] {$b$};
\filldraw[red] (0.5,1) circle (2pt) node [anchor=north] {$\zeta_1$};
\filldraw[red] (2,1) circle (2pt) node [anchor=north] {$\alpha_1$};
\filldraw[red] (3,2) circle (2pt) node [anchor=north west] {$\beta_1 = \alpha_2$};
\filldraw[blue] (1,2) circle (2pt) node [anchor=north] {$\zeta_2$};
\filldraw[blue] (4,3) circle (2pt) node [anchor=north] {$\beta_2$};
\draw [color =red] (0.5,1) -- (2,1);
\draw [color =red] (2,1) -- (3,2);
\draw [color =blue,dashed] (1,2) -- (3,2);
\draw [color =blue] (3,2) -- (4,3);
\end{tikzpicture}
\begin{tikzpicture}[scale=0.75]
\draw (-0.3 cm,0pt) -- (-0.3 cm,-0pt) node[anchor=north west] {(d)};
\draw[thick,->] (0,0) -- (6.5,0) node[midway, below] {t};
\draw[thick,->] (0,0) -- (0,4) node[anchor=south west] {A[s,d](t)};
\filldraw[black] (3,2) circle (4pt) node [anchor=south east] {$b$};
\filldraw[red] (0.5,1) circle (2pt) node [anchor=north] {$\zeta_1$};
\filldraw[red] (2,1) circle (2pt) node [anchor=north] {$\alpha_1$};
\filldraw[red] (3,2) circle (2pt) node [anchor=north] {$\beta_1$};
\filldraw[blue] (1,3) circle (2pt) node [anchor=north] {$\zeta_2$};
\filldraw[blue] (5,3) circle (2pt) node [anchor=north] {$\alpha_2$};
\filldraw[blue] (6,4) circle (2pt) node [anchor=north] {$\beta_2$};
\draw [color =red] (0.5,1) -- (2,1);
\draw [color =red] (2,1) -- (3,2);
\draw [color =red,dashed] (3,2) -- (3,3);
\draw [color =blue,dashed] (1,3) -- (3,3);
\draw [color =blue] (3,3) -- (5,3);
\draw [color =blue] (5,3) -- (6,4);
\end{tikzpicture}
  \caption{\label{fig:betabreakex}
   Example of the four possible configurations for an $\beta$ image, $b$.}
\end{figure}
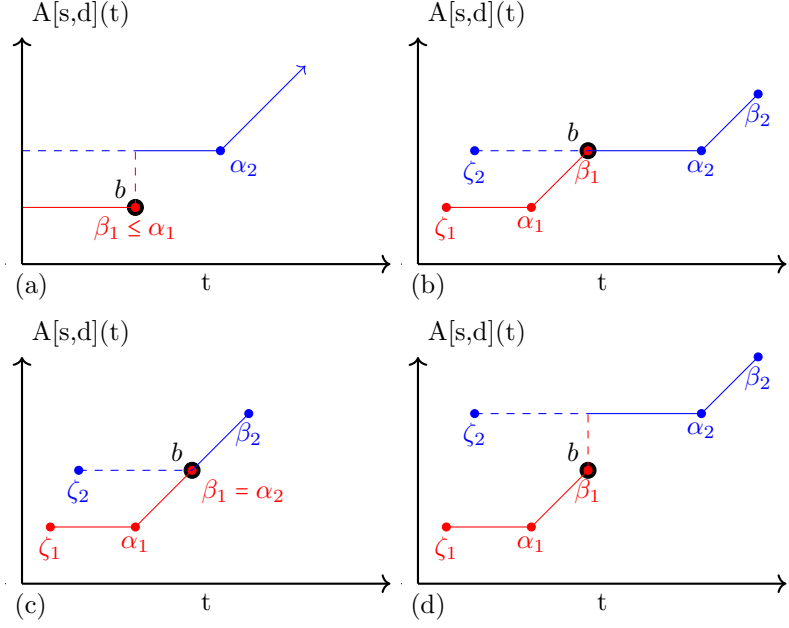

Figure~\ref{fig:betabreakex} shows the four possible configurations for a $\beta$ image.  
It is impossible for the prior or after slope of a $\beta$ image to be $-\infty$ (Equation~\ref{eqn:positiveduration}).  
Figure~\ref{fig:betabreakex}a shows an example with the prior slope $0$ and the after slope $\infty$.
An after slope of $0$ or $1$ would mean that the other ATF would dominate the one from which our $\beta$ image arose, so there must be such a discontinuity if the prior slope is $0$.
This breakpoint is distinct.
The remaining three combinations with prior slope $1$ are all possible, 
with after slope $0$ (Figure~\ref{fig:betabreakex}b, which is also a minimization breakpoint), 
slope $1$ (Figure~\ref{fig:betabreakex}c, which is also an image of $\alpha_2$ and a minimization breakpoint), and
slope $\infty$ (Figure~\ref{fig:betabreakex}d, which is distinct).

\begin{figure}[b]
\centering
\begin{tikzpicture}[scale=0.75]
\draw (-0.3 cm,0pt) -- (-0.3 cm,-0pt) node[anchor=north west] {(a)};
\draw[thick,->] (0,0) -- (6.5,0) node[midway, below] {t};
\draw[thick,->] (0,0) -- (0,4) node[anchor=south west] {A[s,d](t)};
\filldraw[black] (3,2) circle (4pt) node [anchor=south east] {$b$};
\filldraw[red] (0.5,1) circle (2pt) node [anchor=north] {$\zeta_1$};
\filldraw[red] (2,1) circle (2pt) node [anchor=north] {$\alpha_1$};
\filldraw[red] (4.5,3.5) circle (2pt) node [anchor=north] {$\beta_1$};
\filldraw[blue] (1,2) circle (2pt) node [anchor=north] {$\zeta_2$};
\filldraw[blue] (5,2) circle (2pt) node [anchor=north] {$\alpha_2$};
\filldraw[blue] (6,3) circle (2pt) node [anchor=north] {$\beta_2$};
\draw [color =red] (0.5,1) -- (2,1);
\draw [color =red] (2,1) -- (3,2);
\draw [color =red,dashed] (3,2) -- (4.5,3.5);
\draw [color =blue,dashed] (1,2) -- (3,2);
\draw [color =blue] (3,2) -- (5,2);
\draw [color =blue] (5,2) -- (6,3);
\end{tikzpicture}
\begin{tikzpicture}[scale=0.75]
\draw (-0.3 cm,0pt) -- (-0.3 cm,-0pt) node[anchor=north west] {(b)};

\draw[thick,->] (0,0) -- (6.5,0) node[midway, below] {t};
\draw[thick,->] (0,0) -- (0,4) node[anchor=south west] {A[s,d](t)};
\filldraw[black] (3,2) circle (4pt) node [anchor=south east] {$b$};
\filldraw[black] (3.75,2.75) circle (4pt) node [anchor=south east] {$b'$};

\filldraw[red] (0.5,1) circle (2pt) node [anchor=north] {$\zeta_1$};
\filldraw[red] (2,1) circle (2pt) node [anchor=north] {$\alpha_1$};
\filldraw[blue] (3,2) circle (2pt) node [anchor=north west] {$\alpha_2$};
\filldraw[blue] (1,2) circle (2pt) node [anchor=north] {$\zeta_2$};
\filldraw[blue] (4.5,3.5) circle (2pt) node [anchor=north] {$\beta_2$};
\draw [color =red] (0.5,1) -- (2,1);
\draw [color =red] (2,1) -- (3,2);
\draw [color = blue, very thick] (3,2) -- (4.5,3.5);

\draw [color =red,dashed] (3,2) -- (3.75,2.75);
\draw [color =blue,dashed] (1,2) -- (3,2);
\filldraw[red] (3.75,2.75) circle (2pt) node [anchor=north west] {$\beta_1$};
\end{tikzpicture}
\caption{\label{fig:minbreakex}
Example of the two possible configurations for a minimization breakpoint, $b$.}
\end{figure}
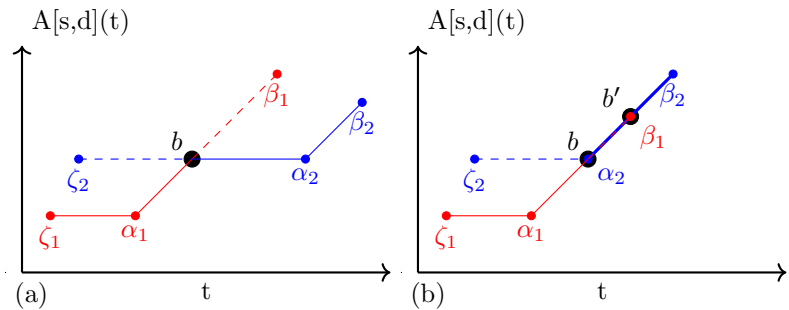

Figure~\ref{fig:minbreakex} shows the two possible configurations for minimization breakpoints.  
By definition, the ATF around a minimization breakpoint must be continuous (however, not the same slope) so the prior and after slope must be finite.
Also, by definition, there must be a different ATF optimal prior to the minimization breakpoint (henceforth the prior ATF) and after (the after ATF).
If the prior slope was $0$, then the after ATF would dominate the prior ATF, so the prior slope must be $1$.
Figure~\ref{fig:minbreakex}a shows the example where the after slope is $0$0 and Figure~\ref{fig:minbreakex}b shows the example where the after slope is $1$.
When the after slope is $1$, as in this example, there might be many possible times to assign as the minimization breakpoint, depending on when the algorithm decides to switch from the prior ATF to the after ATF in the interval $[b, b']$ where they are both optimal.
What is important is that there should only be one such breakpoint, there is no reason for the compound AFT to switch back and forth between the prior and after ATF.
Given that, notice that because the prior slope for a minimization breakpoint must be $1$, any (non redundant) pair of minimization breakpoints must have an $\alpha$ breakpoint between them.

In this subsection we showed examples of every possible way for each type of breakpoint to manifest. 
In the following sections, we first show that the number of each type of image are constrained to be linear in the size of the graph, then we show that the number of $\alpha$ images constrains the number of minimization breakpoints, so the complexity of the compound ATF is also linear in the size of the graph.

\subsubsection{Constraining the Number of Primitive Images}\label{sec:aspcomplex}

In this subsection we constrain the number of each type of image breakpoint in the solution ATF of an any-start-time planning problem. 
The solution ATF, which we instantiate as $A[s, d]$, is a compound ATF for the set of optimal plans from start state $s$ to goal state $d$ for all possible start times.  
First, we constrain the number of $\zeta$ images, which is a simple task because all plans originate in the same state.
\begin{lemma}\label{lmm:primitivezetabreak}
There is at most one $\zeta$ image in $A[s,d]$.
\end{lemma}
\begin{proof}
All plans from $s$ to $d$ have a $\zeta$ image corresponding to the first edge on the plan, by Equation~\ref{eqn:pathzeta}.
The source of those first edges is always $s$, so by Equation~\ref{eqn:zetacompute} for all plans the $\zeta$ image occurs at $\zeta(s)$.
Thus all plans from $s$ to $d$ share the same $\zeta$ image, so it is impossible to have more than one in $A[s,d]$. 
\end{proof}

In the next several lemmas, we constrain the number of primitive images in $A[s,d]$ by showing that each edge breakpoint can at most lead to a single image in $A[s,d]$.

We begin by showing that compound ATFs in any-start-time planning are monotone non-decreasing functions.  
\begin{lemma}\label{lmm:monotone}
$A[s,d]$ is monotone non-decreasing.
\end{lemma}
\begin{proof}
If $A[s,d]$ is not monotone non-decreasing, then there must exist departure times $t,t'$ such that $t<t'$ and $A[s,d](t) > A[s,d](t')$.
However, $s$ is always safe when $A[s,d]$ is finite, so it is always possible to wait at $s$ from $t$ to $t'$, so $A[s,d](t) \leq A[s,d](t')$.
It is not possible for $A[s, d]$ to become infinite and then finite again at a later time, otherwise the correct $A[s, d]$ would have remained finite by waiting for the later departure time.
\end{proof}
This monotonicity is illustrated in all example ATFs (e.g., Figures~\ref{fig:zetabreakex}-\ref{fig:minbreakex}), 
The next two properties follow directly from the use of the SIPP state space and the same dominance relation used by the SIPP algorithm.  
These are closely related to the `first-in first-out' assumption common in the TDSP.
Next, we show that the solution has the ATF version of optimal substructure, and plans with loops are dominated.
\begin{lemma}\label{lmm:substructure}
Plans minimizing $A[s,d]$ exhibit optimal substructure:  If plan $p$ minimizes $A[s, d]$ at time $t$, then $\forall i \in [1, |p|]: p_{i,|p|}$ minimizes $A[p[i-1], d](\tau)$ where $\tau = A[p_{1,i-1}](t)$.
\end{lemma}
\begin{proof}
Suppose there existed a plan $\rho$ such that $\rho[0] = p[i]$, $\rho[|\rho|] = p[|p|]$, and $A[\rho](\tau) < A[p_{i,|p|}](\tau)$. 
Therefore, the plan
$p'=(p_{1,i-1},\rho )$ has 
$A[p'](\tau) < A[p](\tau)$, contradicting
the optimality of $p$.
\end{proof}

\begin{lemma}\label{lmm:simple}
Plans containing loops are dominated: For any given $t,s,d$ there exists a simple
plan $p$ such that $A[s,d](t)= A[p](t)$. 
\end{lemma}
\begin{proof}
Consider any plan with a loop, edge costs are positive so the loop can be replaced by an equally good plan that waits rather than traversing the loop.
\end{proof}

\noindent The proof of Lemma~\ref{lmm:simple}
leads immediately to the following corollary:
\begin{corollary}\label{col:simple}
For every plan $p$ between $s$ and $d$ in which the same edge $e$ appears at both $p_i$ and $p_j$, there exists a simple plan $p'$
between $s$ and $d$ such that for all time $t$:  $A[p'](t) \leq A[p](t)$ and $e$ occurs exactly once in $p'$.
\end{corollary}
\begin{proof}
Without loss of generality, assume $i < j$. 
$p'$ is the plan created by concatenating  $p_{1,i}$ and $p_{j+1, |p|}$.
We can remove the loop $p_{i+1,j}$, as $p_i$ is the same edge as $p_j$ and edges have non-negative duration.  
\end{proof}

We are now ready to constrain the number of $\alpha$ images in $A[s.d]$. 
\begin{lemma}\label{lmm:primitivealphabreak}
For every edge $e \in E$ and every source and destination state $s, d \in S$ in an SIPP state space, there exists no more than one $\alpha$ image of $\alpha(e)$ in $A[s,d]$.
\end{lemma}
\begin{proof}
Assume in contradiction that $b < b'$ are the times of two distinct $\alpha$ images in $A[s,d]$ of $\alpha(e)$.
By definition of $A[s,d]$ and Lemma~\ref{lmm:simple}, there  exist simple plans $p,p'$, both containing $e$, minimizing $A[s,d]$ at $b,b'$ respectively.
Thus $A[s,d](b) = A[p](b)$ and $A[s,d](b') = A[p'](b')$.
Both $p$ and $p'$ are simple plans, so $e=p_i=p'_{i'}$ for unique $i$ and $i'$.

Now, we will show that $A[s,d](b) = A[s,d](b')$.
Now consider $A[p](b)$ in two parts, the part up to and including the constraining edge $e$ ($A[p_{1,i}](b)$), and the part after the constraining edge ($A[p_{i+1,|p|}]$), where $A[p](b) = (A[p_{i+1,|p|}] \circ A[p_{1,i}])(b)$.
Because $e$ is the constraining edge,  $A[p_{1,i}](b) = \alpha(e) + \Delta(e)$ (Definition~\ref{def:alphaconstraint}).
The same holds for $p'$, $A[p'_{1,i}](b) = \alpha(e) + \Delta(e)$.

Because both $p$ and $p'$ are constrained by $e$, they have the same arrival time at $u(e)$, the end state of $e$.
Likewise, they thus have the same plan duration after $e$, otherwise the one using a longer route would not be optimal because using the shorter route used by the other would make it arrive earlier (Lemmma~\ref{lmm:substructure}), so $A[s,d](b) = A[p](b) = A[p'](b') = A[s,d](b')$.

However $b$ is an alpha-breakpoint, so $A[s,d]$ must rise after $b$ (see Figure~\ref{fig:alphabreakex}a-d).
Consider one such later departure time, $b+\epsilon$, such that $b < b+\epsilon \leq b'$, and $A[s,d](b) < A[s, d](b + \epsilon)$.
Because $A[s,d]$ is monotonic non-decreasing, $A[s,d](b+\epsilon) \leq A[s,d](b')$, so $A[s,d](b) < A[s,d](b')$, a contradiction.
\end{proof}
Essentially, once the edge has constrained the plan's ATF, the remaining post-edge portion of the plan does not change the breakpoint, so any duplicate breakpoint would have to be at the same time.

\begin{lemma}\label{lmm:primitivebetabreak}
For every edge $e$ and destination vertex $d$, there exists no more than one ($\beta$) image of $\beta_e$ in $A[s,d]$.
\end{lemma}
\begin{proof}
    The same proof as lemma~\ref{lmm:primitivealphabreak} applies.
    Because $e$ is constraining the image in $A[s,d]$, the distinction that $\alpha$ updates through $\max$ and $\beta$ through $\min$ does not matter.
    The only difference is that while an $\alpha$ image is preceded by a flat segment, and followed by a rising segment in $A[s,d]$ (Figure~\ref{fig:betabreakex}a), a $\beta$ image is proceeded by a rising segment, and followed by either a discontinuous increase in $A[s,d]$ (Figure~\ref{fig:betabreakex}d) which means  $b+\epsilon$ with $A[s,d](b) < A[s, d](b + \epsilon)$ exists.
    Otherwise, it is also a minimization breakpoint and so is followed by a flat segment in $A[s,d]$, which means there must be a later rising segment before there can be another $\beta$ image, and again $b+\epsilon$ exists.
\end{proof}

Therefore, counting all the primitive images we get:
\begin{theorem}\label{thm:primitivebreakpoints}
The number of primitive breakpoints in $A[s, d]$ is linear in the number of edges in $G$.
\end{theorem}
\begin{proof}
By Lemmas~\ref{lmm:primitivezetabreak}, \ref{lmm:primitivealphabreak}, and \ref{lmm:primitivebetabreak} the number of primitive breakpoints in $A[s, d]$ is bounded by the number of edge breakpoints in $G$, which is at most three
per edge.
\end{proof}

\subsubsection{Constraining the Number of Minimization Breakpoints}
Having shown that the number of images in linear in the size of $G$, we now show that the number of minimization breakpoints is linear in the number of images in $A[s,d]$.
Pictorially, this result can be seen clearly by examining both possible configurations for a distinct minimization breakpoint, shown in Figure~\ref{fig:minbreakex}a, which are always proceeded by $\alpha$ image.
All others are co-temporal with an image: Figure~\ref{fig:minbreakex}b's minimization breakpoints at $b$ and $b'$ are at the same time as $\alpha_2$ and $\beta_1$ respectively, Figure~\ref{fig:alphabreakex}c and d, and Figure~\ref{fig:betabreakex}a-d are all co-temporal with $\beta_1$ (and often others).

\begin{lemma}\label{lmm:minalpha}
Let $m, b$ be two consecutive distinct breakpoints of $A[s,d]$ such that $m < b$.  If $m$ is only a minimization breakpoint (not also a $\zeta, \alpha$, or $\beta$ image), then $b$ must be an $\alpha$ image. 
\end{lemma}
\begin{proof}
The slopes available to segments of $A[s,d]$ are $0,1,-\infty,\infty$. 
We will examine which pairs of these slopes can occur before and after a distinct minimization breakpoint like $m$.
First the trivial cases.
Slopes of $-\infty$ occur only at the time immediately before the $\zeta$ breakpoint, so the slope before $m$ cannot be $-\infty$ as $m$ would also be a $\zeta$ image (see Figure~\ref{fig:zetabreakex}).  
The slope can never be $-\infty$ to the right of $m$, as $A[s,d]$ is monotonic non-decreasing (Lemma~\ref{lmm:monotone}).
The slope must change, or it can not be a distinct breakpoint: slope $0$ to slope $0$, the earlier plan is dominated by the later and slope $1$ to slope $1$ is either a $\alpha$ or $\beta$ image (see Figure~\ref{fig:minbreakex}b) or both (see Figure~\ref{fig:betabreakex}c).

Likewise, slopes of $\infty$ only occur at $\beta$ images when the corresponding plan expires, so if the minimization creating $m$ involves a slope of $\infty$, $m$ must also be a $\beta$ image (see Figure~\ref{fig:betabreakex}).
If the slopes about $m$ go from $0$ to $1$, $m$ must also be an $\alpha$ image (see Figure~\ref{fig:alphabreakex}b and d), as otherwise the slope $0$ segment would extend past $m$ to some later $\alpha$ image and remain the minimizer (e.g., ac vs bc in Figure~\ref{fig:dominanceexample}), meaning that $m$ would not be a breakpoint. 
Therefore, the only possibility is for the slope before $m$ to be $1$ and after $m$ to be $0$ (see Figure~\ref{fig:minbreakex}).  
By Lemma~\ref{lmm:monotone}, $A[s,d]$ is monotonic non-decreasing, so $A[s,d]$ cannot have a distinct breakpoint until the $\alpha$ image that ends the slope $0$ segment.
As can be seen from Figures~\ref{fig:zetabreakex}-\ref{fig:minbreakex} no breakpoint that is not co-temporal with an $\alpha$ image terminates a slope $0$ segment. 
\end{proof}

\begin{theorem}\label{thm:nbreak}
The space complexity of Any-start-time Planning is linear in the number of edges in $G$. 
\end{theorem}
\begin{proof}
All breakpoints in $A[s,d]$ are either primitive images or minimization breakpoints.
Lemma~\ref{lmm:minalpha} shows that there must be an $\alpha$ image after each minimization breakpoint, limiting the number of minimization breakpoints to at most the number of $\alpha$ breakpoints, which as a subset of primitive breakpoints is linear in the number of edges by Theorem~\ref{thm:primitivebreakpoints}.
\end{proof}
This result follows directly from using the SIPP state space as the foundation of Any-start-time Planning, which resulted in the limited set of slopes we enumerated in Lemma~\ref{lmm:minalpha}.
In the following subsection, we follow up by showing that the theoretical time complexity of the RePEAT algorithm is quasi-quadratic in the size of the problem.
After that, we experimentally evaluate the runtime of RePEAT on grid path planning among moving obstacles.

\subsection{Time Complexity of RePEAT}  \label{sec:time}

The time complexity of RePEAT is determined by how many times the inner Augmented SIPP routine is called.  
We will show that RePEAT runs only a linear number of calls to Augmented SIPP.
In our setting, as a best-first search like Dijkstra's algorithm, the runtime of \aSIPPab\ is $O(|E|log(|V|))$ where $|E|$ is the number of edge ATFs in our problem, and $|V|$ is the number of SIPP states.

\begin{theorem}
    The runtime of RePEAT is $O((|E| + |V|)|E|log(|V|))$.
\end{theorem}
\begin{proof}
    Each time the reference time is updated, we are either including a plan ATF (with at least one associated primitive image) in the compound ATF or eliminating a primitive image from being in the compound ATF. 
    In the case where Augmented SIPP returns a plan that is not dominated by the existing compound ATF, then after the reference time update, that same plan will not be returned again, unless that round of search also expanded a node corresponding to the plan ATF that set the reference time, which eliminates the associated $\alpha$ image from future searches.
    The number of such breakpoints is $|E| + |V|$, so at most that many rounds of Augmented SIPP will be required, leading to an overall worst-case runtime of $O((|E| + |V|)|E|log(|V|))$ which is quasi-quadratic in the number of safe intervals in $G$ ($|E| + |V|$).   
\end{proof}

\section{Empirical Evaluation}\label{sec:empiricaleval}
It is natural to wonder how useful any of these techniques are in practice.
To evaluate their behavior concretely in a popular dynamic planning problem, we ran a series of experiments on path planning in 2D grids among time-dependent obstacles (the domain in which SIPP was originally demonstrated).
We address questions in three different categories: 
1) The practicality of Augmented SIPP on its own: What is the overhead of Augmented SIPP compared to original SIPP?
2) The adequacy of replanning: How extreme must problems get before Replanning SIPP becomes an inadequate solution?
3) The practicality of RePEAT: How complex are any-start-time plans in practice?  How long do they take to compute?  Can they be queried quickly?

\subsection{Experimental Set-Up}

We begin by describing our experimental set-up, our C++ implementation of RePEAT and the other algorithms, the hardware used to evaluate the performance, and how we generated the problem instances used in this evaluation.  

\subsubsection{Implementation}

We implemented the algorithms used in our experiments in a modern C++ code base.
For these experiments, we used GCC 15.2.1 and Boost 1.88.0 with -O3 optimization.
We use the Boost n\_ary\_heap data structure with an arity of $4$ to implement our open lists (we found this to outperform std::priority\_queue and a n\_ary\_heap with arity 2).
Where possible, the implementations of our algorithms share data structures and code.

\subsubsection{Problem Instance Generation}

\begin{figure}
\centering
\begin{subfigure}{.34\textwidth}
    \centering
    \caption{den520d (257x256)}
    \label{fig:den520d}
    \includegraphics[width=\linewidth]{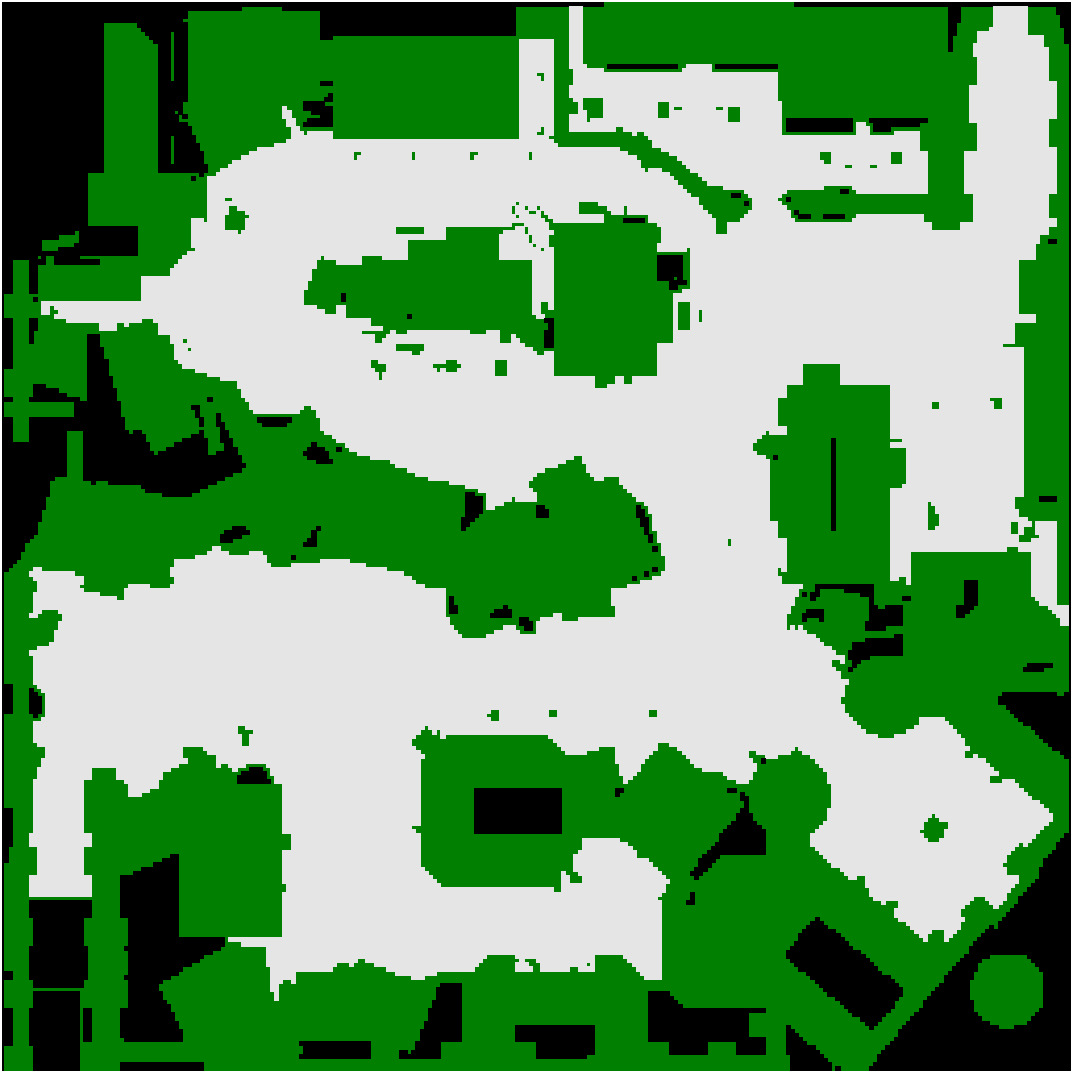}
\end{subfigure}
\begin{subfigure}{.34\textwidth}
    \centering
    \caption{den520d in detail}
    \label{fig:den520dzoom}
    \includegraphics[width=\linewidth]{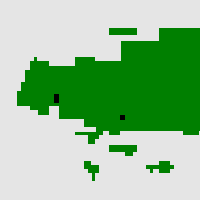}
\end{subfigure}\\
\begin{subfigure}{.34\textwidth}
    \centering
    \caption{random512-10-0 (512x512)}
    \label{fig:random}
    \includegraphics[width=\linewidth]{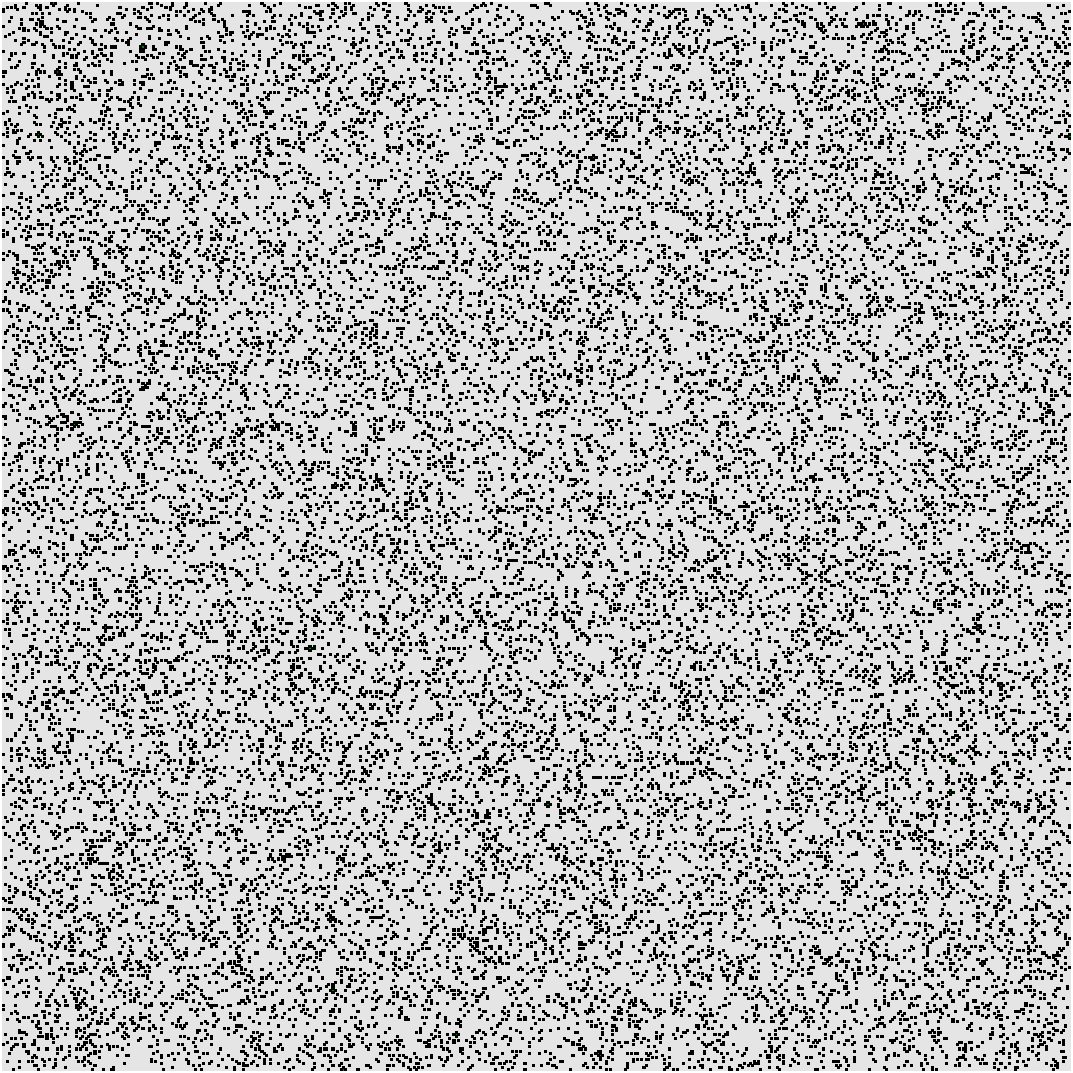}
\end{subfigure}
\begin{subfigure}{.34\textwidth}
    \centering
    \caption{random512-10-0 in detail}
    \label{fig:randomzoom}
    \includegraphics[width=\linewidth]{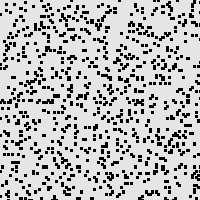}
\end{subfigure}\\
\begin{subfigure}{.34\textwidth}
    \centering
    \caption{64room\_006 (512x512)}
    \label{fig:rooms}
    \includegraphics[width=\linewidth]{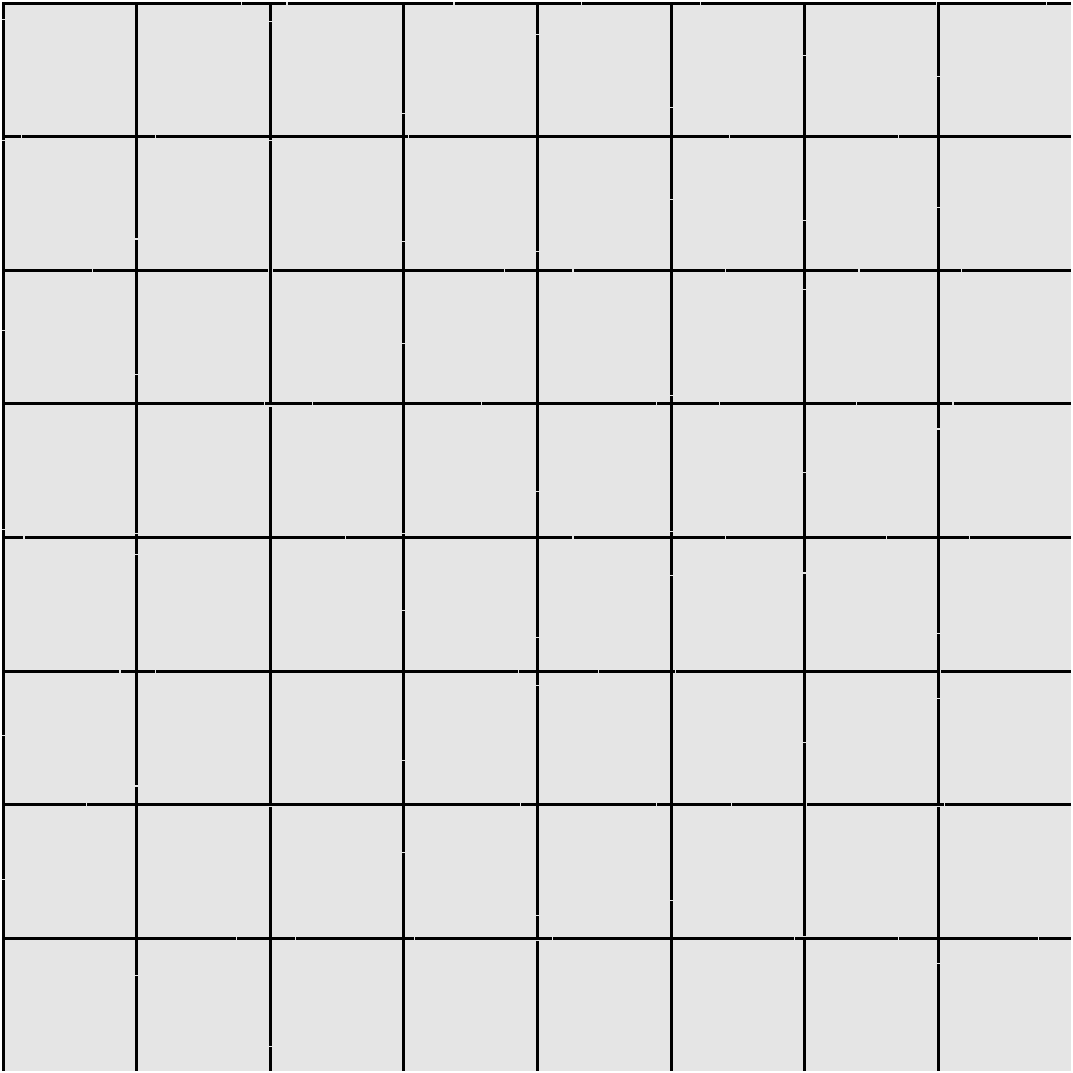}
\end{subfigure}
\begin{subfigure}{.34\textwidth}
    \centering
    \caption{64room\_006 in detail}
    \label{fig:roomszoom}
    \includegraphics[width=\linewidth]{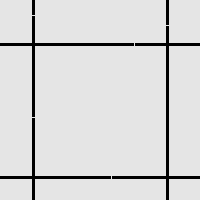}
\end{subfigure}\\
\caption{Benchmark grid maps.}
\end{figure}

Our test problems are adapted from maps and scenarios from the Moving AI Lab 2D Pathfinding Benchmarks~\cite{benchmarkmaps}.  
We selected three maps --- den520d (Figure~\ref{fig:den520d}, trees are shown as green and considered impassible), random (Figure~\ref{fig:random}), and rooms (Figure~\ref{fig:rooms}) --- to provide three contrasting static environments.
Rooms and den520d were selected, in part, to correspond with prior SIPP papers, which have used these or similar maps for their experiments.
Rooms provides a static environment with regularly shaped and spaced local minima to increase the challenge of the search, while den520d consists of larger irregular rooms connected by long hallways.
The random map contrasts with the other two, with small local minima but many near-optimal plans, providing an intermediate difficulty between the easier den520d and the more challenging rooms.  
The den520d map is 256x257, while the other two maps are 512x512.

\begin{figure}
    \centering
        \begin{subfigure}{.32\textwidth}
        \centering
        \caption{den520d [5, 50]}
        \includegraphics[width=\linewidth]{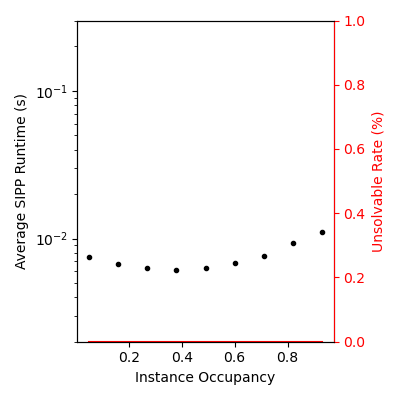}
        \end{subfigure}
        \begin{subfigure}{.32\textwidth}
        \centering
        \caption{den520d [10, 100]}
        \includegraphics[width=\linewidth]{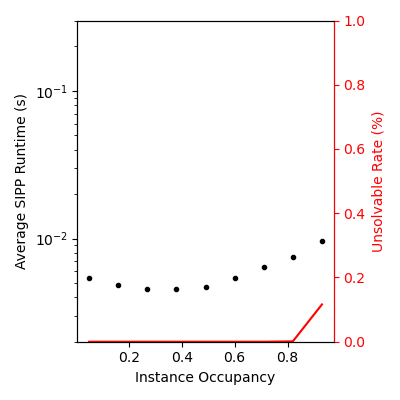}
        \end{subfigure}
        \begin{subfigure}{.32\textwidth}
        \centering
        \caption{den520d [50, 500]}
        \includegraphics[width=\linewidth]{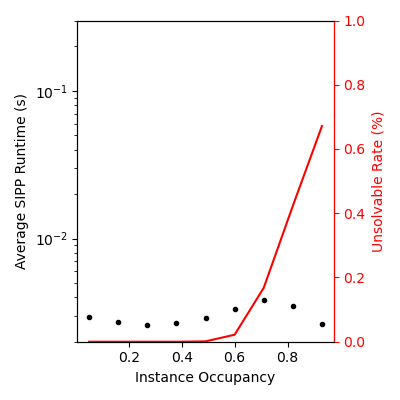}
        \end{subfigure}\\
        \begin{subfigure}{.32\textwidth}
        \centering
        \caption{random512 [5, 50]}
        \includegraphics[width=\linewidth]{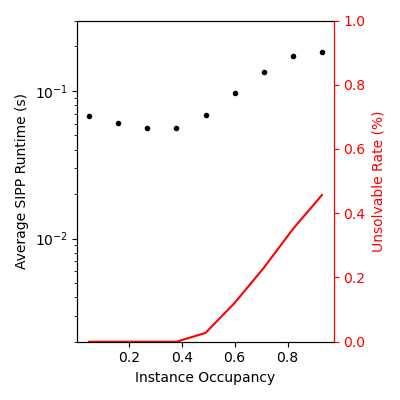}
        \end{subfigure}
        \begin{subfigure}{.32\textwidth}
        \centering
        \caption{random512 [10, 100]}
        \includegraphics[width=\linewidth]{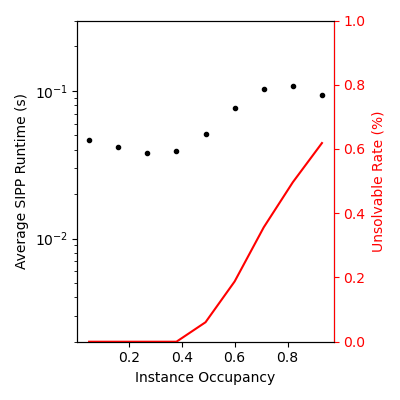}
        \end{subfigure}
        \begin{subfigure}{.32\textwidth}
        \centering
        \caption{random512 [50, 500]}
        \label{fig:random50500}
        \includegraphics[width=\linewidth]{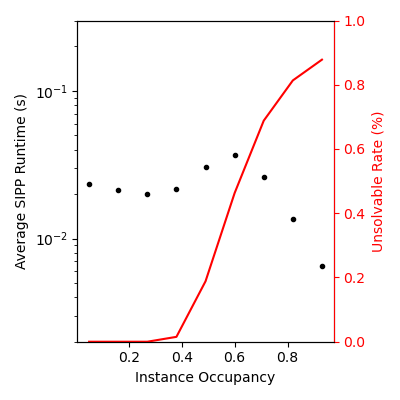}
        \end{subfigure}\\
        \begin{subfigure}{.32\textwidth}
        \centering
        \caption{64room [5, 50]}
        \includegraphics[width=\linewidth]{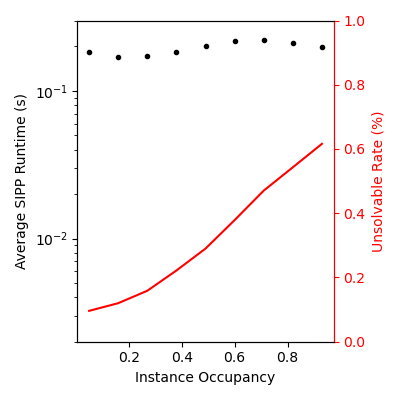}
        \end{subfigure}
        \begin{subfigure}{.32\textwidth}
        \centering
        \caption{64room [10, 100]}
        \includegraphics[width=\linewidth]{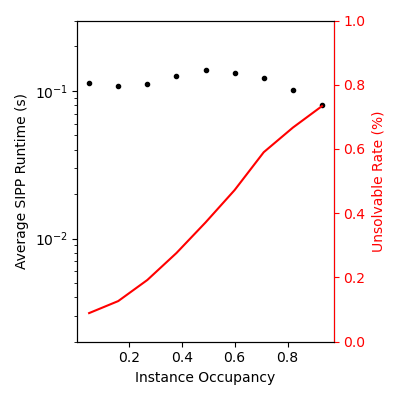}
        \end{subfigure}
        \begin{subfigure}{.32\textwidth}
        \centering
        \caption{64room [50, 500]}
        \includegraphics[width=\linewidth]{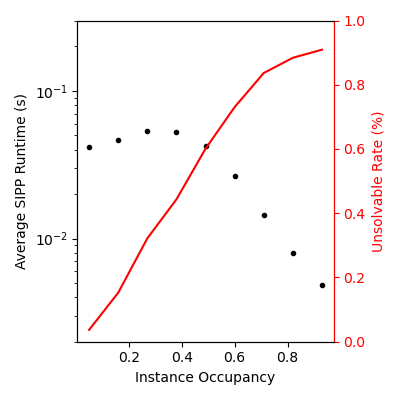}
        \end{subfigure}\\
        \caption{SIPP geometric mean runtime and unsolvable rate, for 2D grid pathfinding instances with random obstacles with (minDuration,~maxDuration)~$\in\{(5, 50),(10, 100),(50, 500)\}$, and occupancy $\in [0.05,0.93]$.}
        \label{fig:occupancy}
\end{figure}

For each problem instance, we populated the static map with random obstacles that make their locations unsafe for a random duration.
The instance generation is controlled by three parameters: minDuration, maxDuration, and occupancy.
Occupancy is the odds that a location will be unsafe at a particular time.
The sequence of states for each location are generated by appending a safe interval with duration drawn from the continuous uniform distribution from minDuration to maxDuration, then an unsafe interval with duration drawn from the continuous uniform distribution between 0 and  occupancy times the average safe duration ($\mathit{occupancy}*\frac{\mathit{maxDuration} - \mathit{minDuration}}{2}$).
So an occupancy of 0 yields an instance with no time-dependent obstacles, 1 yields an instance where every location is unsafe half of the time on average.

To simplify the generation of viable random instances, we set the start and goal to be always safe.
After 2000 time units, all states become unsafe to provide a finite horizon for the random safe interval generation.
We have taken care to use a pseudo-random number generator that, for the same random seed, generates the same sequence of random instances even when run on different machines or compiled with different compilers or Boost library versions, so the same start and goal location on the same static map with the same generation parameters should produce the same SIPP problem instance no matter what computer it is run on. 
For our experiments, we generate one instance for each scenario provided along with the static problem instances.
For den520d, there are 870 scenarios, for random512, there are 1780, and for 64room, there are 2300.

In general, the goal of our empirical tests is to characterize the performance of our algorithms on non-trivial problems.
Towards this, the objective of our first experiment is to determine what settings for the three parameters (minDuration, maxDuration, and occupancy) produce computationally challenging instances that still mostly have solutions (as any-start-time planning on an unsolvable instance finds the same `no solution' result as SIPP).
Towards this, we run SIPP on sets of instances generated with (minDuration,~maxDuration)~$\in\{(5, 50),(10, 100),(50, 500)\}$ and occupancy varying uniformly with $9$ values from $0.05$ to $0.93$.

Figure~\ref{fig:occupancy} shows the geometric mean runtime and the percentage of instances where Augmented SIPP proved that no feasible solution exists for the three maps and three sets of random interval durations.
In several (especially Figure~\ref{fig:random50500}), we see a similar `phase transition' that has been observed in related problems, such as the TSP with time windows\cite{rifki2025phase}.
The transition is from low occupancy, where there are few obstacles and many feasible solutions, to high occupancy, where there are few, if any, feasible solutions.
In both extremes, the search is relatively easy, with few obstacles it amounts to static grid pathplanning, and with many, there are so few viable successors that the search can quickly find the only solution if it exists, or eliminate all possible plans if it does not.
It is in the intermediate area where the search is the most challenging, which results in the peaks in the runtimes shown in Figure~\ref{fig:occupancy}.

The den520d instances are roughly an order of magnitude easier than the random512 and 64room instances.  
Examining what proportion of the instances are solvable, we find that the long random intervals result in many unsolvable instances with larger occupancies ($[50, 500]$), while $[5, 50]$ and $[10, 100]$ are more similar to each other.
Overall, den520d results in fewer unsolvable instances, with more for random512, and even more for 64room.
With this in mind, we select the settings: minDuration $= 10$, maxDuration $= 100$, and occupancy $= 0.35$ because these settings avoid all of the unsolvable instances for den520d and random512, along with most of the unsolvable instances for 64room.  
The average runtimes are sufficiently long to accurately measure, and the occupancy is high enough that we expect the agent to encounter temporal obstacles, leading different plans to be optimal for different start times, and any-start-time planning to be of use. 

\subsubsection{Infrastructure}
\begin{figure}[t]
    \centering
    \begin{subfigure}{0.49\textwidth}
    \includegraphics[width=\linewidth]{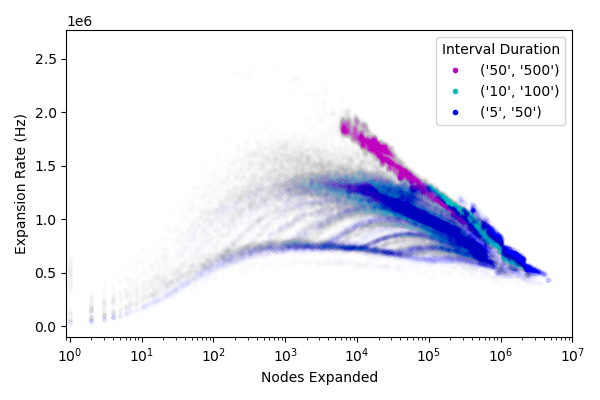}
    \end{subfigure}
    \begin{subfigure}{0.49\textwidth}
    \includegraphics[width=\linewidth]{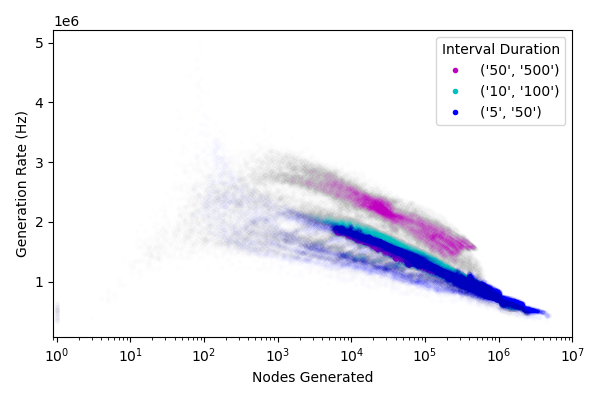}
    \end{subfigure}
    \caption{The distribution of SIPP expansion rates on randomly generated SIPP instances.}
    \label{fig:sippexpansionrate}
\end{figure}

The experiments were performed on a cluster of 15 identical desktop machines with Intel 4.3GHz i3-12100 quad-core CPUs and 64 GB RAM, with one experiment at a time run on each machine (and no other programs running on the same machine).
All timings measure only the search itself, excluding generation of safe intervals from random obstacles and I/O, which are shared by all the tested algorithms.  
We measured the expansion rate of our SIPP and Augmented SIPP implementations at about 1 MHz on non-trivial problem instances.
Specifically,
Figure~\ref{fig:sippexpansionrate} shows the expansion rate and generation rate of SIPP on each instance tested to set the parameters for the random instance generation (the same set of instances as in Figure~\ref{fig:occupancy}).
Over all instances requiring more than $10,000$ expansions, we measure the geometric mean expansion rate at $970,000$ Hz with a $95\%$ confidence interval of $968,400 - 971,800$ Hz.
Especially on smaller problems, larger random intervals can lead to a lower branching factor and a faster expansion rate.

\subsection{The Practicality of Augmented SIPP}\label{sec:asippoverhead}
In this section, we address our first research question: how practical is Augmented SIPP?
We answer this question by comparing the runtime of SIPP and Augmented SIPP on identical problem instances.  
If there is little to no overhead, then Augmented SIPP is at least as practical as SIPP.
We calculate the overhead of Augmented SIPP on a problem instance as the percent difference of the runtime of Augmented SIPP ($r_{ASIPP}$) and SIPP ($r_{SIPP}$) divided by their average runtime:
\begin{equation}\label{eqn:overhead}
    \mathit{Overhead} = \frac{r_{ASIPP} - r_{SIPP}}{r_{SIPP}}
\end{equation}

\begin{figure}[tp!]
    \centering
        \begin{subfigure}{.32\textwidth}
        \centering
        \includegraphics[width=\linewidth]{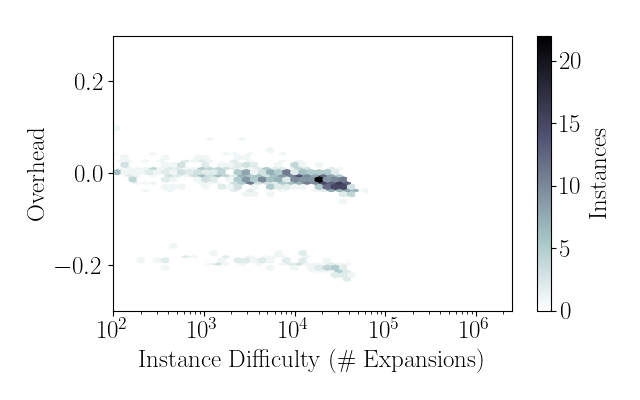}
        \caption{den520d}
        \label{fig:den520doverheaddif}
        \end{subfigure}
        \begin{subfigure}{.32\textwidth}
            \centering
            \includegraphics[width=\linewidth]{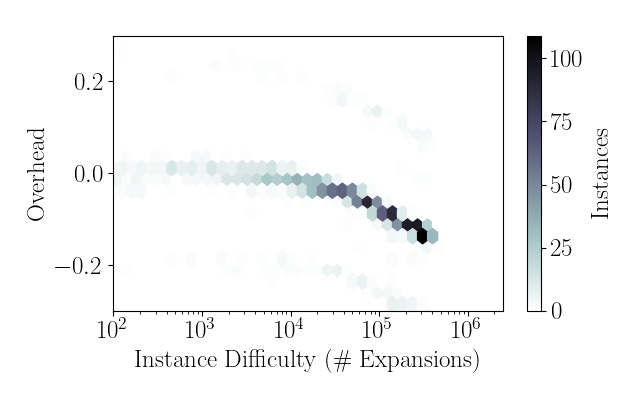}
            \caption{random512-10-0}
            \label{fig:randomoverheaddif}
        \end{subfigure}
        \begin{subfigure}{.32\textwidth}
            \centering
            \includegraphics[width=\linewidth]{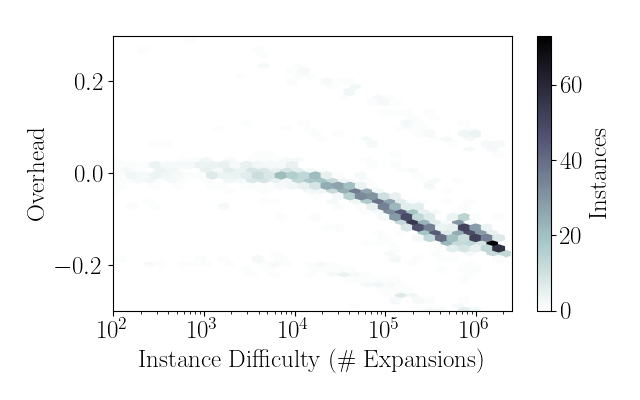}
            \caption{64room\_006 }
            \label{fig:roomoverheaddif}
        \end{subfigure}
        \caption{Augmented SIPP Overhead}
        \label{fig:overheadvdiff}
\end{figure}

Figure~\ref{fig:overheadvdiff} shows the distributions of overhead for Augmented SIPP on our random instances.  
We observe a median overhead of $-2\%, -8\%,$ and $-11\%$ respectively for den520d, random512, and 64room.
The central $95\%$ of the data ($2.5$th - $97.5$th percentile) spans $[-24.1\%, 0.9\%], [-30.0\%, 8.1\%]$, and $[-30.6\%, 7.9\%]$ respectively. 
Unexpectedly, the overhead is typically negative.

Our implementations of SIPP and Augmented SIPP use the same basic data structures for open and closed, and differ only as it relates to the calculation and representation of $g$ during successor generation.
They are both run on identical instances, and we have confirmed that both expand and generate the same number of nodes, confirming that our implementations are performing the same search (see Observation~\ref{obs:asippissipp}).
One difference that could be a possible source of this speedup is that Augmented SIPP only fetches a single object from memory to generate a successor, the edge ATF.
In contrast, SIPP may need to fetch the safe intervals for the source state of the edge, the destination state of the edge, and the edge itself.

Given these results, in situations where the path ATF can be useful, such as knowing for how long the SIPP plan is optimal or valid, Augmented SIPP appears to be a very practical algorithm in its own right. 

We also measure the overhead of \aSIPPab{} in the same manner as Augmented SIPP.
Compared to SIPP, \aSIPPab{} should be expected to need to expand more nodes, because it is able to do less pruning (dominance only vs SIPP pruning on $\leqfg$), this may be especially detrimental in a 2D grid setting where ties are common.
The runtime overhead of \aSIPPab{} over Augmented SIPP is shown in Figure~\ref{fig:overheadvdiffab}a-c.
The runtime can be over 100x that of Augmented SIPP with the usual node ordering and pruning.

To examine the cause of this, Figure~\ref{fig:overheadvdiffab}d-f shows the number of expansions done by \aSIPPab{} as a function of the number of expansions done by SIPP. 
For each map, we perform a log-linear least squares to fit a power function of the form $y = ax^b$ to the problem instances with more than $10^4$ SIPP expansions.
Qualitatively, all three maps appear to be well fit by power functions of a relatively low power. 
This indicates that our \aSIPPab{} searches are generally doing more expansions, but not exponentially more expansions than the corresponding SIPP search. 

\begin{figure}[t]
    \centering
        \begin{subfigure}{.32\textwidth}
        \centering
        \includegraphics[width=\linewidth]{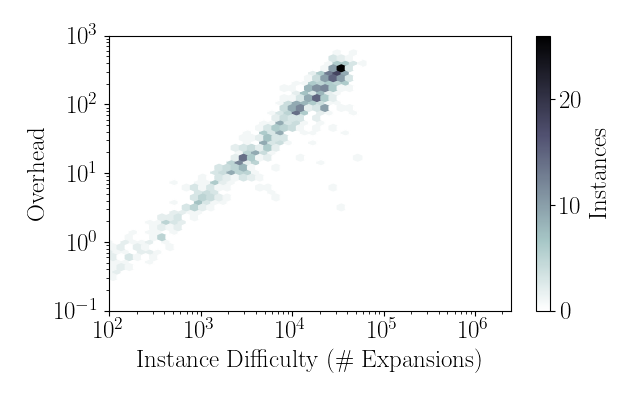}
        \caption{den520d}
        \label{fig:den520abdoverheaddif}
        \end{subfigure}
        \begin{subfigure}{.32\textwidth}
            \centering
            \includegraphics[width=\linewidth]{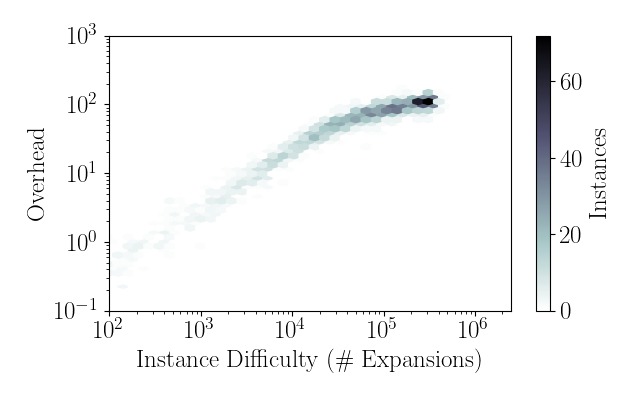}
            \caption{random512-10-0}
            \label{fig:randomaboverheaddif}
        \end{subfigure}
        \begin{subfigure}{.32\textwidth}
            \centering
            \includegraphics[width=\linewidth]{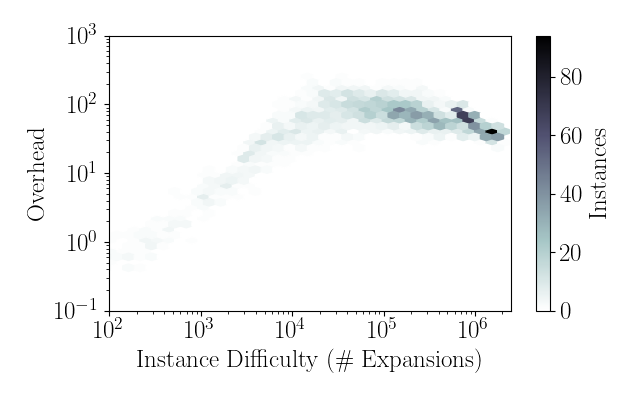}
            \caption{64room\_006 }
            \label{fig:roomaboverheaddif}
        \end{subfigure}
        \begin{subfigure}{.32\textwidth}
        \centering
        \includegraphics[width=\linewidth]{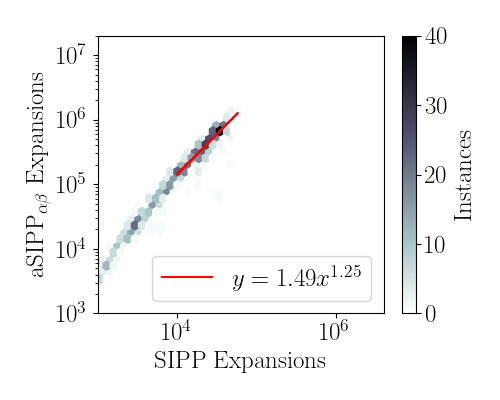}
        \caption{den520d}
        \label{fig:den520dexpdif}
        \end{subfigure}
        \begin{subfigure}{.32\textwidth}
            \centering
            \includegraphics[width=\linewidth]{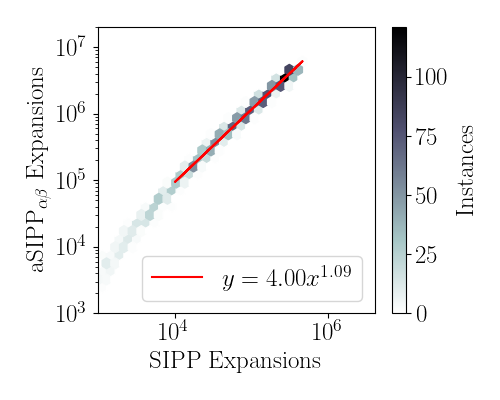}
            \caption{random512-10-0}
            \label{fig:randomexpdif}
        \end{subfigure}
        \begin{subfigure}{.32\textwidth}
            \centering
            \includegraphics[width=\linewidth]{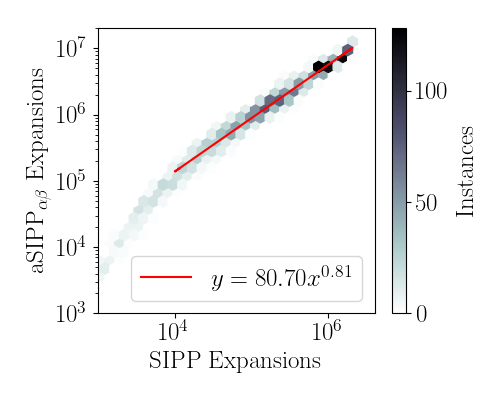}
            \caption{64room\_006 }
            \label{fig:roomexpdif}
        \end{subfigure}
        \caption{\aSIPPab Overhead}
        \label{fig:overheadvdiffab}
\end{figure}

\subsection{The Adequacy of Replanning}

\begin{figure}[b]
    \centering
        \begin{subfigure}{.32\textwidth}
        \centering
        \includegraphics[width=\linewidth]{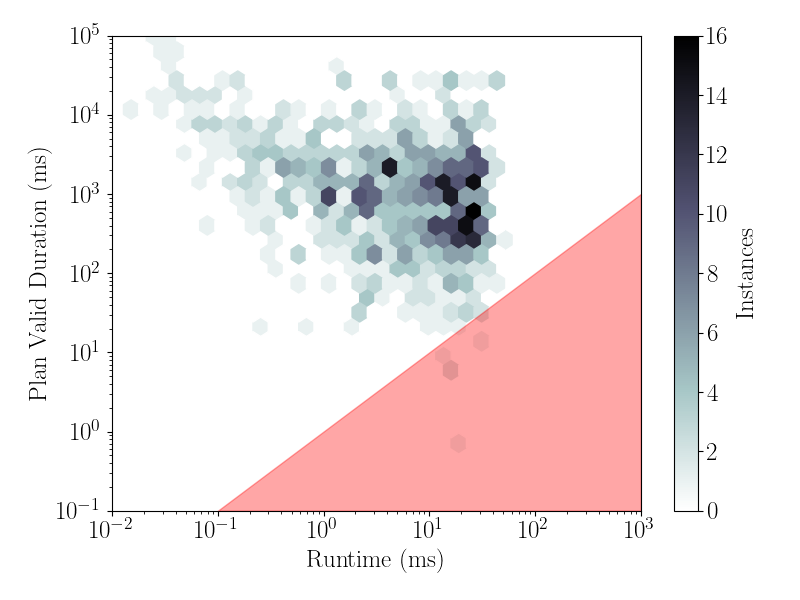}
        \caption{den520d}
        \label{fig:den520dinv}
        \end{subfigure}
        \begin{subfigure}{.32\textwidth}
            \centering
            \includegraphics[width=\linewidth]{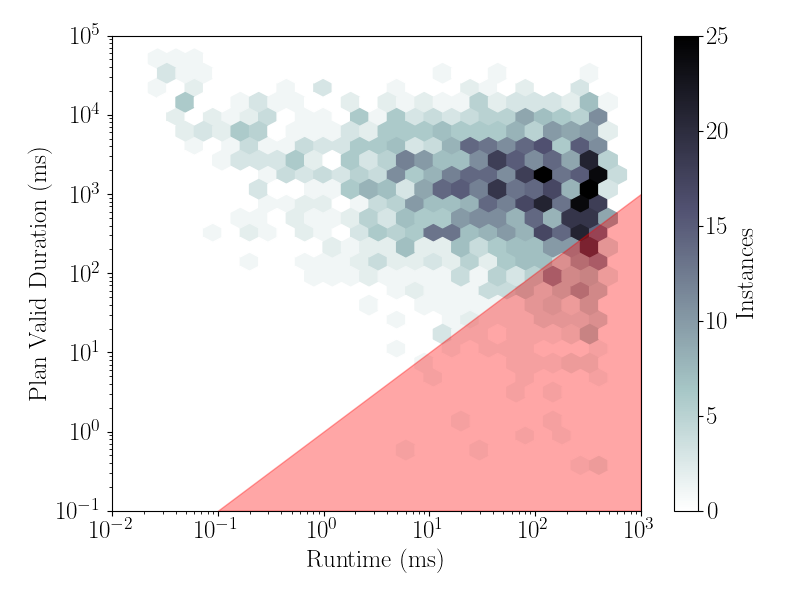}
            \caption{random512-10-0}
            \label{fig:randominv}
        \end{subfigure}
        \begin{subfigure}{.32\textwidth}
            \centering
            \includegraphics[width=\linewidth]{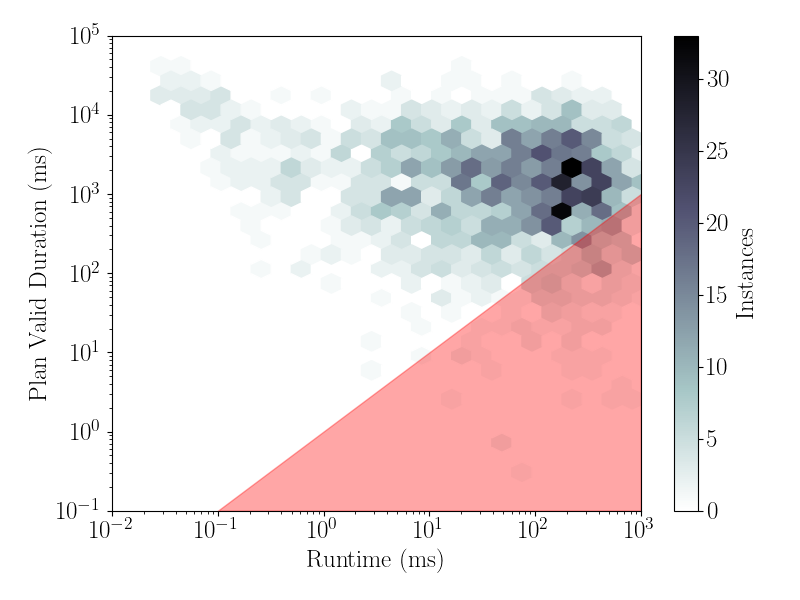}
            \caption{64room\_006 }
            \label{fig:roominv}
        \end{subfigure}\\
        \caption{The time for which SIPP plans are valid versus the time required to find them.}
        \label{fig:invalid}
\end{figure}

We measured the duration for which the SIPP plans remain valid (not necessarily optimal) and compared that to the runtime of the search used to find those plans. 
If the runtime of the search is longer than the duration that the plan remains valid, then the agent will not be able to initiate execution when planning completes and instead will need to immediately replan again. 
In fact, replanning may fail indefinitely, causing the agent to be stuck forever at the start state. 
Our objective is to see if any of our problem instances approach this region, where replanning is not as viable because of the cadence of the problem.

Figure~\ref{fig:invalid} shows the results.  In red is the region where the plan is valid for less time than it takes to compute.
In most of the tested instances, planning is faster than the plans become invalid.
However, for each map, there were some instances ($0.6\%, 10.8\%$, and $11.1\%$ for den520d, random512, and 64room respectively) where planning took longer than the resulting plan was valid for.
Therefore, especially in the two more challenging maps there are already many instances where replanning would be an inadequate substitute for any-start-time planning, even though we are in a setting where planning takes milliseconds.
If, we have the plan ATF produced by Augmented SIPP, we will know immediately that the execution will fail.  
Or, had we had planned an any-start-time plan offline, execution would have succeeded because the on-line planning time would have been negligible.
In these problems, an agent that relied on simple replanning SIPP could become paralyzed, repeatedly being unable to plan fast enough to enact the resulting plan.
Worse, a SIPP agent may not realize that the plan is invalid until it has executed most of the plan, potentially leaving it in an unsafe situation.

\subsection{The Practicality of RePEAT}
\begin{figure}[t]
    \centering
        \begin{subfigure}{.32\textwidth}
        \centering
        \includegraphics[width=\linewidth]{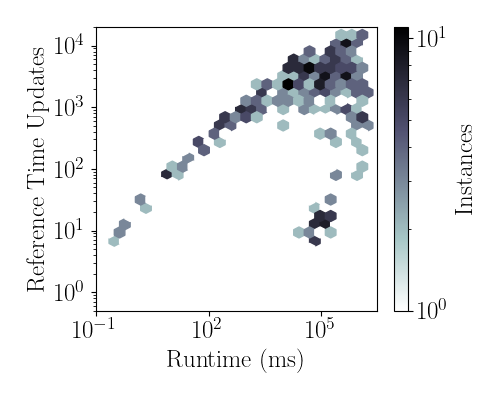}
        \caption{den520d}
        \label{fig:den520dupdates}
        \end{subfigure}
        \begin{subfigure}{.32\textwidth}
            \centering
            \includegraphics[width=\linewidth]{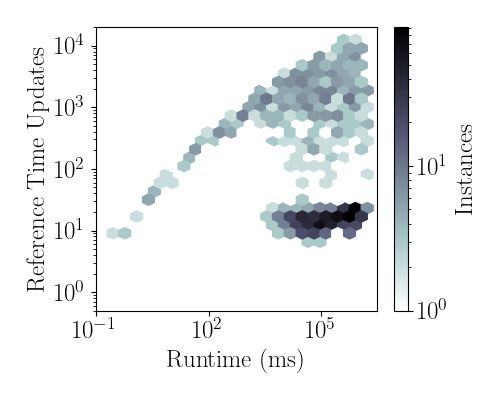}
            \caption{random512}
            \label{fig:randomupdates}
        \end{subfigure}
        \begin{subfigure}{.32\textwidth}
            \centering
            \includegraphics[width=\linewidth]{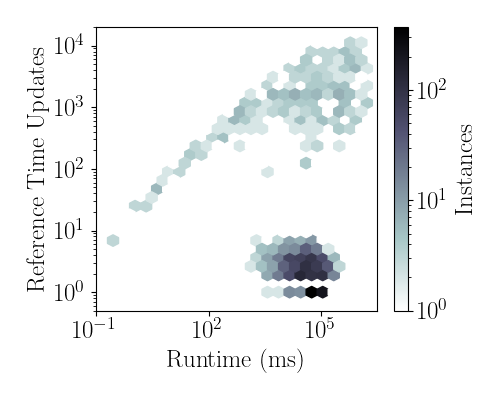}
            \caption{64room}
            \label{fig:roomupdates}
        \end{subfigure}\\
        \caption{The distribution of the number of reference time updates for RePEAT over the instances for our three different maps.}
        \label{fig:updates}
\end{figure}

The runtime performance of RePEAT depends on the speed of the \aSIPPab{} searches and how many searches were run.
Figure~\ref{fig:updates} shows the number of reference time updates (searches run) for RePEAT to generate any-start-time plans for the first 100 seconds on our benchmark instances.
In all three, we see two sorts of instances: one cluster of instances with high runtime but few reference time updates, which we will call the `low-level' hard instances because their time is spent on a few hard Augmented SIPP searches, and a distribution of instances with a number of reference time updates that scales proportionally to the runtime, which we will call the `high-level' hard instances because their difficulty comes from having a more complicated any-start-time plan, that requires many more (but less challenging) Augmented SIPP searches.
Note that the color scale of Figure~\ref{fig:updates} is logarithmic, as otherwise only the low-level hard instances are visible for random512 and 64room.

This difference in the frequency of low-level hard instances may be because random512 and 64room have many tight bottlenecks, which might result in simpler ATFs, with plans that change due mostly to the deadlines at those bottlenecks.
In contrast, in den520d where there are not particularly tight bottlenecks in the static graph, the optimal plan may change very frequently, resulting in performance that is closer to our theoretical bounds on the worst case performance of RePEAT.


\begin{table}[b]
\centering
\begin{tabular}{l|lll}
Map                 & den520d & random512 & 64room \\ \hline
\% Solved in 30 min & 60.3    & 77.8      & 96.3   \\
\% Solved in 100 s  & 34.5    & 36.1      & 80.5  
\end{tabular}
\caption{RePEAT Performance}\label{tbl:repeatperf}
\end{table}

Table~\ref{tbl:repeatperf} shows the percentage of instances for which RePEAT found a 100s any-start-time plan within the 30 min time limit and within 100s.  
We seem to have designed the instances to be sufficiently challenging to accurately measure RePEAT's performance.  
Even though the test instances are quite congested with changing obstacles, many instances still plan fast enough that RePEAT would be able to be used on-line in an anytime manner (because it takes less than 100 seconds to plan for the next 100 seconds).  
Interestingly, though den520d was the easiest map for SIPP, it is the most challenging for RePEAT.

\begin{figure}[t]
    \centering
    \includegraphics[width=0.75\linewidth]{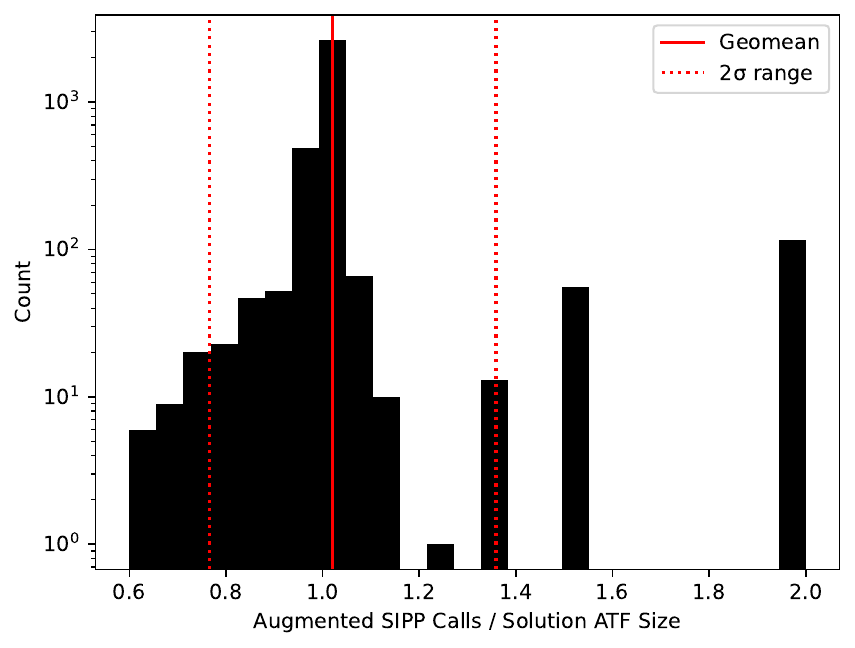}
    \caption{Ratio of Augmented SIPP calls to resulting solution ATF size for RePEAT instances that found a solution.}
    \label{fig:callsvssol}
\end{figure}
Figure~\ref{fig:callsvssol} shows the distribution of ratios of Augmented SIPP calls to size of the solution ATFs.  
The geometric mean of this ratio is $1.02$ (solid red) and the spread is $^*/_\div$ $1.15$ (dotted red), with a $95\%$ confidence interval in the geometric mean of $[1.015, 1.025]$.
The maximum ratio is exactly $2.0$, which lends empirical support to Theorem~\ref{thm:nbreak}.

\begin{figure}[t]
    \centering
    \includegraphics[width=.75\linewidth]{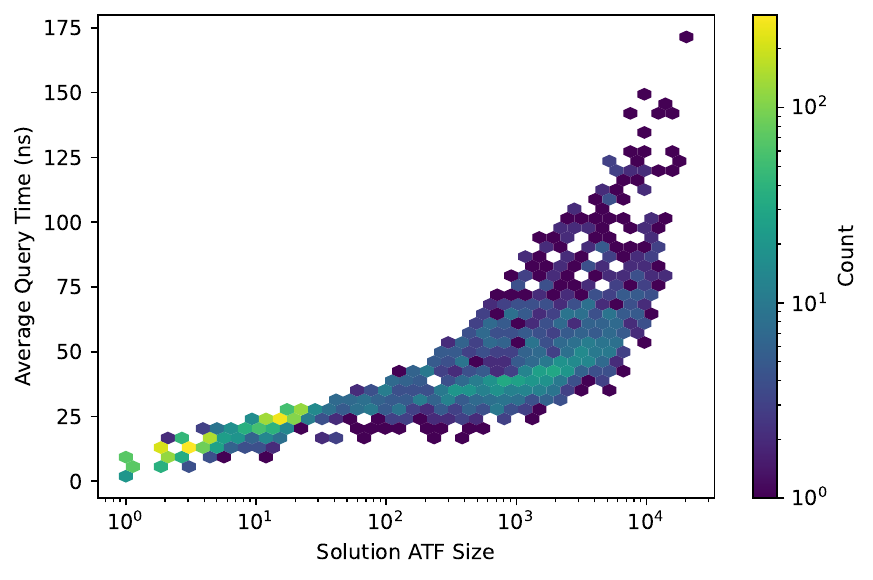}
    \caption{Average online query time for ATFs generated by the RePEAT experiments.}
    \label{fig:querytime}
\end{figure}
Figure~\ref{fig:querytime} shows the mean online query time for each solution ATF for an instance solved by RePEAT.  
For each solution ATF, we performed $1$ billion queries, drawn from the continuous uniformly distribution from $[0, 100)$.
As previously shown in railway networks \cite{Hanou_Thomas_trains}, online performance is effectively instant with  $<\mu s$ query times.
If the ATF is not already in the CPU cache, the single shot performance of the query can be slower due to the memory access overhead.

\section{Discussion}

We discuss applications of any-start-time planning as well as its relationship with related topics.

\subsection{Applications of Any-start-time Planning}

We motivated any-start-time planning using scenarios where an agent did not know the plan execution start time in advance, but that is only the most obvious application.  For example, any-start-time plans contain information that agents can use to assess trade-offs.  To continue our delivery robot example, suppose the robot, while it is waiting for the original meal to be ready, is offered the opportunity to deliver a second meal. If the robot has an any-start-time plan for the original delivery, it can use that to evaluate the potential tradeoff in arrival time for the first meal if it were to depart and deliver the second meal first instead of continuing to wait.

The trade-offs that an agent makes can occur inside the planner itself.  For example, a situated agent, one that seeks to minimize time-to-goal while planning online `as the clock ticks' \cite{clockticks, coles:paw}, needs to decide whether additional search will result in a plan whose superiority justifies delaying execution.  
By having an any-start-time plan to metareason over, it can quickly evaluate the opportunity cost of delaying execution. 

An any-start-time plan does not need to refer to the agent's primary activity --- it can be useful as an emergency back-up.
For example, consider a self-driving vehicle that maintains an any-start-time plan of avoidance maneuvers in case the car in front slams on the brakes.

Any-start-time plans can also be useful in a multi-agent setting.
If each agent has an any-start-time plan, a higher-level planner tasked with resolving conflicts can quickly query these plans to determine what the cost of delaying each agent would be.  A decentralized multi-agent application was explored by \textcite{Hanou_Thomas_trains}, who showed how any-start-time planning can be used to keep a railway network safe when a train is delayed.  Each train views the others' nominal plans as dynamic obstacles and maintains an any-start-time plan for how to reach its destination if it itself were to begin the next segment of its journey at a delayed time.  When a delay happens, the affected train queries its any-start-time plan and instantly has its new safe route.  It then informs the other trains of its new nominal plan, allowing them to adjust their any-start-time contingency plans in turn.

While our experimental evaluation focused on grid pathfinding, we emphasize that the methods we have discussed in this paper can apply to any planning problem in which the applicability or feasibility of states or actions varies over time and the objective is to achieve the goal as early as possible.
If a problem can be represented in these terms, then it is possible to efficiently calculate any-start-time plans.

Even more broadly, our use of ATFs illustrates one way to perform heuristic search with functions instead of scalar values.  We used a function as our cost-so-far $g$ value, accumulating cost through composition.
In the closely related real-time SIPP setting, where an online agent has a fixed time to plan before returning the next action to take,  \textcite{rtsipp} use ATFs to back up heuristic information from the search frontier, again accumulating cost through function composition, but this time representing the cost-to-go $h$.

\subsection{Related Work}

Our work can be seen as a specialization of TDSP problems or as a generalization of SIPP.
To our knowledge, while problems similar to @SIPP have been suggested in remarks by \textcite{halpern1977} and \textcite{foschini2014}, it has not been explored in prior work.
Early work on TDSP by \textcite{COOKE1966} and \textcite{dreyfus1969} provided a Dijkstra style algorithm for finding the quickest plan between nodes on a graph with time-dependent edge delays.
\textcite{halpern1974} introduced TDSP with arc closures, effectively safe intervals on edges.
Later, \textcite{halpern1977} introduced TDSP with parking bans on nodes, and time-dependent edge costs.
However, \textcite{halpern1977} and other more recent work are focused on the case where the node time windows are `parking bans': where a node can be temporarily closed for waiting, but may still be traversed.
They provide a manipulation to their graph that describes how to implement safe intervals on a node by replacing the node with a pair of nodes connected by a single edge corresponding to the safe interval.
The combination of the `Nodes with no-passing-through periods' i.e.,safe intervals described in \textcite{halpern1977} and the edge functions from \textcite{halpern1974} that alternate between a constant duration and infinite duration directly corresponds to our problem.

\textcite{sancho1992} provides a dynamic programming algorithm for the TDSP with time constraints on movement and parking described by \textcite{halpern1974}.
\textcite{ordarom1990} formalize the problem of \textcite{halpern1977} and provide algorithms for a variety of extension and modifications to the problem.
\textcite{ElSherbeny2014} give an A* algorithm for TDSP with time windows.
\textcite{foschini2014} study the complexity of the problems presented in \textcite{ordarom1990}; they prove that for a restricted family of edge slopes, which includes our case of $\{0, 1, \infty\}$ the complexity of the ATF, $A[s,d]$, is linear in the number of edges of the graph.
Modern work in TDSP has focused on rapidly answering repeated queries on a graph often using contraction hierarchies~\cite{contraction}, or shortcuts~\cite{tdsharc} as well as in logistics domains, for instance trucking with driving bans~\cite{vandertuin2018}.

In standard temporal planning~\cite{shperberg:stp}, time is considered, but there is no exogenous change.  In the SIPP-like domains, where TILs have limited effects (demarcating intervals), situated temporal planning is limited by its duplicate state detection, designed for the general case of task planning --- where TILs may have arbitrary state effects --- and degrades to keeping all plans to all states.  

In heuristic search, there are currently two popular ways to cope with dynamic worlds. 
Both are fundamentally based on replanning.  
The first is `incremental search' methods, such as D*lite \cite{koenig:dsl}, that speed-up replanning if only a small number of changes have been made to the state space graph since the original plan was found.  
Such methods are useful only if both the intitial planning and the replanning are fast enough to ensure that the found plans are viable when execution actually begins. 
The second is `real-time search' methods, such as LSS-LRTA* \cite{koenig:cri}, which plan only to a limited horizon.  
These methods iteratively plan and select a single action at a time for execution, rather than forming a complete plan to the goal. 
Thus, they are only applicable in forgiving domains in which a goal is always reachable from any state (and thus there are no `dead-ends').

\section{Conclusion}

This paper explores how a simple change to the typical offline planning problem --- making the start time unknown --- fundamentally changes how we address the problem.  In a dynamic world, it becomes natural to represent plan cost as a function of time.  By adapting results from work on the time-dependent shortest path problem, we have seen how these functions can be compactly represented and composed over plans, yielding new variations on best-first heuristic search.  We proved these algorithms correct and saw that the size of the resulting plans is at most linear in the problem size.  Experimental results illustrated that it is not hard to find SIPP problems where replanning utterly fails, while any-start-time planning appears to be fast enough for practical use.

  For those facing SIPP problems, these results show how any-start-time planning can yield insight into how long a plan is valid, while taking minimal additional overhead.  For those building agents that confront a wide variety of planning tasks, they show how to cope with dynamic environments even though planner runtime can be hard to predict.  And for those interested in combinatorial search, this work provides representations for heuristic search settings in which $g$ and $f$ are better represented as functions instead of scalars.

\section*{Acknowledgments}
We are very grateful for support for this work from the United States National Science Foundation via NSF grant 2008594 and the United States-Israel Binational Science Foundation via BSF grant 2019730.

\printbibliography

@inproceedings{coles:paw,
author = "Andrew Coles and Erez Karpas and Andrey Lavrinenko and Wheeler Ruml and Solomon Eyal Shimony and Shahaf S. Shperberg", 
title = "Planning and Acting While the Clock Ticks", 
booktitle = "Proceedings of the Thirty-fourth International Conference on Automated Planning and Scheduling (ICAPS-24)",
year = 2024
}

@article{ordarom1990,
  title={Shortest-path and minimum-delay algorithms in networks with time-dependent edge-length},
  author={Orda, Ariel and Rom, Raphael},
  journal={Journal of the {ACM}},
  volume={37},
  number={3},
  pages={607--625},
  year={1990},
  publisher={ACM New York, NY, USA}
}

@inproceedings{haehn2020,
  title={Freight Train Scheduling in Railway Systems},
  author={Haehn, Rebecca and {\'A}brah{\'a}m, Erika and Nie{\ss}en, Nils},
  booktitle={International Conference on Measurement, Modelling and Evaluation of Computing Systems},
  pages={225--241},
  year={2020},
  organization={Springer}
}

@article{zhang2020,
  title={Cold chain distribution: how to deal with node and arc time windows?},
  author={Zhang, Yi and Hua, Guowei and Cheng, TCE and Zhang, Juliang},
  journal={Annals of Operations Research},
  volume={291},
  number={1},
  pages={1127--1151},
  year={2020},
  publisher={Springer}
}

@inproceedings{vandertuin2018,
  title={Route planning with breaks and truck driving bans using time-dependent Contraction Hierarchies},
  author={van der Tuin, Marieke and de Weerdt, Mathijs and Batz, G},
  booktitle={Proceedings of ICAPS},
  pages={356--364},
  year={2018}
}

@article{ElSherbeny2014,
  title={The algorithm of the time-dependent shortest path problem with time windows},
  author={El-Sherbeny, Nasser A and others},
  journal={Applied Mathematics},
  volume={5},
  number={17},
  pages={2764},
  year={2014},
  publisher={Scientific Research Publishing}
}

@article{COOKE1966,
title = {The shortest route through a network with time-dependent internodal transit times},
journal = {Journal of Mathematical Analysis and Applications},
volume = {14},
number = {3},
pages = {493-498},
year = {1966},
issn = {0022-247X},
doi = {https://doi.org/10.1016/0022-247X(66)90009-6},
url = {https://www.sciencedirect.com/science/article/pii/0022247X66900096},
author = {Kenneth L Cooke and Eric Halsey}
}

@article{halpern1977,
  title={Shortest route with time dependent length of edges and limited delay possibilities in nodes},
  author={Halpern, Jonathan},
  journal={Zeitschrift fuer operations research},
  volume={21},
  number={3},
  pages={117--124},
  year={1977},
  publisher={Springer}
}

@article{halpern1974,
  title={Shortest path with time constraints on movement and parking},
  author={Halpern, J and Priess, I},
  journal={Networks},
  volume={4},
  number={3},
  year={1974}
}

@article{dreyfus1969,
  title={An appraisal of some shortest-path algorithms},
  author={Dreyfus, Stuart E},
  journal={Operations Research},
  volume={17},
  number={3},
  pages={395--412},
  year={1969},
  publisher={INFORMS}
}

@article{sancho1992,
  title={A dynamic programming solution of a shortest path problem with time constraints on movement and parking},
  author={Sancho, NGF},
  journal={Journal of Mathematical Analysis and Applications},
  volume={166},
  number={1},
  pages={192--198},
  year={1992},
  publisher={Elsevier}
}

@ARTICLE{benchmarkmaps,  author={Sturtevant, Nathan R.},  journal={IEEE Transactions on Computational Intelligence and AI in Games},   title={Benchmarks for Grid-Based Pathfinding},   year={2012},  volume={4},  number={2},  pages={144-148},  doi={10.1109/TCIAIG.2012.2197681}}

@article { hart:fbh,
author = "Peter E. Hart and Nils J. Nilsson and Bertram Raphael",
title = "A Formal Basis for the Heuristic Determination of Minimum Cost Paths",
journal = "IEEE Transactions on System Science and Cybernetics",
year = "1968",
volume = "SSC-4",
number = "2",
pages = "100--107",
month = "7"
}

@article{koenig:cri,
author = "Sven Koenig and Xiaoxun Sun",
year = 2008,
title =  "Comparing real-time and incremental heuristic search for real-time situated agents",
journal = "Journal of Autonomous Agents and Multi-Agent Systems",
volume = 18,
number = 3,
pages = "313-–341"
}

@inproceedings{phillips2011sipp,
  title={{SIPP}: Safe interval path planning for dynamic environments},
  author={Phillips, Mike and Likhachev, Maxim},
  booktitle={Proceedings of {ICRA}},
  pages={5628--5635},
  year={2011}
}

@inproceedings{shperberg:stp,
  title={Situated Temporal Planning Using Deadline-aware Metareasoning},
  author={Shperberg, Shahaf S and Coles, Andrew and Karpas, Erez and Ruml, Wheeler and Shimony, Solomon E},
  booktitle={Proceedings of {ICAPS}},
  pages={340--348},
  year={2021}
}

@inproceedings{clockticks,
  title={Temporal planning while the clock ticks},
  author={Cashmore, Michael and Coles, Andrew and Cserna, Bence and Karpas, Erez and Magazzeni, Daniele and Ruml, Wheeler},
  booktitle={{ICAPS}-28},
  year={2018}
}

@inproceedings{SIPP,
author = {Phillips, Mike and Likhachev, Maxim},
year = {2011},
month = {06},
pages = {5628 - 5635},
title = {{SIPP}: Safe interval path planning for dynamic environments},
journal = {Proceedings - IEEE International Conference on Robotics and Automation},
doi = {10.1109/ICRA.2011.5980306}
}

@INPROCEEDINGS{aSIPP,  author={Narayanan, Venkatraman and Phillips, Mike and Likhachev, Maxim},  booktitle={{IROS}-2012},   title={Anytime Safe Interval Path Planning for dynamic environments},   year={2012},  volume={},  number={},  pages={4708-4715},  doi={10.1109/IROS.2012.6386191}}

@article{dijkstra1959note,
  title={A Note on Two Problems in Connexion with Graphs},
  author={Dijkstra, Edsger W.},
  journal={Numerische Mathematik},
  volume={1},
  pages={269--271},
  year={1959}
}

@article{ginsberg1989universal,
  title={Universal planning: An (almost) universally bad idea},
  author={Ginsberg, Matthew L},
  journal={AI Magazine},
  volume={10},
  number={4},
  pages={40--40},
  year={1989}
}

@inproceedings{rtsipp,
  title={Real-time Safe Interval Path Planning},
  author={Thomas, Devin Wild and Ruml, Wheeler and Shimony, Solomon Eyal},
  year={2024},
  booktitle = "Proceedings of the International Symposium on Combinatorial Search (SoCS-24)"
}

@inproceedings{anystarttimesipp,
  title={Any-Start-Time Planning for {SIPP}},
  author={Thomas, Devin Wild and Shimony, Solomon Eyal and Ruml, Wheeler and Karpas, Erez and Shperberg, Shahaf S and Coles, Andrew},
  booktitle={Proceedings of the ICAPS-23 Workshop on Heuristics and Search for Domain-Independent Planning (HSDIP-23)},
  year={2023}
}

@inproceedings{edelkamp2005cost,
  title={Cost-algebraic heuristic search},
  author={Edelkamp, Stefan and Jabbar, Shahid and Lluch-Lafuente, Alberto},
  booktitle={AAAI},
  volume={5},
  pages={1362--1367},
  year={2005}
}

@article{dechteroptimalastar,
author = {Dechter, Rina and Pearl, Judea},
title = {Generalized best-first search strategies and the optimality of {A}*},
year = {1985},
issue_date = {July 1985},
publisher = {Association for Computing Machinery},
address = {New York, NY, USA},
volume = {32},
number = {3},
issn = {0004-5411},
url = {https://doi.org/10.1145/3828.3830},
doi = {10.1145/3828.3830},
journal = {Journal of the ACM},
month = {7},
pages = {505–536},
numpages = {32}
}

@article{Holte_Zilles_optimalitycostalgebra, title={On the Optimal Efficiency of Cost-Algebraic {A}*}, volume={33}, url={https://ojs.aaai.org/index.php/AAAI/article/view/4067}, DOI={10.1609/aaai.v33i01.33012288}, abstractNote={&lt;p&gt;Edelkamp et al. (2005) proved that A*, given an admissible heuristic, is guaranteed to return an optimal solution in any cost algebra, not just in the traditional shortest path setting. In this paper, we investigate cost-algebraic A*’s optimal efficiency: in the cost-algebraic setting, under what conditions is A* guaranteed to expand the fewest possible states? In the traditional setting, this question was examined in detail by Dechter &amp; Pearl (1985). They identified five different situations in which A* was optimally efficient. We show that three of them continue to hold in the cost-algebraic setting, but that one does not. We also show that one of them is false, it does not hold even in the traditional setting. We introduce an alternative that does hold in the cost-algebraic setting. Finally, we show that a well-known result due to Nilsson does not hold in the general cost-algebraic setting but does hold in a slightly less general setting.&lt;/p&gt;}, number={01}, journal={Proceedings of the AAAI Conference on Artificial Intelligence}, author={Holte, Robert C. and Zilles, Sandra}, year={2019}, month={7}, pages={2288-2295} }

@article{Hanou_Thomas_trains, title={Replanning in Advance for Instant Delay Recovery in Multi-Agent Applications: Rerouting Trains in a Railway Hub}, volume={34}, url={https://ojs.aaai.org/index.php/ICAPS/article/view/31483}, DOI={10.1609/icaps.v34i1.31483}, abstractNote={Train routing is sensitive to delays that occur in the network. When a train is delayed, it is imperative that a new plan be found quickly, or else other trains may need to be stopped to ensure safety, potentially causing cascading delays. In this paper, we consider this class of multi-agent planning problems, which we call Multi-Agent Execution Delay Replanning. We show that these can be solved by reducing the problem to an any-start-time safe interval planning problem. When an agent has an any-start-time plan, it can react to a delay by simply looking up the precomputed plan for the delayed start time. We identify crucial real-world problem characteristics like the agent’s speed, size, and safety envelope, and extend the any-start-time planning to account for them. Experimental results on real-world train networks show that any-start-time plans are compact and can be computed in reasonable time while enabling agents to instantly recover a safe plan.}, number={1}, journal={Proceedings of the International Conference on Automated Planning and Scheduling}, author={Hanou, Issa K. and Thomas, Devin Wild and Ruml, Wheeler and de Weerdt, Mathijs}, year={2024}, month={5}, pages={258-266} }

@article{dehne2012shortest,
  title={Shortest paths in time-dependent FIFO networks},
  author={Dehne, Frank and Omran, Masoud T and Sack, J{\"o}rg-R{\"u}diger},
  journal={Algorithmica},
  volume={62},
  number={1},
  pages={416--435},
  year={2012},
  publisher={Springer}
}

@inproceedings{contraction,
  title={Time-dependent contraction hierarchies},
  author={Batz, G Veit and Delling, Daniel and Sanders, Peter and Vetter, Christian},
  booktitle={ALENEX},
  pages={97--105},
  year={2009},
  organization={SIAM}
}

@article{tdsharc,
  title={Time-dependent SHARC-routing},
  author={Delling, Daniel},
  journal={Algorithmica},
  volume={60},
  number={1},
  pages={60--94},
  year={2011},
  publisher={Springer}
}

@inproceedings{thomas2025real,
  title={Real-time Cost-algebraic Heuristic Search},
  author={Thomas, Devin Wild and Ruml, Wheeler},
  booktitle={Proceedings of the International Symposium on Combinatorial Search (SoCS-25)},
  year={2025}
}

@article{andreychuk2022multi,
  title={Multi-agent pathfinding with continuous time},
  author={Andreychuk, Anton and Yakovlev, Konstantin and Surynek, Pavel and Atzmon, Dor and Stern, Roni},
  journal={Artificial Intelligence},
  volume={305},
  pages={103662},
  year={2022},
  publisher={Elsevier}
}

@article{asai2017tie,
  title={Tie-breaking strategies for cost-optimal best first search},
  author={Asai, Masataro and Fukunaga, Alex},
  journal={Journal of Artificial Intelligence Research},
  volume={58},
  pages={67--121},
  year={2017}
}

@inproceedings{stern2019multi,
  title={Multi-agent pathfinding: Definitions, variants, and benchmarks},
  author={Stern, Roni and Sturtevant, Nathan and Felner, Ariel and Koenig, Sven and Ma, Hang and Walker, Thayne and Li, Jiaoyang and Atzmon, Dor and Cohen, Liron and Kumar, TK and others},
  booktitle={Proceedings of the International Symposium on Combinatorial Search},
  volume={10},
  number={1},
  pages={151--158},
  year={2019}
}

@article{rifki2025phase,
  title={On the Phase Transition of the Euclidean Travelling Salesman Problem with Time Windows},
  author={Rifki, Omar and Solnon, Christine},
  journal={Journal of Artificial Intelligence Research},
  volume={82},
  pages={2167--2188},
  year={2025}
}

@article{foschini2014,
  title={On the Complexity of Time-Dependent Shortest Paths},
  author={Foschini, Luca and Hershberger, John and Suri, Subhash},
  journal={Algorithmica},
  volume={68},
  number={4},
  pages={1075--1097},
  year={2014},
  publisher={Springer}
}

@inproceedings{koenig:dsl,
author = {Koenig, Sven and Likhachev, Maxim},
title = {D*lite},
year = {2002},
isbn = {0262511290},
publisher = {American Association for Artificial Intelligence},
address = {USA},
booktitle = {Eighteenth National Conference on Artificial Intelligence},
pages = {476–483},
numpages = {8},
location = {Edmonton, Alberta, Canada}
}

@misc{boost_icl,
  author       = {Faulhaber, Joachim},
  title        = {{Boost.ICL}: Interval Container Library},
  howpublished = {\url{https://www.boost.org/doc/libs/latest/libs/icl/doc/html/index.html}},
  note         = {Part of the Boost C++ Libraries},
  year         = {2010}
}

@article{Thomas_LCCRTS_2026, 
title={Learning, Commitment, and Completeness in Real-Time Heuristic Search}, volume={19}, url={http://dx.doi.org/10.1609/socs.v19i1.43085}, 
DOI={10.1609/socs.v19i1.43085}, 
abstractNote={<jats:p>In real-time heuristic search, an agent must select its next action within a prespecified time bound. Amazingly, it has been possible to prove that, despite this myopia, certain real-time search algorithms are complete (guaranteed to eventually reach a goal) under a small set of assumptions. This hinges on the algorithms learning updated heuristic values. In order to elucidate the space of complete real-time algorithms, previous work has provided general algorithm frameworks, into which components that meet certain criteria can be plugged, and proved them complete. Recently, a cost-algebraic real-time framework was proposed that makes strong assumptions on its heuristic learning component but only weak assumptions on its action commitment component. In this paper, we show that the opposite is also possible: we present and prove complete an alternative framework that makes weaker assumptions about heuristic learning but stronger assumptions on action commitment. We demonstrate in the domain of Real-time SIPP how the new framework implies the completeness of previously-published algorithms This work expands our knowledge of the design space of complete real-time search algorithms.</jats:p>}, 
number={1}, journal={Proceedings of the International Symposium on Combinatorial Search}, 
publisher={Association for the Advancement of Artificial Intelligence (AAAI)}, 
author={Thomas, Devin Wild and Ruml, Wheeler}, year={2026}, month=aug, pages={166–174} }

@article{Hanou_flexsipp_2026, title={Precomputing Multi-Agent Path Replanning Using Temporal Flexibility}, volume={19}, url={http://dx.doi.org/10.1609/socs.v19i1.43072}, DOI={10.1609/socs.v19i1.43072}, abstractNote={<jats:p>Executing a multi-agent plan can be challenging when an agent is delayed, because this typically creates conflicts with other agents. So, we need to quickly find a new safe plan. Replanning only the delayed agent often does not yield an efficient plan, and sometimes cannot even yield a feasible one. On the other hand, replanning other agents may lead to a cascade of changes and delays, and it is computationally expensive. We show how to efficiently replan a single delayed agent by tracking and using the temporal flexibility of other agents while avoiding cascading delays. This flexibility is the maximum delay that the agent can take without changing the order with agents other than the initially delayed agent, or further delaying other agents. Our algorithm, FlexSIPP, precomputes all possible plans for the delayed agent and returns the changes to the other agents within the given scenario. We demonstrate our method in a real-world case study of replanning trains in the densely-used Dutch railway network and in the MovingAI MAPF benchmark set. Our experiments show that FlexSIPP provides effective solutions relevant to real-world adjustments, and within a reasonable timeframe.</jats:p>}, number={1}, journal={Proceedings of the International Symposium on Combinatorial Search}, publisher={Association for the Advancement of Artificial Intelligence (AAAI)}, author={Hanou, Issa and Kemmeren, Eric and Wild Thomas, Devin and de Weerdt, Mathijs}, year={2026}, month=aug, pages={47–55} }

@article{Andreychuk_improveCCBS_2021, title={Improving Continuous-time Conflict Based Search}, volume={35}, url={http://dx.doi.org/10.1609/aaai.v35i13.17338}, DOI={10.1609/aaai.v35i13.17338}, abstractNote={<jats:p>Conflict-Based Search (CBS) is a powerful algorithmic framework for optimally solving classical multi-agent path finding (MAPF) problems, where time is discretized into the time steps. Continuous-time CBS (CCBS) is a recently proposed version of CBS that guarantees optimal solutions without the need to discretize time. However, the scalability of CCBS is limited because it does not include any known improvements of CBS. In this paper, we begin to close this gap and explore how to adapt successful CBS improvements, namely, prioritizing conflicts (PC), disjoint splitting (DS), and high-level heuristics, to the continuous time setting of CCBS. These adaptions are not trivial, and require careful handling of different types of constraints, applying a generalized version of the Safe interval path planning (SIPP) algorithm, and extending the notion of cardinal conflicts. We evaluate the effect of the suggested enhancements by running experiments both on general graphs and 2^k-neighborhood grids. CCBS with these improvements significantly outperforms vanilla CCBS, solving problems with almost twice as many agents in some cases and pushing the limits of multi-agent path finding in continuous-time domains.</jats:p>}, number={13}, journal={Proceedings of the AAAI Conference on Artificial Intelligence}, publisher={Association for the Advancement of Artificial Intelligence (AAAI)}, author={Andreychuk, Anton and Yakovlev, Konstantin and Boyarski, Eli and Stern, Roni}, year={2021}, month=may, pages={11220–11227} }

@article{Li_mapfLNS2_2022, title={MAPF-LNS2: Fast Repairing for Multi-Agent Path Finding via Large Neighborhood Search}, volume={36}, url={http://dx.doi.org/10.1609/aaai.v36i9.21266}, DOI={10.1609/aaai.v36i9.21266}, abstractNote={<jats:p>Multi-Agent Path Finding (MAPF) is the problem of planning collision-free paths for multiple agents in a shared environment. In this paper, we propose a novel algorithm MAPF-LNS2 based on large neighborhood search for solving MAPF efficiently. Starting from a set of paths that contain collisions, MAPF-LNS2 repeatedly selects a subset of colliding agents and replans their paths to reduce the number of collisions until the paths become collision-free. We compare MAPF-LNS2 against a variety of state-of-the-art MAPF algorithms, including Prioritized Planning with random restarts, EECBS, and PPS, and show that MAPF-LNS2 runs significantly faster than them while still providing near-optimal solutions in most cases. MAPF-LNS2 solves 80% of the random-scenario instances with the largest number of agents from the MAPF benchmark suite with a runtime limit of just 5 minutes, which, to our knowledge, has not been achieved by any existing algorithms.</jats:p>}, number={9}, journal={Proceedings of the AAAI Conference on Artificial Intelligence}, publisher={Association for the Advancement of Artificial Intelligence (AAAI)}, author={Li, Jiaoyang and Chen, Zhe and Harabor, Daniel and Stuckey, Peter J. and Koenig, Sven}, year={2022}, month=june, pages={10256–10265} }

@article{Kerimov_manipsipp_2025, title={Safe Interval Randomized Path Planning For Manipulators}, volume={35}, url={http://dx.doi.org/10.1609/icaps.v35i1.36120}, DOI={10.1609/icaps.v35i1.36120}, abstractNote={<jats:p>Planning safe paths in 3D workspace for high DoF robotic systems, such as manipulators, is a challenging problem, especially when the environment is populated with the dynamic obstacles that need to be avoided. In this case the time dimension should be taken into account that further increases the complexity of planning. To mitigate this issue we suggest to combine safe-interval path planning (a prominent technique in heuristic search) with the randomized planning, specifically, with the bidirectional rapidly-exploring random trees (RRT-Connect) -- a fast and efficient algorithm for high-dimensional planning. Leveraging a dedicated technique of fast computation of the safe intervals, we end up with an efficient planner dubbed SI-RRT. We compare it with the state of the art and show that SI-RRT consistently outperforms the competitors both in runtime and solution cost.</jats:p>}, number={1}, journal={Proceedings of the International Conference on Automated Planning and Scheduling}, publisher={Association for the Advancement of Artificial Intelligence (AAAI)}, author={Kerimov, Nuraddin and Onegin, Aleksandr and Yakovlev, Konstantin}, year={2025}, month=sept, pages={213–217} }

@inproceedings{li2019safe,
  title={Safe and efficient high dimensional motion planning in space-time with time parameterized prediction},
  author={Li, Shen and Shah, Julie A},
  booktitle={2019 international conference on robotics and automation (ICRA)},
  pages={5012--5018},
  year={2019},
  organization={IEEE}
}

@inproceedings{li2021scalable,
  title={Scalable rail planning and replanning: Winning the 2020 flatland challenge},
  author={Li, Jiaoyang and Chen, Zhe and Zheng, Yi and Chan, Shao-Hung and Harabor, Daniel and Stuckey, Peter J and Ma, Hang and Koenig, Sven},
  booktitle={Proceedings of the international conference on automated planning and scheduling},
  volume={31},
  pages={477--485},
  year={2021}
}

@article{KUROIWADIDP,
title = {Domain-independent dynamic programming},
journal = {Artificial Intelligence},
volume = {354},
pages = {104506},
year = {2026},
issn = {0004-3702},
doi = {https://doi.org/10.1016/j.artint.2026.104506},
url = {https://www.sciencedirect.com/science/article/pii/S0004370226000329},
author = {Ryo Kuroiwa and J. Christopher Beck}
}

@article{wolffmulti,
author = {Wolff, Maya and Felner, Ariel and Salzman, Oren},
issn = {2832-9171},
journal = {Proceedings of the International Symposium on Combinatorial Search},
number = {1},
title = {Bridging Multi-Valued Heuristics and Dimensionality Reduction in Multi-Objective Search},
volume = {19},
year = {2026-08-14},
}

\end{document}